\documentclass[review]{elsarticle}
\usepackage[margin=3cm]{geometry}

\usepackage{xcolor}
\usepackage{tabularx} 
\usepackage{graphicx}%
\usepackage{multirow}%
\usepackage{amsmath,amssymb,amsfonts}%
\usepackage{amsthm}%
\usepackage{mathrsfs}%
\usepackage[title]{appendix}%
\usepackage{xcolor}%
\usepackage{textcomp}%
\usepackage{manyfoot}%
\usepackage{booktabs}%
\usepackage{algorithm}%
\usepackage{algorithmicx}%
\usepackage{algpseudocode}%
\usepackage{listings}%
\usepackage{rotating}
\usepackage{dsfont}
\usepackage{bbm}
\usepackage{multicol} 
\usepackage{nicematrix}
\usepackage{multirow}
\usepackage{array}
\usepackage[export]{adjustbox}
\usepackage{fontawesome}
\usepackage{array}
\usepackage{makecell}
\usepackage{natbib}
\usepackage{subfig}
\usepackage{arydshln}

\usepackage{yhmath}
\usepackage{url}
\usepackage{comment}
\usepackage{wrapfig}
\usepackage{rotating}

\usepackage{appendix}

\usepackage{float}
\usepackage{mdframed}
\usepackage{afterpage}

\usepackage{array,etoolbox}
\preto\tabular{\setcounter{magicrownumbers}{0}}
\newcounter{magicrownumbers}

\definecolor{DSgray}{cmyk}{0,1,0,0}
\newcommand{\Authornote}[2]{{\small\textcolor{DSgray}{\sf$<<<${  #1: #2}$>>>$}}}
\newcommand{\kmnote}[1]{{\Authornote{KM}{#1}}} 

\newtheorem{lem}{Lemma}
\newtheorem{assum}{Assumption}

\DeclareMathOperator*{\argmax}{argmax}

\renewcommand{\arraystretch}{1}

\usepackage{marvosym}
\usepackage{mdframed}

\theoremstyle{thmstyleone}%
\newtheorem{theorem}{Theorem}
\newtheorem{proposition}[theorem]{Proposition}%

\theoremstyle{thmstyletwo}%

\theoremstyle{thmstylethree}%
\newtheorem{definition}{Definition}%

\newtheorem{lemma}{Lemma}[section]

\journal{Artificial Intelligence}

\begin{document}

\begin{frontmatter}

\title{Mass Distribution versus Density Distribution\\ in the Context of Clustering}








\author[kaiming,kaiming2]{Kai Ming Ting\textsuperscript{\Letter}}
	\ead{tingkm@nju.edu.cn}
\author[deakinSRC]{Ye Zhu}
	\ead{ye.zhu@ieee.org}
\author[kaiming,kaiming2]{Hang Zhang\textsuperscript{\Letter}}
        \ead{zhanghang@lamda.nju.edu.cn}
\author[kaiming,kaiming2]{Tianrun Liang}
    \ead{liangtr@lamda.nju.edu.cn}

\address[kaiming]{State Key Laboratory for Novel Software Technology, Nanjing University, China 210023}
\address[kaiming2]{School of Artificial Intelligence, Nanjing University, Nanjing, China 210023}
 
\address[deakinSRC]{School of IT, Deakin University, Geelong, Australia 3125}


\begin{abstract}
    This paper investigates two fundamental descriptors of data, i.e., density distribution versus mass distribution, in the context of clustering. Density distribution has been the de facto descriptor of data distribution since the introduction of statistics. We show that density distribution has its fundamental limitation---high-density bias, irrespective of the algorithms used to perform clustering.   
Existing density-based clustering algorithms have employed different algorithmic means to counter the effect of the high-density bias with some success, but the fundamental limitation of using density distribution remains an obstacle to discovering clusters of arbitrary shapes, sizes and densities.  
Using the mass distribution as a better foundation,
we propose a new algorithm which maximizes the total mass of all clusters, called mass-maximization clustering (MMC).
The algorithm can be easily changed to maximize the total density of all clusters in order to examine the fundamental limitation of using density distribution versus mass distribution. The key advantage of the MMC over the density-maximization clustering is that the maximization is conducted without a bias towards dense clusters. 


\end{abstract}



\end{frontmatter}

\section{Introduction and motivation}


Density-based clustering \cite{ester1996density,DENCLUE,rodriguez2014clustering} has its appeal because it can discover clusters of arbitrary shapes and sizes which match the distribution of the given dataset in high density regions. However, it has two key shortcomings. First, in terms of clustering outcomes, density-based clustering has been `haunted' by the difficulty of discovering clusters of low density in the presence of high-density clusters. Despite various improvements (e.g., \cite{ankerst1999optics,HDBSCAN-2015,
DP_jain,chen2018local,li2019LGD,ros2022detection,yang2023hcdc}), the issue of bias towards dense clusters has merely shifted from one form to another, solving the problem in an early version of algorithm but creating a new problem in the new version. Examples are given in Table \ref{tab-motivation}, where the four density-based clustering algorithms  fail to detect all clusters in at least one of the four datasets having different clusters of varied densities. That is, DBSCAN \cite{ester1996density} could succeed on Jain and AC only; DP \cite{rodriguez2014clustering} could succeed on AC and 3G only; LGD \cite{li2019LGD} failed on AC though succeeded on others; and a new density-based algorithm called DMC introduced here succeeded on all datasets, except 3G.

\begin{table}[h]
  \centering
    \caption{Examples in which no single density-based clustering algorithm (out of DBSCAN, DP, LGD \& DMC) can successfully discover all clusters in all four datasets having different clusters of varied densities. The ones with a yellow frame indicate perfect or near-perfect clustering outcomes. The number shown under each subfigure is the clustering outcome result in terms of F1.}
    \label{tab-motivation}
  \begin{tabularx}{\textwidth}{c|cccc|m{3cm}}
    \hline
   & DBSCAN & DP&  LGD & DMC & Description\\
    \hline 
\rotatebox{90}{RingG}&  \includegraphics[width=.16\textwidth,valign=c]{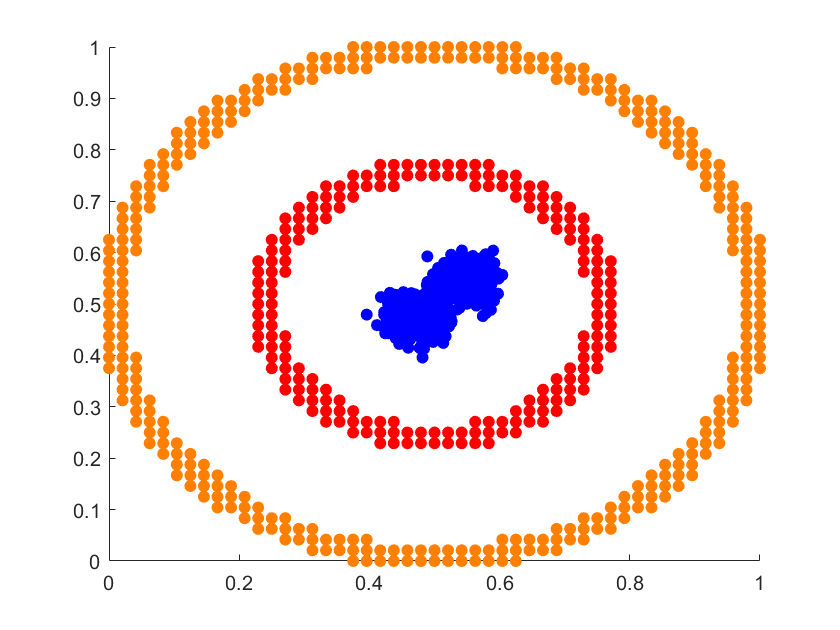} 
&\includegraphics[width=.16\textwidth,valign=c]{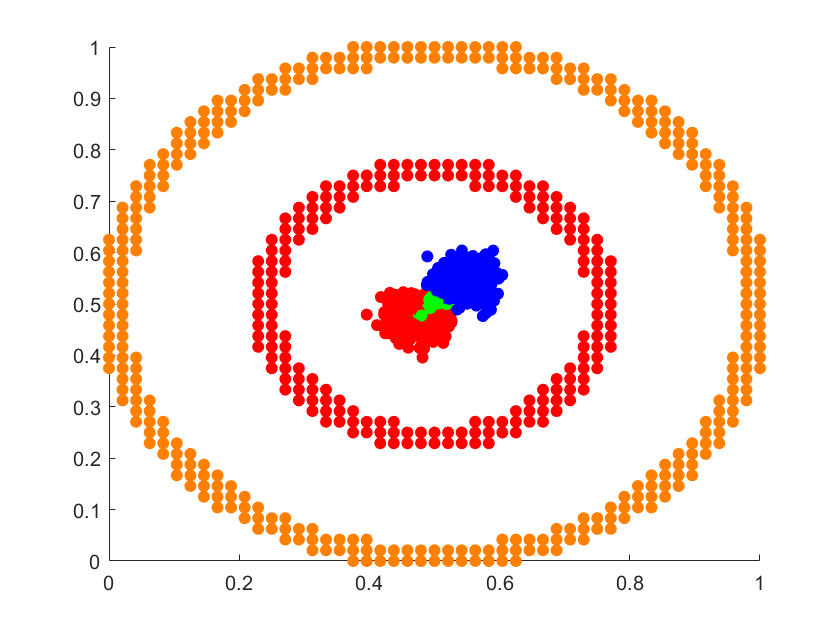} 
&\includegraphics[width=.16\textwidth,cframe=yellow 0.5mm,valign=c]{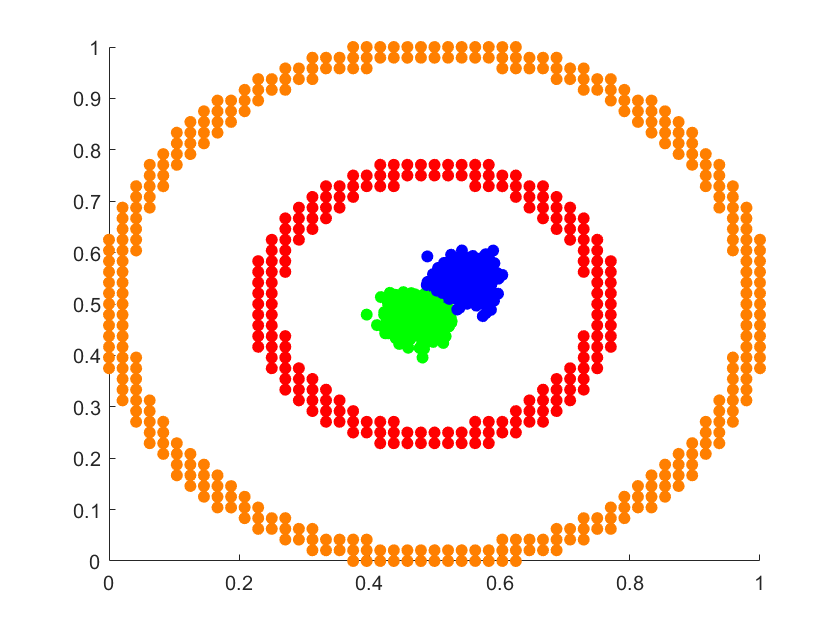}
&\includegraphics[width=.16\textwidth,cframe=yellow 0.5mm,valign=c]{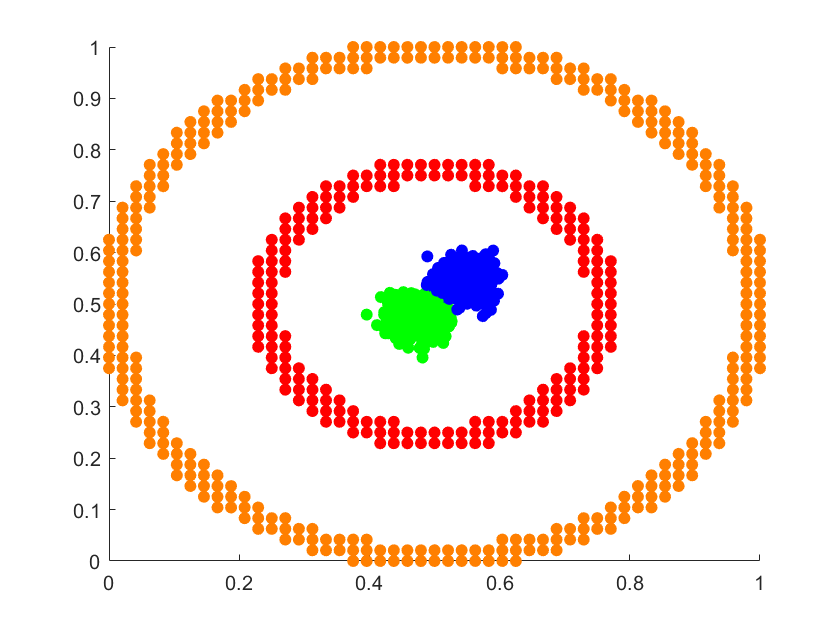} 
& Two dense Gaussian clusters \& two sparse ring clusters 
\\ & 0.67  & 0.96   & 1 & 1
\\ \hline
\rotatebox{90}{Jain}&\includegraphics[width=.16\textwidth,cframe=yellow 0.5mm,valign=c]{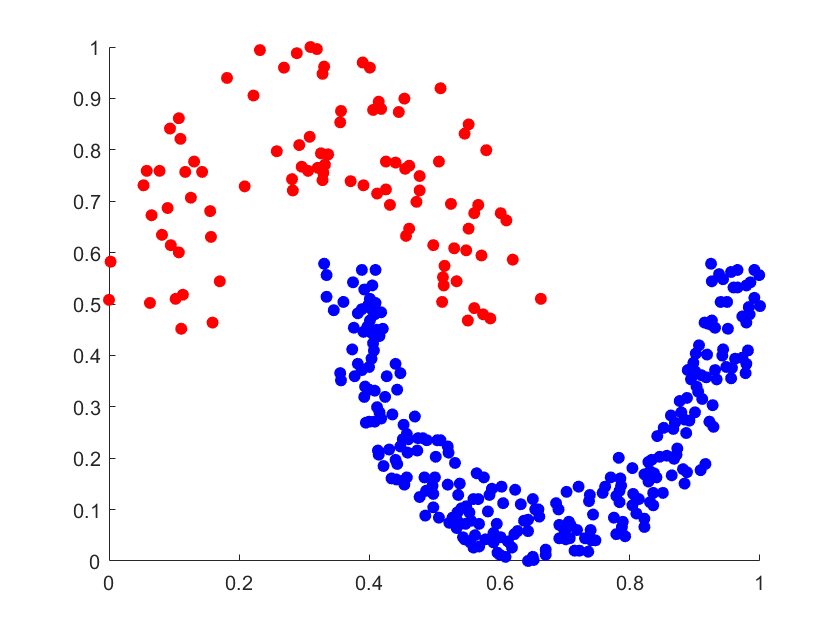}  
&\includegraphics[width=.16\textwidth,valign=c]{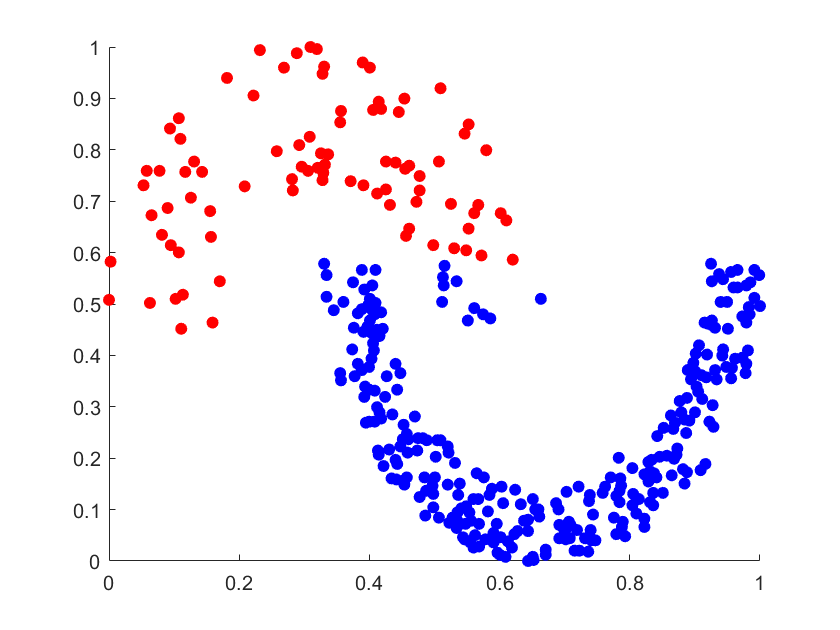} &\includegraphics[width=.16\textwidth,cframe=yellow 0.5mm,valign=c]{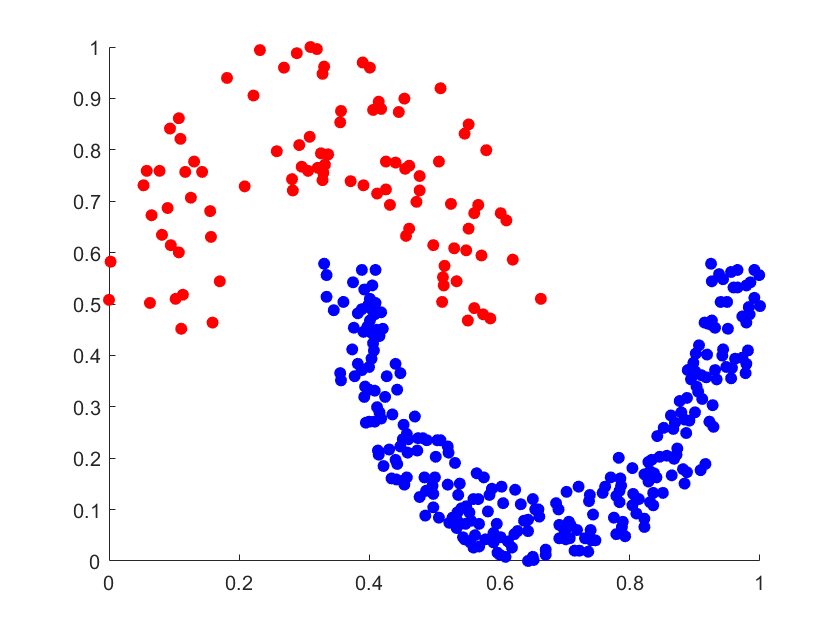}
 &\includegraphics[width=.16\textwidth,cframe=yellow 0.5mm,valign=c]{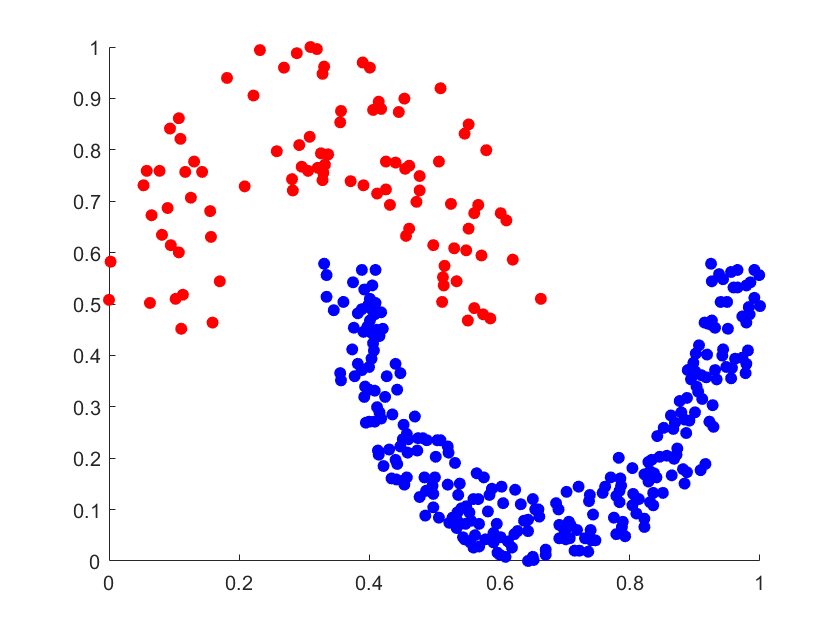}
 & \noindent Sparse top arc cluster \& dense bottom  arc cluster
 \\ & 1 & 0.74 & 1 & 1
\\ \hline
\rotatebox{90}{AC}& \includegraphics[width=.16\textwidth,cframe=yellow 0.5mm,valign=c]{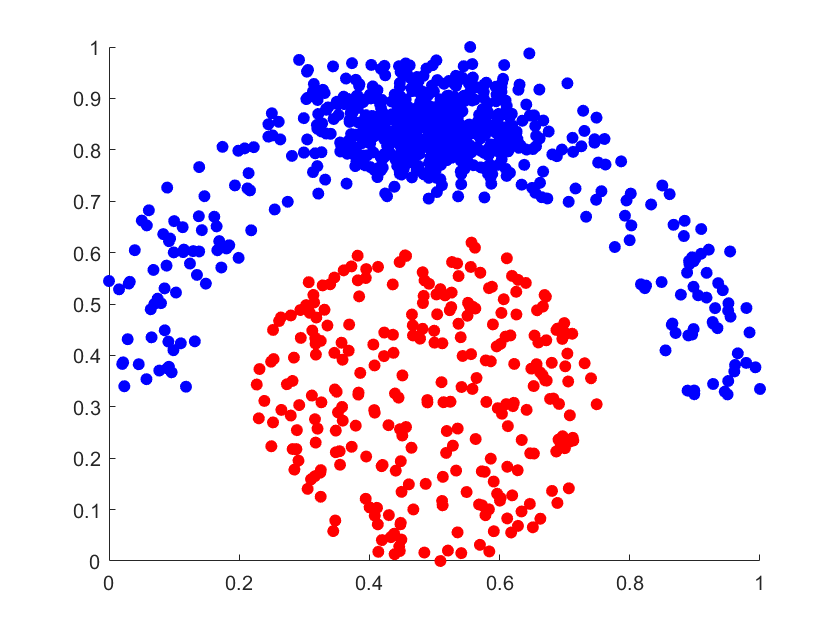} 
&\includegraphics[width=.16\textwidth,cframe=yellow 0.5mm,valign=c]{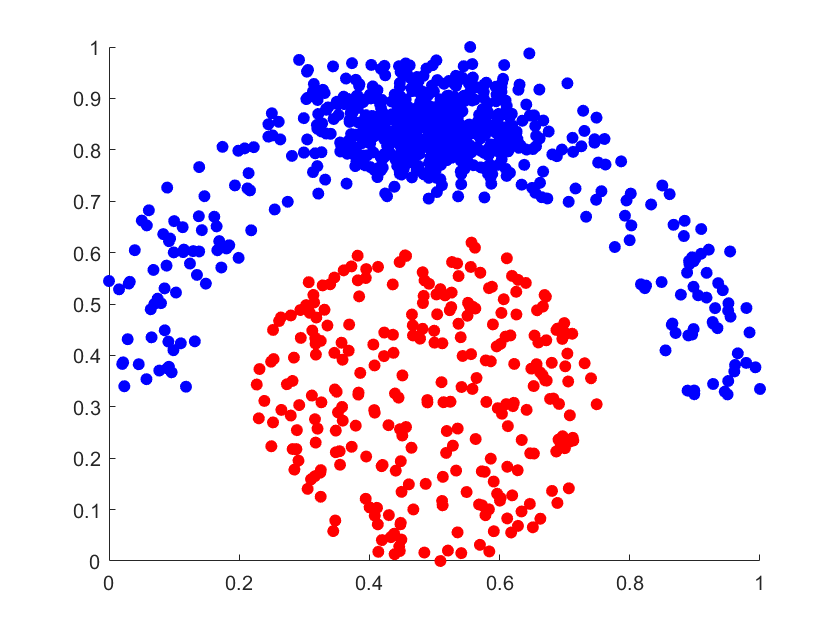} &\includegraphics[width=.16\textwidth,valign=c]{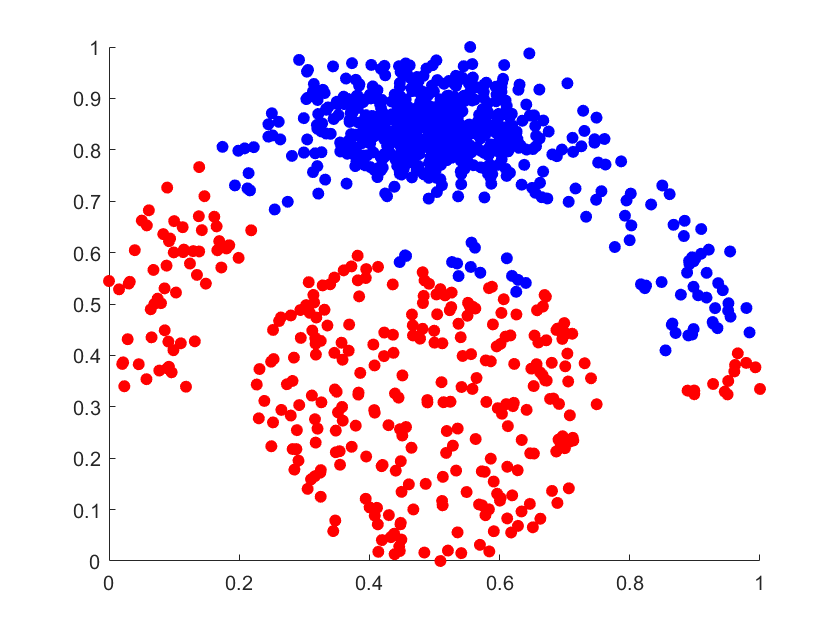} 
&\includegraphics[width=.16\textwidth,cframe=yellow 0.5mm,valign=c]{figures/2dimVis/AC-dmc.png}  
& Dense-centered arc cluster \& sparse circle cluster
\\    & 1  & 1  & 0.90  & 1
\\ \hline
 \rotatebox{90}{3G}&\includegraphics[width=.16\textwidth,valign=c]{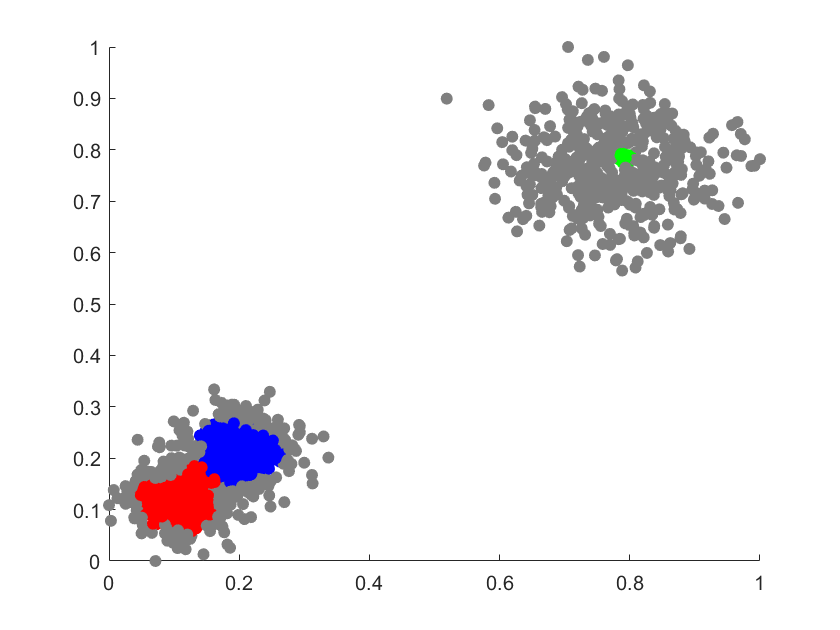} 
 &\includegraphics[width=.16\textwidth,cframe=yellow 0.5mm,valign=c]{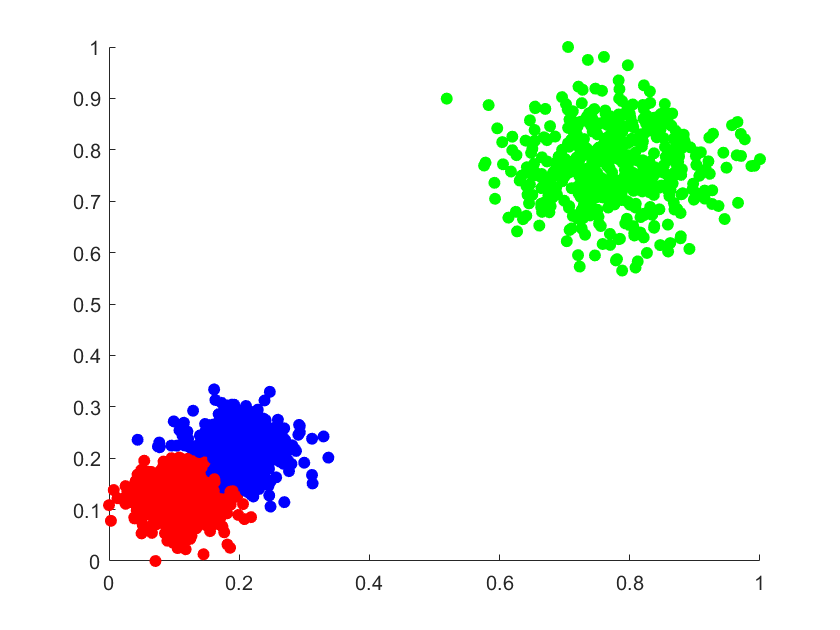} &\includegraphics[width=.16\textwidth,cframe=yellow 0.5mm,valign=c]{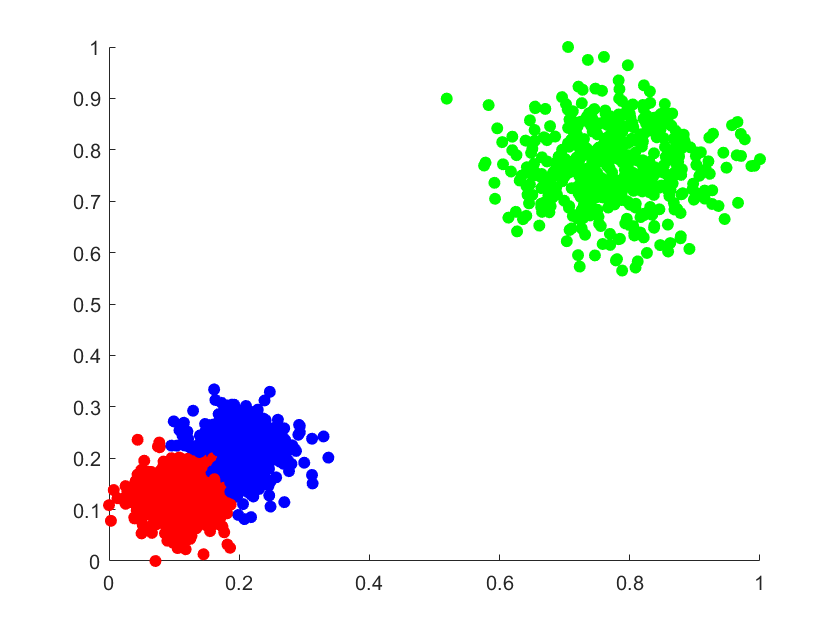} &\includegraphics[width=.16\textwidth,valign=c]{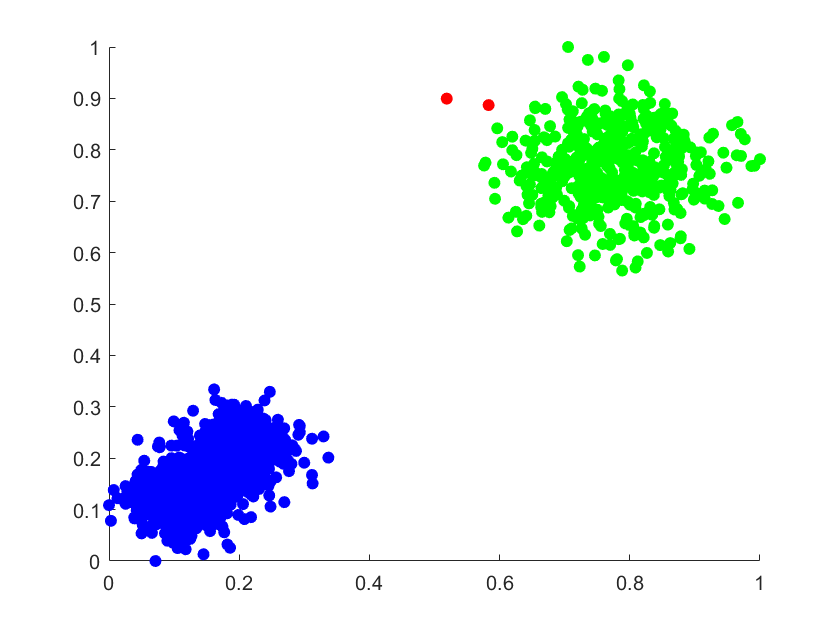}  
 & Two dense Gaussian clusters \& one sparse Gaussian cluster 
 \\ & 0.57   & 0.98   & 0.97   & 0.73
 \\ \hline

  \end{tabularx}

\end{table}

Second, density-based clustering algorithms have at least quadratic time complexity\footnote{The term `time complexity' refers to the worst-case time complexity with respect to the dataset size.} because they must use a density estimator in order to estimate the density for every point in the given dataset. 

\textcolor{black}{Recent works \cite{Mass-based-similarity-KDD2016,LMN-MLJ2019, IsolationKernel-AAAI2019} have revealed that the first key shortcoming is a result of using density distribution---a fundamental limitation that can be rectified by using mass distribution instead. They demonstrate the different clustering outcomes due to density distribution versus mass distribution using the same DBSCAN algorithm by simply replacing Euclidean distance with a mass-based similarity or a data dependent kernel. Though the effect of using a data-dependent/data-independent measure is clear, these works do not provide an explanation to this effect. Our work here fills in this gap in giving a new understanding in terms of cluster cohesiveness that has a wider impact on all clustering algorithms.}

 \textcolor{black}{Our analysis further reveals that the first key shortcoming is a result of an additional fundamental limitation of the point-based linking process in existing density-based clustering algorithms, for which the improvements thus far have failed to recognize and address. By addressing these two limitations, we show that the bias towards dense clusters can be eliminated.\\
In addition, the solution has a significant efficiency gain---creating an algorithm having linear time complexity  with respect to the dataset size, where existing density-based clustering algorithms have at least quadratic time complexity}.

Our contributions are:
\begin{enumerate}
    \item \textcolor{black}{Making explicit that the aim of clustering is to discover clusters of arbitrary shapes, sizes and densities in a given dataset. This means that a clustering algorithm shall find all clusters, irrespective of their shapes, sizes and densities. As the use of density-distribution leads to the high-density bias, density-based clustering algorithms have an inherent weakness in finding low-density clusters.}   
    \item \textcolor{black}{Targeting to address both} the two fundamental limitations of existing density-based clustering algorithms: (a) The use of density distribution to describe the data distribution; and (b) the use of point-to-point linking process to form the final clusters. 
    \item  \textcolor{black}{Proposing an integrated means that} (i) uses mass distribution to describe the data distribution; (ii) represents each cluster as a distribution via a kernel; and (iii) finds a representative sample of every cluster as the means to assign each point in the dataset to its most similar distribution/cluster.
    \item Enacting a kernel mass estimator for the first time based on a recent \emph{data dependent} kernel called Isolation Kernel (IK). Its density counterpart is kernel density estimator which typically uses the \emph{data independent} Gaussian kernel. We show that IK guarantees
the cohesiveness of every cluster to be approximately the same; and this produces a
representative sample for every cluster.
Density-based clustering has high-density bias because the low-density clusters are significantly less cohesive than high-density clusters. (See the definition of cohesiveness in Section \ref{sec-cohesiveness}.)
    \item Establishing that mass distribution, via  representative samples,  enables clusters of arbitrary shapes, sizes and densities
in a dataset to be discovered,  without the high-density bias. 
    \item Creating a new clustering algorithm, called Mass-Maximization Clustering (MMC), which maximizes the total mass of all clusters.
    It is a generic algorithm which can be easily converted to one that maximizes the total density of all clusters in order to examine the fundamental limitation of using density distribution versus mass distribution. Its density counterpart is called Density-Maximization Clustering (DMC).
    \item Showing that (a) MMC has superior clustering outcomes than its density counterpart DMC as well as existing density-based algorithms;
    and (b) both MMC and DMC have linear time complexity whereas existing density-based clustering algorithms have at least quadratic time complexity.
\end{enumerate}

The proposed clustering algorithm is a new class of clustering  which has two key steps. First, it finds initial clusters which are the representative samples of the distributions for individual clusters. 
Second, it assigns each point in the given dataset to its most similar initial cluster, as measured via a kernel, aiming to maximize the total mass (or density) of all clusters. 

Both steps ensure that both the initial and final clusters discovered are cohesive clusters, and all clusters have approximately the same average cohesiveness, when a recent data dependent kernel is used.  

The use of a data independent kernel such as Gaussian kernel in the proposed algorithm produces a density-based clustering which has the same bias towards high density clusters as existing density-based clustering algorithms, and it can not achieve the desired clustering outcome---discovering clusters
of arbitrary shapes, sizes and densities in a dataset---without bias.

The cluster definition in MMC is unique among existing density-based clustering algorithms in the following aspects:
\begin{enumerate}
    \item \textbf{Cluster-as-distribution definition}: The clusters discovered by the proposed algorithm are defined based on a kernel, i.e., each final cluster is discovered by  treating the initial cluster as a distribution. This differs from the point-to-point linking process used in existing density-based clustering algorithms to form the final clusters---the distribution of a cluster is not used to form the final cluster. This is despite the fact that the density of each point must be estimated by a density estimator. 
    \item \textbf{Clustering objective function}: The cluster-as-distribution definition leads to a clustering objective function which maximizes the total mass (or density) of all clusters.
    In contrast,  none of the existing density-based clustering algorithms have a clustering objective function, as far as we are aware.
\end{enumerate}

\textcolor{black}{The proposed mass-based clustering is limited to clusters which can be represented as distributions and the representative points of each cluster can be easily obtained via sampling.}

\section{Two fundamental limitations in density-based algorithms}
\label{sec-fundamental-limitation}

%
\textcolor{black}{\emph{The aim of clustering is to discover clusters of arbitrary shapes, sizes and densities}. Yet, this has not been stated explicitly often enough in the literature. The aim requires an algorithm to find all clusters,  irrespective of their shapes, sizes and densities. Yet, many existing clustering algorithms have problems finding clusters having different densities.}

\textcolor{black}{Given the aim, the issue of density-based clustering in this respect is obvious because, being density-based, it has an inherent bias towards high-density clusters.}

\textcolor{black}{This creates a difficulty for density-based clustering algorithms to identify low-density clusters. This phenomenon, known as high-density bias, is defined as follows:
\begin{definition}
    A clustering algorithm is said to have high-density bias when it is more likely to correctly assign points of high-density clusters than points of low-density clusters.
\end{definition}
This inherent weakness of density-based clustering is often overlooked because the aim of clustering is not made explicit\footnote{\textcolor{black}{An early algorithm that has brought this issue on the spotlight is the SNN clustering algorithm \cite{SNN-DBSCAN-SDM2003}. However, it has algorithmic issues because it is still based on the k-nearest neighbor search and the DBSCAN algorithm. See Section \ref{sec-mass} for further discussion.}}.}

\textcolor{black}{The first successful density-based algorithm,
DBSCAN \cite{ester1996density}, uses a global threshold in order to identify high-density (or core) points. This produces two outcomes: (a) significantly fewer points in sparse clusters are identified as core points than those in dense clusters---as a consequence, the discovered sparse clusters are either significantly smaller than what they actually are, or some sparse clusters are not detected at all. An example is shown on the 3G dataset shown in Table \ref{tab-motivation}. The point-to-point linking process, used to identify the (high) density-connected clusters, has no recourse to non-core points, many of which belong to sparse clusters. As a result, many of these points are designated as noise points.  (b) Neighboring dense clusters are merged into a single cluster if the global threshold is lower to an extent in order to discover sparse clusters of size closer to the original clusters. An example is shown on the RingG dataset shown in Table \ref{tab-motivation}. }

\textcolor{black}{Many existing density-based clustering algorithms mitigate the impact of the high-density bias through various algorithmic techniques. 
For example, DP chooses some peaks and uses a different linking process to assign every point to its nearest neighbor of higher density to avoid using a global threshold to identify core points. Though this procedure helps to avoid DBSCAN's difficulty in discovering low-density clusters, it creates a different kind of problem which does not exist in DBSCAN. 
A cluster, having higher density than other clusters, could be split into multiple clusters because multiple peaks are selected in this high density cluster (as all density peaks occur in this cluster only), rather than one peak per cluster. As a result, clusters of smaller size and lower density are often not identified as individual clusters. 
An example on the RingG dataset is shown in Table \ref{tab-motivation}.
The DP's result is poor not only because the overlap `cluster' is a poor representation of the overlap region of the two dense clusters, but the inner ring cluster is grouped together with one of the dense cluster as a single cluster. Setting other parameters of DP produce worse F1 results than what we have shown here.
This is a direct outcome of the high-density bias, i.e., the bias is towards the density peaks that exist in one or more high-density clusters.}

\textcolor{black}{Indeed, using a different means called local gap density to remove edges of a kNN-graph constructed from a given dataset, LGD \cite{li2019LGD} has overcome the weaknesses of DBSCAN and DP on the RingG and 3G datasets. Yet, on the AC datasets that both DBSCAN and DP have no issue, LGD erroneously removes two ends of the sparse arc cluster and then connects the two subclusters, that should be associated with the top dense cluster, with the bottom sparse ball cluster.}

\textcolor{black}{Using  distribution-defined clusters, the proposed density-based clustering algorithm DMC has successfully identified all clusters on the first three datasets, shown in Table \ref{tab-motivation}, which DBSCAN, DP and LGD have failed on at least one of them. Yet, DMC fails on the 3G dataset that both DP and LGD succeed. 
}


In a nutshell, existing density-based clustering algorithms have a fundamental limitation, i.e., the high-density bias, because of the use of density distribution. This limitation can manifest in incorrect cluster identification in different ways, depending on the algorithms used. Each of existing density-based algorithms, e.g., DBSCAN, DP and LGD, can be seen as a patch to the fundamental limitation. Solutions without addressing this fundamental limitation have the following sign: a latter patch may seem to plug the `leak' of the previous patch. Yet, the latter patch has a new `leak' that does not exist in the previous patch.  
No existing density-based algorithms or the proposed DMC is immune to this fundamental limitation.

  In other words, the fundamental limitation of density distribution persists even if there is no density estimation error. 
To eliminate this fundamental limitation, one must use a distribution which has no bias towards either dense or sparse clusters that exist in a dataset.

\textcolor{black}{To do this, we propose to \emph{change the fundamental descriptor of data from density to mass}, i.e., to use a mass distribution rather than a density distribution. }




Note that the linking process, used to form either linking-defined or kNN-graph defined clusters, in existing density-based clustering algorithms has high time complexity because it requires a nearest neighbor search to perform the point-to-point linking process. Here we propose to use a non-linking process to do the final cluster formation which has linear time complexity.

The different kinds of clusters produced by different density based algorithms (DBSCAN, DP \& LGD) are provided in Table \ref{tab:Density-based-clustering}.
Rather than the linking-defined and the kNN-graph-defined clusters, the proposed algorithm DMC produces distribution-defined clusters via a kernel (to be described in Section~\ref{sec-MMC}). Yet, the fundamental limitation of using density distribution still persists, as demonstrated in Table \ref{tab-motivation}.

\begin{table}[h]
    \centering
     \caption{The kinds of clusters density-based clustering algorithms produced.}  
    \begin{tabular}{cc|c|c}
    \toprule
       \multicolumn{2}{c|}{DBSCAN  \& DP} & LGD &  DMC\\ \midrule
 \multicolumn{2}{c|}{Linking-defined clusters} & kNN-graph-defined clusters & Distribution-defined clusters\\ 
 \multicolumn{2}{c|}{High-density clusters} & Locally high-density clusters & Clusters having the highest total density \\ 
\bottomrule
    \end{tabular}
   
    \label{tab:Density-based-clustering}
\end{table}


\section{Isolation Kernel}
\label{sec_IsolationKernel}

In Section \ref{sec_prelim-IK}, we provide the existing understanding of the Isolation Kernel \cite{ting2018IsolationKernel, IsolationKernel-AAAI2019}. In Sections \ref{sec-cohesiveness} and \ref{sec-representative-sample}, we present our new findings that the Isolation Kernel produces clusters of approximately the same cohesiveness and also representative samples of all clusters in a dataset. The key symbols and notations used are provided in Table \ref{tbl_symbols}.

\begin{table}[h]
		\centering
		\caption{Key symbols and notations used.}
		\label{tbl_symbols}
        \setlength{\tabcolsep}{1pt}
		\begin{tabular}{ll}
			\toprule
		$\mathbf{x}$ & A point in input space $\mathbb{R}^d$\\
        $\eta_\mathbf{x}$ & The most-similar-neighbor of $\mathbf{x}$ in some set\\
		$D$ & A set of points $\{\mathbf{x}_i\ |\ i=1,\dots,n\}$ in $\mathbb{R}^d$,  where $\mathbf{x} \sim {P}_D$\\
        $C$ & A cluster of points, $C \subset D$\\
		$P_D$ & An (unknown) distribution that generates a set $D$ of points in $\mathbb{R}^d$;\\ 
         & so as $P_C$ for any set $C$ of points in $\mathbb{R}^d$.\\	
         $H \in \mathbb{H}_\psi(D)$ & $H$ is a space partitioning having $\psi$ partitions that can be generated from $D$ in set $\mathbb{H}_\psi(D)$\\
		$\mathcal{P}(\theta)$ & Probability mass in space partition $\theta \in H$\\
        $\ell_p$ & The $\ell_p$-norm or Minskowski norm \\
        $\parallel \cdot \parallel$ & $\ell_2$-norm\\ 
        $\rho(\mathbf{x})$ & Density of  $\mathbf{x}$\\
        $\kappa,\ \phi$    & A point-to-point kernel and its feature map \\
			\bottomrule
		\end{tabular}
	\end{table}

\subsection{Current understanding of the Isolation Kernel}
\label{sec_prelim-IK}

Let $D \subset \mathbb{R}^d$ be a dataset sampled from an unknown distribution $P_D$\footnote{Here $P_D$ or the term `distribution' is neutral, and it could be associated to mass distribution or density distribution, depending on the context. The details of the distinction are provided in Section \ref{sec_Mass}.}; and \textcolor{black}{$\mathbb{H}_\psi(D)$ denote the set of all partitionings $H$ that can be generated
from $\mathcal{D} \subset D$, which is a random subset of $\psi$ points that have the equal probability of being selected from $D$. Each partition $\theta[\mathbf{z}] \in H$ isolates a point $\mathbf{z} \in \mathcal{D}$ from the rest of the points in $\mathcal{D}$. Each partitioning may cover the entire space or part of the regions, depending on the isolation mechanism used (e.g., Voronoi Diagrams or Hyperspheres).}

\begin{definition} \cite{ting2018IsolationKernel,IsolationKernel-AAAI2019} For any two points $\mathbf{x},\mathbf{y} \in \mathbb{R}^d$,
	Isolation Kernel of $\mathbf{x}$ and $\mathbf{y}$ is defined to be
	the expectation taken over the probability distribution on all partitionings $H \in \mathds{H}_\psi(D)$ \textcolor{black}{with equal weighting} that both $\mathbf{x}$ and $\mathbf{y}$  fall into the same isolating partition $\theta[\mathbf{z}] \in H$:
	\begin{equation}
\kappa_I(\mathbf{x},\mathbf{y}\ |\ D) = {\mathbb E}_{H \sim \mathds{H}_\psi(D)} [\mathds{1}(\mathbf{x},\mathbf{y} \in \theta[\mathbf{z}]\ | \ \theta[\mathbf{z}] \in H)]
		\label{eqn_kernel}
	\end{equation}
where $\mathds{1}(\cdot)$ is an indicator function.
\end{definition}


In practice, the Isolation Kernel $\kappa_I$ is constructed using a finite number of partitionings $H_i, i=1,\dots,t$, where each $H_i$ is created using randomly subsampled $\mathcal{D}_i \subset D$; and $\theta$ is a shorthand for $\theta[\mathbf{z}]$:
\begin{eqnarray}
\kappa_I(\mathbf{x},\mathbf{y}\ |\ D)  \approx   \frac{1}{t} \sum_{i=1}^t   \mathds{1}(\mathbf{x},\mathbf{y} \in \theta\ | \ \theta \in H_i) 
= \frac{1}{t} \sum_{i=1}^t \sum_{\theta \in H_i}   \mathds{1}(\mathbf{x}\in \theta)\mathds{1}(\mathbf{y}\in \theta) 
 \label{Eqn_IK}
\end{eqnarray}

\textcolor{black}{Note that $\psi$ in the Isolation Kernel is the equivalent of the bandwidth parameter in the Gaussian Kernel, where it is required to be tuned for each dataset. Also, note that, the Isolation Kernel does not have an ideal kernel which has a closed form expression, as in the case of the Gaussian Kernel, but it is a data-dependent estimation from a dataset.}

Let $\rho(\mathbf{x})$ denote the density of $P_D$ at point $\mathbf{x}$. The unique aspect of the Isolation Kernel, compared with other kernels, is given as follows:

\begin{lem} \cite{IsolationKernel-AAAI2019}
\label{lem_characteristic}
$\forall \mathbf{x}, \mathbf{y} \in \mathcal{X}_\mathsf{S}$ (sparse region) and $\forall \mathbf{x}',\mathbf{y}' \in \mathcal{X}_\mathsf{T}$ (dense region) such that $\forall_{\mathbf{z}\in \mathcal{X}_\mathsf{S}, \mathbf{z}'\in \mathcal{X}_\mathsf{T}} \ \rho(\mathbf{z})<\rho(\mathbf{z}')$,
the Isolation Kernel $\kappa_I$ has the unique characteristic that for $\ell_p(\mathbf{x}-\mathbf{y})\ =\ \ell_p(\mathbf{x}'- \mathbf{y}')$ implies that \textcolor{black}{in expectation}:
\[ P(\mathbf{x},\mathbf{y}\in \theta[\mathbf{z}]) > P(\mathbf{x}',\mathbf{y}'\in \theta[\mathbf{z}'])  \equiv  \kappa_I( \mathbf{x}, \mathbf{y}\ |\ D) > \kappa_I( \mathbf{x}', \mathbf{y}'\ |\ D)
\]

\end{lem}

In simple terms,  the unique characteristic of the Isolation Kernel is: {\bf two points in a sparse region \textcolor{black}{are expected to be} more similar than two points of equal inter-point distance in a dense region}. 

The required property of the space partitioning mechanism is to produce large partitions in a sparse region and small partitions in a dense region in order to yield the above unique characteristic \cite{ting2018IsolationKernel}.

\vspace{3mm}
\noindent
\textbf{Feature map of the Isolation Kernel}.

\textcolor{black}{
Let the Isolation Kernel be implemented using $\psi$ isolating partitions \cite{IDK} for each partitioning from a sample $\mathcal{D}$ of $\psi$ points. 
Given a partitioning $H_i$, let feature $\phi_i(\mathbf{x})$ be a $\psi$-dimensional binary column vector representing all partitions $\theta_j \in H_i$, $j=1,\dots,\psi$.
The $j$-component of the vector due to $H_i$ is:
$\phi_{ij}(\mathbf{x})=\mathds{1}(\mathbf{x}\in \theta_j\ |\ \theta_j\in H_i)$. Given $t$ partitionings, $\phi(\mathbf{x})$ is the concatenation of $\phi_1(\mathbf{x}),\dots,\phi_t(\mathbf{x})$.
\begin{definition}
	\label{def:featureMap}
	The Isolation Kernel $\kappa_I$ has no closed form expression and has a feature map $\phi: \mathbf{x}\rightarrow \mathbb \{0,1\}^{t\times \psi}$,  and it is expressed in terms of $\phi$ as:
\[
\kappa_I(\mathbf{x},\mathbf{y}\ |\ D)   \approx  \frac{1}{t} \left< {\phi}(\mathbf{x}), {\phi}(\mathbf{y}) \right>
\]
\end{definition}
The Isolation Kernel is a positive definite kernel and its feature map is a reproducing kernel Hilbert space (RKHS)  because its Gram matrix  is full rank as $\phi({\bf x})$ for all points ${\bf x} \in D$ are mutually independent (see \cite{IDK} for details). \\
Two possible isolating partitioning mechanisms are hyperspheres \cite{IDK} and Voronoi Diagrams \cite{IsolationKernel-AAAI2019} that can be used to build the Isolation Kernel. When hyperspheres are used,  the radius of each hypersphere centered at $\mathbf{z}$ is the distance between $\mathbf{z}$ and its nearest neighbor in $\mathcal{D}\setminus \{\mathbf{z}\}$; any $\mathbf{x} \in \mathbb{R}^d$ falls into one of the $\psi$ hyperspheres or none; and $ 0 \le \parallel {\phi}(\mathbf{x}) \parallel\ \le \sqrt{t}$.\\
When Voronoi Diagrams are used, each $\mathbf{z} \in \mathcal{D}$ is at the center of a Voronoi cell; any $\mathbf{x} \in \mathbb{R}^d$ must fall into one of the $\psi$ Voronoi cells; and  $\parallel {\phi}(\mathbf{x}) \parallel\ = \sqrt{t}$.\\
The $\psi$ parameter of IK is equivalent to the bandwidth parameter of Gaussian Kernel: the larger $\psi$ is the sharper the kernel distribution. Increasing $t$ improves the kernel estimation.
}

\subsection{IK produces clusters of same cohesiveness}
\label{sec-cohesiveness}

Here we first provide definitions of cluster and cluster cohesion, and then show that the kernel $\kappa$ used has a critical  impact on cluster cohesion.

\begin{definition}
    $\kappa_\tau$-connected:
    Two points $\mathbf{x},\mathbf{y}$ in a dataset $D$  are $\kappa_\tau$-connected points if there is a chain of points $\mathbf{z}_1,\dots,\mathbf{z}_w$, where $\mathbf{z}_1=\mathbf{x}, \mathbf{z}_w=\mathbf{y}$ such that $\kappa(\mathbf{z}_i,\mathbf{z}_{i+1}) > \tau$ and $\mathbf{z}_i \in D$ for all $i \in [1,w-1]$; and $\tau \in [0,1)$.
   \label{def-tau-connection}
\end{definition}

\begin{definition}
   A $\tau$-cohesive cluster $C^\tau$ with respect to kernel $\kappa$ in  $D$ is a $\kappa_\tau$-connected component, where
   $\forall \mathbf{x}, \mathbf{y} \in C^\tau,\ \mathbf{x} \mbox{ and } \mathbf{y} \mbox{ are } \kappa_\tau\mbox{-connected}$.
   \label{def-tau-cohesion}
\end{definition}
\textcolor{black}{Note that Definitions \ref{def-tau-connection} and \ref {def-tau-cohesion} are an adaptation of the density-connected cluster definition used in DBSCAN \cite{ester1996density} in terms of a kernel instead of a distance function. This enables the use of a notion of cluster cohesiveness that is not used previously, and it is given below:}



\begin{definition}
    Two $\tau$-cohesive clusters $C^\tau_i$ and $C^\tau_j$ have the same cohesiveness if 
     $\bar{S}_\kappa(C^\tau_i) = \bar{S}_\kappa(C^\tau_j) > \tau$, 
     where $\bar{S}_\kappa(C^\tau) = \displaystyle \frac{2}{|C^\tau|(|C^\tau|-1)} \sum_{\mathbf{x}_\imath,\mathbf{x}_\jmath\in C^\tau,\imath>\jmath} \kappa(\mathbf{x}_\imath,\mathbf{x}_\jmath)$ is the average similarity of all points $\mathbf{x}$  in a $\tau$-cohesive cluster $C^\tau$.

    


    \label{def_same-tau-cohesion}
\end{definition}


\textcolor{black}{Given a dataset having two clusters $C^\tau_\beta$ and $C^\tau_\alpha$ of varied densities such that $\rho(\mathbf{x}) > \rho(\mathbf{y})\ \forall \mathbf{x} \in C^\tau_\beta,  \forall \mathbf{y} \in C^\tau_\alpha$.}  
If $\kappa$ is a Gaussian kernel\footnote{Note that Definitions \ref{def-tau-cohesion} and \ref{def_same-tau-cohesion} can be similarly defined with respect to a distance function rather than a kernel. The conclusion of $\bar{S}(C^\tau_\beta) > \bar{S}(C^\tau_\alpha)$ is the same.}, 
we have the following proposition:
\begin{proposition}
\textcolor{black}{If $\kappa$ is a Gaussian Kernel, $0 < [\bar{S}_\kappa(C^\tau_\beta) - \bar{S}_\kappa(C^\tau_\alpha)] < \hat{\tau}$ for all settings of $\tau < \hat{\tau}$; and  $[\bar{S}_\kappa(C^\tau_\beta) - \bar{S}_\kappa(C^\tau_\alpha)] =  \bar{S}_\kappa(C^\tau_\beta) > \hat{\tau}$, for all $\tau \ge \hat{\tau}$,
where $\hat{\tau}$ be the smallest high $\tau$ setting such that $\forall \tau \ge \hat{\tau}, C^\tau_\alpha = \emptyset$; and a $\tau < \hat{\tau}$ yields $C^\tau_\alpha \ne \emptyset$.}
\label{Lemma-cohesiveness-Gaussian}
\end{proposition}
\begin{proof}
\textcolor{black}{
Gaussian kernel is defined as: $\kappa(\mathbf{x}, \mathbf{y}) = \exp\left(-\frac{\|\mathbf{x} - \mathbf{y}\|^2}{2\sigma^2}\right)$, where $\sigma$ is the bandwidth.
Assume that each of the two clusters is isotropic with a density maximum. 
For a range of high $\tau$ values, $C^\tau_\alpha = \emptyset$ and $C^\tau_\beta \ne \emptyset$ due to the fact that there is a huge difference in densities between the two modes of the clusters. Thus, $\bar{S}_\kappa(C^\tau_\beta) > \tau$ and $\bar{S}_\kappa(C^\tau_\alpha) = 0$. Given that $\hat{\tau}$ is the smallest high $\tau$ setting such that $C^\tau_\alpha = \emptyset$. Then, for all $\tau \ge \hat{\tau}$, $[\bar{S}_\kappa(C^\tau_\beta) - \bar{S}_\kappa(C^\tau_\alpha)] =  \bar{S}_\kappa(C^\tau_\beta) > \hat{\tau}$. \\
For all $\tau < \hat{\tau}$ where $C^\tau_\alpha \ne \emptyset$,\\
(i) $[\bar{S}_\kappa(C^\tau_\beta) - \bar{S}_\kappa(C^\tau_\alpha)] < \hat{\tau}$ because $\bar{S}_\kappa(C^\tau_\beta) < \hat{\tau}$ and $\bar{S}_\kappa(C^\tau_\alpha) < \hat{\tau}$; and\\ (ii) $[\bar{S}_\kappa(C^\tau_\beta) - \bar{S}_\kappa(C^\tau_\alpha)] > 0$ because there exists at least a point-pair $\mathbf{y}, \mathbf{y}' \in C^\tau_\alpha$ for every point-pair $\mathbf{x}, \mathbf{x}' \in C^\tau_\beta$  such that $\|\mathbf{x} - \mathbf{x}'\| < \|\mathbf{y} - \mathbf{y}'\|$ and $\kappa(\mathbf{x}, \mathbf{x}') > \kappa(\mathbf{y}, \mathbf{y}')$, then the average similarity $\bar{S}(C^\tau_\beta)$ in the dense region is higher than the average similarity $\bar{S}(C^\tau_\alpha)$ in the sparse region.
}
\end{proof}

\vspace{3mm}
In contrast, the proposition for the Isolation Kernel is: 
\begin{proposition}
If $\kappa$ is the Isolation Kernel implemented using Voronoi Diagrams, there exists some $\psi$ setting such that $\bar{S}_\kappa(C^\tau_\beta) \approx \bar{S}_\kappa(C^\tau_\alpha)$ for a range of $\tau$ settings, irrespective of densities, shapes and sizes of the clusters.
\label{Lemma-cohesiveness}
\end{proposition}

\begin{proof}
\textcolor{black}{A recent theoretical result on Voronoi cells $\theta$ in Voronoi Diagram $H$ provides the basis for the proof. Let $\mathbf{x}$ be i.i.d. drawn from any probability distribution $P$ on its support $\Omega$. Devroye et al. \cite{devroye2017measure} have shown that the $P$-measure of the Voronoi cell $\theta(\mathbf{x})$, which is equivalent to probability ${P}(\mathbf{x} \in \theta)$ under $\mathbf{x}\sim P$ on $\Omega$, asymptotically converges to $\frac{1}{\psi}$ as $\psi \rightarrow \infty$, and its variance
asymptotically converges to 0. This property
is independent of $\mathbf{x}$ and $P$. In other words, as the  Voronoi cells have approximately equal probability for some $\psi$, independent of $\mathbf{x}$ and $P$, then the similarity computed by Isolation Kernel $\kappa$ of two points in either the dense cluster $C^\tau_\beta$ or the sparse cluster $C^\tau_\alpha$ has approximately the same similarity. This provides the proof.}
\end{proof}

The intuition is that the Voronoi Diagrams used to construct the Isolation Kernel are built based on nearest neighbors in a sample of $\psi$ points. Assuming that these points are representative samples of the given dataset $D$, then a point $\mathbf{x}$ and its nearest neighbor $\eta_\mathbf{x}$ in $D$ are expected to fall into a same Voronoi Cell with approximately the same probability for all pairs of ($\mathbf{x}$,$\eta_\mathbf{x}$) in $D$, irrespective of the densities, shapes and sizes in different regions.

The property that $\kappa(\mathbf{x},\eta_\mathbf{x}) \approx \kappa(\mathbf{y},\eta_\mathbf{y})\ \forall \mathbf{x} \in C^\tau_\beta,\  \forall \mathbf{y} \in C^\tau_\alpha$ is due to the space partitioning mechanism employed in the Isolation Kernel which produces large partitions in a sparse region and small partitions in a dense region \cite{ting2018IsolationKernel}, i.e., a point and its most similar neighbor in any $\tau$-cohesive cluster have approximately the same probability of falling into a same partition, independent of the densities of the clusters. As a result, there is no bias towards dense cluster(s).

\subsection{IK yields $\tau$-cohesive clusters that are representative samples of all clusters}
\label{sec-representative-sample}

On a dataset of sparse cluster $C_\alpha$ and dense cluster $C_\beta$ that are in close proximity (the actual condition is provided in Section \ref{sec-second-condition}), Propositions \ref{Lemma-cohesiveness-Gaussian} and \ref{Lemma-cohesiveness} can be interpreted as:
\begin{itemize}
    \item  $\tau$-cohesive cluster $C^\tau$  is a representative sample of cluster $C$ if there exists some $\tau$  such that $P_{C^\tau} \approx P_{C}$;
    \item the property $\bar{S}_\kappa(C^\tau_\beta) > \bar{S}_\kappa(C^\tau_\alpha)$ implies that no $\tau$ exists such that $P_{C^\tau} \approx P_{C}$ for both clusters $C_\beta$ and $C_\alpha$;
    \item the property $\bar{S}_\kappa(C^\tau_\beta) \approx \bar{S}_\kappa(C^\tau_\alpha)$ implies that there exists some $\tau$ such that $P_{C^\tau} \approx P_{C}$ for both clusters $C_\beta$ and $C_\alpha$.
\end{itemize}

\begin{figure}[b]   
    \centering
    \subfloat[GK: $C_\beta^\tau$ \& $C_\alpha^\tau$ at $\tau=0.99$]{\includegraphics[width=.3\textwidth]{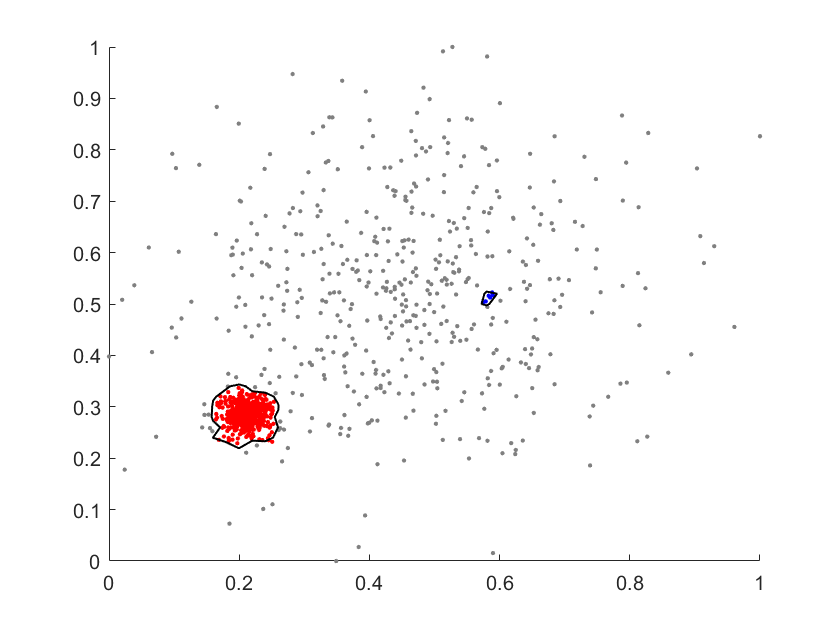} }
    \hspace{0.5cm}
    \subfloat[GK: $C_\beta^\tau$ \& $C_\alpha^\tau$ merge into one at $\tau=0.81$]{\includegraphics[width=0.3\textwidth]{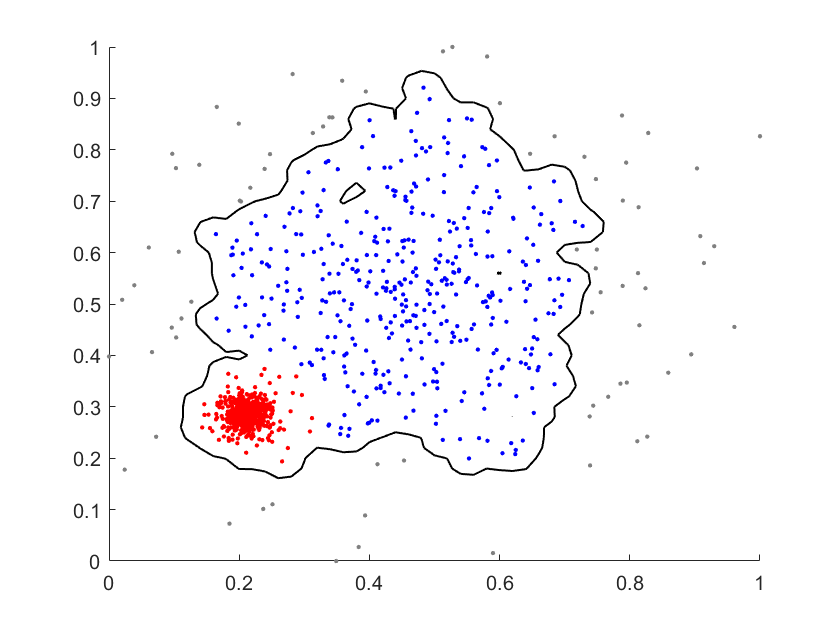}}
     \vspace{-3mm}
    \hspace{0.5cm}

    \subfloat[IK: $C_\beta^\tau$ \& $C_\alpha^\tau$ at $\tau=0.7$]{\includegraphics[width=0.28\textwidth]{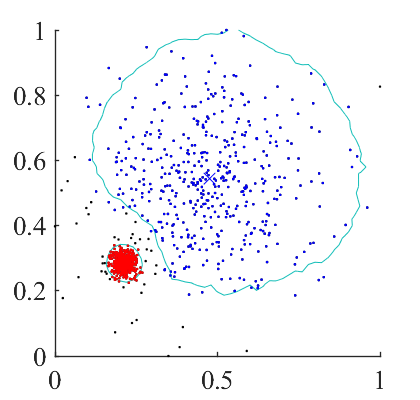}}
    \hspace{0.8cm}
    \subfloat[IK: $C_\beta^\tau$ \& $C_\alpha^\tau$ merge into one at $\tau=0.35$]{\includegraphics[width=0.28\textwidth]{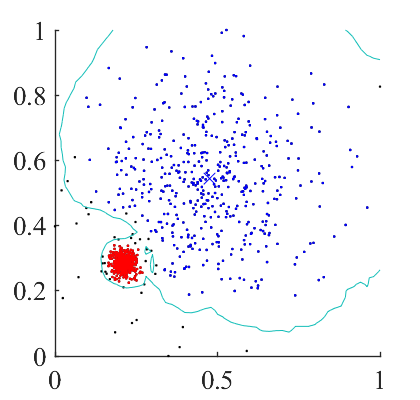}}\\ 
    \subfloat[GK]{\includegraphics[width=0.3\textwidth]{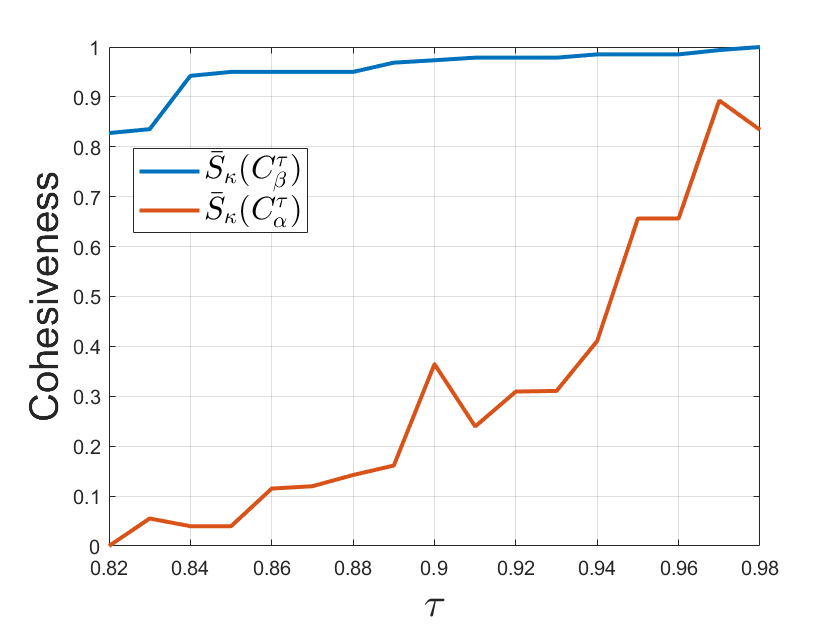}}
     \hspace{0.5cm}
   \subfloat[IK]{\includegraphics[width=0.3\textwidth]{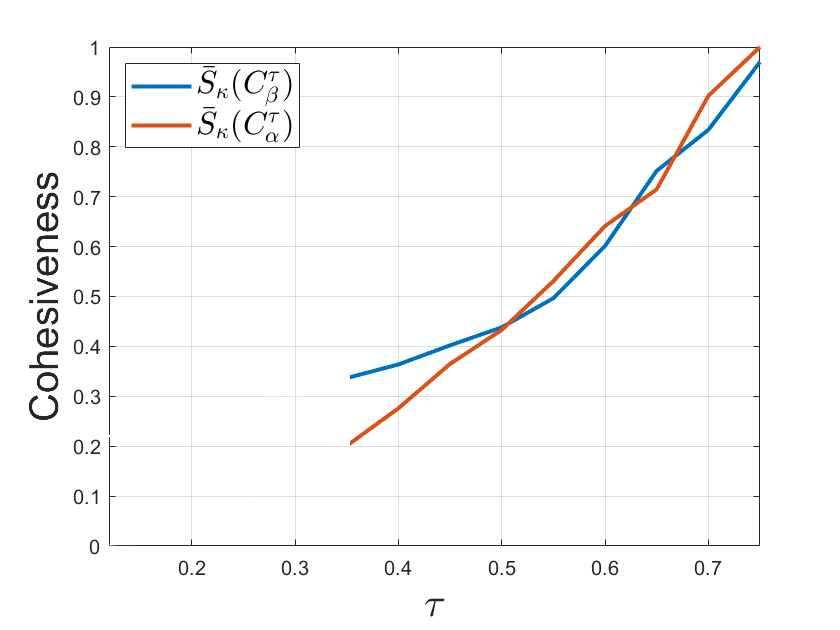}}
    
\caption{Illustrations of the boundaries of $C^\tau_\beta$ (dense) and $C^\tau_\alpha$ (sparse) clusters due to different $\tau$ settings.  \textcolor{black}{Note that clusters $C^\tau$ with  high $\tau$ cover smaller regions than those with low $\tau$: (a) The entire dense cluster is covered when using the Gaussian Kernel with $\tau=0.99$; and (b) both the dense and sparse clusters are merged when $\tau=0.81$. (c) This example denotes the lowest $\tau=0.7$ using the Isolation Kernel to cover the entire dense cluster, while lowering $\tau=0.35$ merges the two clusters is shown in (d).}  The plots of cohesiveness $\bar{S}_\kappa(C^\tau)$ versus $\tau$ for dense cluster $C_\beta$ \& sparse cluster $C_\alpha$ are shown in (e) \& (f). The cohesiveness is min-max normalized over values derived from all $\tau$ values.}   

  \label{fig:section3}
\end{figure}

The property $\bar{S}_\kappa(C^\tau_\beta) > \bar{S}_\kappa(C^\tau_\alpha)$, due to the Gaussian Kernel, could not produce representative samples with any $\tau$ for both clusters, and it can be seen in the following two scenarios: (i) a $\tau$, which produces $C^\tau_\beta$ that is a representative sample, yields $C^\tau_\alpha$ which is far short of a representative sample of the sparse cluster. An example is shown in Figure \ref{fig:section3}(a). This is because, for $C^\tau_\alpha$ to be the representative sample, it requires a much lower $\tau$; and (ii) a $\tau$, that yields $C^\tau_\alpha$ to be the representative sample for the sparse cluster, produces an over-sampled $C^\tau_\beta$ which encroaches into the region of the sparse cluster bordering the dense cluster. An example is shown in Figure \ref{fig:section3}(b).

As the Isolation Kernel has the property $\bar{S}_\kappa(C^\tau_i) \approx \bar{S}_\kappa(C^\tau_j)\ \forall i \ne j$, irrespective of cluster densities, shapes and sizes of clusters, it ensures that all $C^\tau$ produced are representative samples for all clusters. The example in Figure \ref{fig:section3}(c) shows that, using IK, there exists some setting of $\tau$ such that $P_{C^\tau} \approx P_{C}$ for both clusters. Notice the huge difference between the $C^\tau_\alpha$'s produced from the GK and IK in Figure \ref{fig:section3}(a) and Figure~\ref{fig:section3}(c), even though the $C^\tau_\beta$'s produced are approximately the same.

The properties $\bar{S}_\kappa(C^\tau_\beta) > \bar{S}_\kappa(C^\tau_\alpha)$ and $\bar{S}_\kappa(C^\tau_\beta) \approx \bar{S}_\kappa(C^\tau_\alpha)$, due to the Gaussian Kernel and Isolation Kernel, are shown in Figure \ref{fig:section3}(e) and Figure \ref{fig:section3}(f), respectively.

\textcolor{black}{In summary, using mass distribution allows large proportions of both the dense and sparse clusters to be discovered at some $\tau$ setting; yet, no such setting of $\tau$ exists to find the two equally-good clusters when density distribution is used, without the dense cluster encroaching into the region of the sparse cluster---a manifestation of the high-density bias.}

\section{Definition of Isolation-induced Mass}
\label{sec_Mass}

We now introduce the concept of mass, derived from space partitioning via an isolating mechanism, as used in the Isolation Kernel.

\textcolor{black}{
\begin{definition} Isolation-induced Mass of a point $\mathbf{x} \in \mathbb{R}^d$ with respect to a given dataset $D$ is defined as
	the expected probability mass $\mathcal{P}(\theta)$ of the isolating partition $\theta \in H$ in which $\mathbf{x}$ falls:
	\begin{eqnarray}
\mathbbm{m}(\mathbf{x}|D) = {\mathbb E}_{\theta \sim \mathcal{H}(\mathbf{x}|D)} [\mathcal{P}(\theta)],
		\label{eqn_mass}
	\end{eqnarray}
\end{definition}
\noindent
where the expectation is taken over all partitions $\theta$ of all partitionings $H \in \mathbb{H}_\psi(D)$ which cover $\mathbf{x}$, i.e., $\mathcal{H}(\mathbf{x}|D) := \{\theta\ |\ \mathbf{x} \in \theta \in H,\ \forall H \in \mathbb{H}_\psi(D) \}$; and $\theta$ is a shorthand for $\theta[\mathbf{z}]$, where $\mathbf{z} \in \mathcal{D} \subset D$, $\psi=|\mathcal{D}|$.\\
The probability mass $\mathcal{P}(\theta)$ can be estimated from the partitionings and $D$ as:\\ $\mathcal{P}(\theta) = \frac{1}{|D|}\sum_{\mathbf{y} \in D} \mathds{1}(\mathbf{y}\in \theta)$.\\
In practice, $\mathbbm{m}(\mathbf{x}|D)$ is estimated using a finite number of partitionings $H_i, i=1,\dots,t$, where each $H_i$ is created using randomly subsampled $\mathcal{D}_i \subset D$:
\begin{eqnarray}
\mathbbm{m}(\mathbf{x}|D) =   \frac{1}{t} \sum_{i=1}^t  \mathcal{P}(\theta(\mathbf{x}) \in H_i) =  \frac{1}{t|D|} \sum_{i=1}^t \sum_{\mathbf{y} \in D} \mathds{1}(\mathbf{y}\in \theta(\mathbf{x}) \in H_i),  
 \label{eqn_mass2}
\end{eqnarray}
\noindent
where $\theta(\mathbf{x})$ is one of the $\psi$ partitions in $H_i$ which covers $\mathbf{x}$.\\
Recall from the feature map $\phi$ of the Isolation Kernel stated in Section \ref{sec_prelim-IK} is:\\ ${\phi}(\mathbf{y}) = [\phi_1(\mathbf{y}),\dots,\phi_t(\mathbf{y})]$; and the $j$-component of $\phi_i(\mathbf{y})$ is:
$\phi_{ij}(\mathbf{y})=\mathds{1}(\mathbf{y}\in \theta_{ij}\ |\ \theta_{ij}\in H_i)$ for $j \in [1,\psi]$.  Let $ {\Phi}(D)= \frac{1}{|D|}\sum_{\mathbf{y} \in D} {\phi}(\mathbf{y}) = \frac{1}{|D|}\sum_{\mathbf{y} \in D} [\phi_1(\mathbf{y}),\dots,\phi_t(\mathbf{y})]$. Then, ${\Phi}(D)$ can be expressed in terms of $\mathcal{P}(\theta_{ij})$ as follows:
\[
{\Phi}(D)= [ \mathcal{P}(\theta_{11}),\dots,\mathcal{P}(\theta_{1\psi}),
\mbox{\textasciidieresis \textasciidieresis \textasciidieresis}, \mathcal{P}(\theta_{t1}),\dots,\mathcal{P}(\theta_{t\psi}) ].
\]
Now, $\mathbbm{m}_{\kappa}(\mathbf{x}|D)$  in Equation \ref{eqn_mass2} can be re-expressed in terms of ${\Phi}(D)$ or $\kappa$ as follows: 
\begin{eqnarray} 
\mathbbm{m}_{\kappa}(\mathbf{x}|D) & = & \frac{1}{t} \left< {\phi}(\mathbf{x}), {\Phi}(D) \right> \label{eqn_mass-estimator} \\
 & = & \frac{1}{t|D|}\sum_{\mathbf{y} \in D} \left< {\phi}(\mathbf{x}), {\phi}(\mathbf{y}) \right> = \frac{1}{|D|}\sum_{\mathbf{y} \in D} \kappa(\mathbf{x},\mathbf{y}\ |\ D). 
\label{eqn_mass-estimator2}
\end{eqnarray}
Note that mass estimator $\mathbbm{m}$, in Equation \ref{eqn_mass} or \ref{eqn_mass2}, does not need to be defined in terms of Isolation Kernel. The definition based on the Isolation Kernel not only provides a richer interpretation, but also significantly reduces its time complexity from $\mathcal{O}(n)$ by using $\kappa_I$ to $\mathcal{O}(1)$ by using its feature map via Eq (\ref{eqn_mass-estimator}) for each point estimation after an one-off $\mathcal{O}(n)$ to compute $\Phi(D)$. Note that the time cost is linear to $t\psi$. These terms are dropped from the time complexity since they are constant.\\
Mass estimation of a point $\mathbf{x}$ with respect to a distribution $P_D$ from which a data sample $D$ is generated can be summarized as follows: (i) It describes the expected probability mass of isolating partitions in which $\mathbf{x}$ falls. (ii) It measures the similarity of $\mathbf{x}$ with respect to $P_D$ via a dot-product of $\phi(\mathbf{x})$ and $\Phi(D)$ in RKHS. 
\\
${\Phi}(D)$ can be viewed as a kernel mean feature mapped point which represents an unknown (mass) distribution $P_D$ from which $D$ is a sample. The dot product of  ${\phi}(\mathbf{x})$ and ${\Phi}(D)$ can be interpreted as the similarity between the feature mapped points of $\mathbf{x}$ and the unknown distribution represented by $D$. 
This interpretation gives a mass distribution of a given dataset, estimated by $\mathbbm{m}_{\kappa}$, having characteristics which are more useful for data-driven analyses than density distribution. This is presented in the next two subsections; and the corresponding density counterparts are given in the third subsection. The advantage in terms of runtime is presented in Section \ref{sec-MMC}.
}

\subsection{Mass estimation with respect to a cluster in $D$}
\label{sec-mass-cluster}
The above mass estimator is derived with respect to the given dataset $D$. In some applications, one would like to estimate the mass of a point with respect to a cluster $C$ in $D$. This mass estimator is defined as:
\begin{equation}
\mathbbm{m}_{\kappa}(\mathbf{x}|C \subset D) = \frac{1}{|C|}\sum_{\mathbf{y} \in C} \kappa(\mathbf{x},\mathbf{y}|D) =  \frac{1}{t} \left< {\phi}(\mathbf{x}|D), {\Phi}(C|D) \right> 
\label{eqn-mass-wrt-cluster}
\end{equation}
where $ {\Phi}(C|D)= \frac{1}{|C|}\sum_{\mathbf{y} \in C} {\phi}(\mathbf{y}|D)$.

\subsection{Mass distribution with respect to all clusters in $D$}

\begin{definition}
Given clusters $C_j, j=1,\dots,k$ in a dataset $D$ and $\mathbbm{m}_{\kappa}(\cdot|D)$ derived from $D$, the estimated  mass distribution $\forall \mathbf{x} \in \mathbb{R}^d$ is defined as: 
\[
\widetilde{\mathbbm{m}}_{\kappa}(C_j \subset D, j=1,\dots,k) = \max_j \mathbbm{m}_{\kappa}(\mathbf{x}|C_j \subset D).
\]

\label{def:mass_dist}
\end{definition}

The distribution is analogous to the density distribution estimated by, e.g.,  a kernel density estimator (KDE); except that the clusters in a dataset must be provided. In other words, the  $\widetilde{\mathbbm{m}}_{\kappa}$ mass distribution describes the data distribution in terms of the given clusters in the dataset.

To simplify the notations, $D$ is dropped in $\mathbbm{m}_{\kappa}$ hereafter when the context is clear.

\subsection{Density distribution with respect to all clusters in $D$}

The corresponding density distribution and density estimator are given as follows:
\[
\widetilde{\mathbbm{f}}_{\kappa}(C_j, j=1,\dots,k) = \max_j \mathbbm{f}_{\kappa}(\mathbf{x}|C_j),
\]
and
\[
\mathbbm{f}_{\kappa}(\mathbf{x}|C) = \frac{1}{|C|}\sum_{\mathbf{y} \in C} \kappa(\mathbf{x},\mathbf{y}).
\]
\noindent
where $\kappa(\cdot,\cdot)$ is a data independent kernel such as the Gaussian Kernel.

To simplify the notation, we use $P_C$ hereafter to denote either the mass distribution  or the density distribution of cluster $C$  derived from $\mathbbm{m}_{\kappa}(\cdot|C)$ or $\mathbbm{f}_{\kappa}(\cdot|C)$, respectively, when the context is clear.





\section{A natural way to define clusters in a dataset}

The above discussion leads to a simple and natural way to define a cluster in a given dataset, independent of the clustering procedure. A cluster having \emph{highest-mass points} is cohesive, and it adheres to the shape of the cluster as it  appears in the data space.

\begin{definition}
    The set of highest-mass clusters $\mathbb{C} = \{C_i, i=1,\dots,k\}$, discovered in a given dataset $D$, where \textcolor{black}{$k \ll |D|$}, is defined such that  
    every point $\mathbf{x} \in C$ has the highest mass with respect to the mass distribution of $C$: \[\forall \mathbf{x} \in C_i, \argmax_j \mathbbm{m}_{\kappa}(\mathbf{x}|C_j ) = i.
    \]
    \label{def-highest-mass-cluster}
\end{definition}

\begin{definition}
To produce the highest-mass clusters from a given dataset $D$, the objective function is to maximize the total mass of \textcolor{black}{all $k$ clusters in $\mathbb{C}$}
as follows:
\begin{equation}
    M(D) = \max_\mathbb{C} \sum_{C \in \mathbb{C}} \sum_{\mathbf{x} \in C} \mathbbm{m}_{\kappa}(\mathbf{x}|C ) 
    \label{eqn-objective-function}
\end{equation}
\label{def-objective}
\end{definition}




It is interesting to note that, though highest-density clusters could be defined similarly, the dense clusters would tend to encroach into the regions of sparse clusters because the density distribution has a natural bias towards the dense clusters, i.e., the dense cluster's $\tau$-cohesiveness is significantly larger than that of sparse cluster, as stated in Section \ref{sec-cohesiveness}.

The current cluster definitions in density-based clustering algorithms depend on (i) a (data independent) distance function used in a density estimator, and (ii) a point-to-point linking procedure to form the final clusters, as in DBSCAN and DP.  
In contrast, the definition of the proposed highest-mass clusters depends on a data dependent kernel and its kernel mass estimator only, where a point-to-point linking procedure plays a role  to form the initial clusters only, but not the final clusters. Here the point-to-point linking procedure applies to a subset of $D$, not the entire dataset (see the details in the next section).

Another key difference is the objective function. The highest-mass clusters lead directly to the objective function, stated in Definition \ref{def-objective}. Yet, none of the  existing density-based clustering algorithms (e.g., \cite{ester1996density,DENCLUE,rodriguez2014clustering,
DP_jain,chen2018local}),  that we are aware, have an objective function\footnote{\textcolor{black}{A recent paper \cite{DBSCAN-unify-KDD2023} claims that DBSCAN achieves an objective function in terms of the density-connectivity distance (dc-dist), i.e., it aims to find the minimum number of clusters such that the maximum dc-dist within any cluster is $\epsilon$. However, it is not the objective of the original DBSCAN algorithm. One must perform an external parameter search over different values of $\epsilon$, running DBSCAN multiple times, to achieve the stated objective. In other words, the original DBSCAN does not function as an optimization algorithm meant to do in order to achieve the stated objective. It is misleading to claim that the original DBSCAN achieves this objective.}}.

Using the concepts of $\tau$-cohesive clusters, highest-mass clusters and its associated objective function (Definitions \ref{def-tau-cohesion}, \ref{def-highest-mass-cluster} \& \ref{def-objective}, respectively), we create a new clustering algorithm called Mass-Maximization Clustering (MMC), which is described in the next section.


\section{Mass-Maximization Clustering}
\label{sec-MMC}

The proposed MMC has three steps as shown in Algorithm \ref{alg:MMC}, where $\kappa$ is the Isolation Kernel. The first step employs the more stringent $\tau$-cohesive clusters to derive the initial clusters from a subset $D_s$ of the given dataset $D$. In our implementation, each initial cluster is obtained as $\kappa_\tau$-connected component by using a standard  function in Matlab called `conncomp' \footnote{\textcolor{black}{This function outputs the number of connected components in a graph, where the set of input data points is treated as a graph with $\kappa_\tau$-connected edges.}}.  The second step assigns each point in $D$ to its most similar initial clusters 
to maximize the total mass of all clusters. 

 \begin{algorithm}[h]
  \caption{Mass-Maximization Clustering (MMC) }
   \label{alg:MMC}
\begin{algorithmic}[1]
\Require $D$ - dataset, $k$ - number of clusters,  $s$ - sample size,\ \
  $\tau$ - similarity threshold
\Ensure $\mathbb{C} = \{ C_1, \dots, C_k \}$
\State Produce largest $k$  $\tau$-cohesive clusters $Q^\tau_i$ (Definition \ref{def-tau-cohesion}) from a subset $D_s \subset D$:
\Statex $\ \ \ \ \ \ \forall \mathbf{x}, \mathbf{y} \in Q^\tau_i \subset D_s,\ \mathbf{x} \mbox{ and } \mathbf{y} \mbox{ are } \kappa_\tau\mbox{-connected}$,
$\forall_{i\in [1,k]}$.\;
\State Assign points in $D$ based on mass-maximization with respect to $Q^\tau_i$ (Definitions \ref{def-highest-mass-cluster} \& \ref{def-objective}):
\Statex $\ \ \ \ \ \ C'_j=\{\mathbf{x}\in D\ |\ \mathop{\argmax}\limits_{i \in [1,k]} \mathbbm{m}_{\kappa}(\mathbf{x}|Q^\tau_i )=j \},
\forall_{j\in [1,k]}.$\;
\State Post-processing to refine $\mathbb{C} = \{ C'_1, \dots, C'_k \}$ to improve the objective:
\Statex $\ \ \ \ \ \ \displaystyle M(D) = \max_\mathbb{C} \sum_{C \in \mathbb{C}} \sum_{\mathbf{x} \in C} \mathbbm{m}_{\kappa}(\mathbf{x}|C )$.
\State \Return $\mathbb{C} = \{ C_1, \dots, C_k \}$\;

\end{algorithmic} 
 \end{algorithm}

The set of  clusters $\mathbb{C} = \{ C'_1, \dots, C'_k \}$ at the end of the second step consists of highest-mass clusters, as a result of maximizing the total mass of all clusters, as stated in Equation (\ref{eqn-objective-function}), based on $\tau$-cohesive clusters $Q^\tau_i$.

The final third step is simply to tweak the clusters found in the second step to further improve the point assignments, if there is still room for improvement. This is conducted based on clusters $C'_j$ obtained in the second step. 

Propositions \ref{Lemma-cohesiveness-Gaussian} \& \ref{Lemma-cohesiveness} and
Section \ref{sec-representative-sample} ensure that step 1 of MMC produces $\tau$-cohesive clusters $Q^\tau$ which are representative samples for all clusters.   
Note that  $\tau$ is a parameter, which determines the sample size of $Q^\tau$, and it must be tuned for a given dataset, as described in Section \ref{sec-representative-sample}.

Table \ref{tab:MMC-vs-DMC} provides a summary of the comparison between MMC and its density counterpart Density-Maximization Clustering (DMC), where the only difference is the use of the Gaussian Kernel instead of Isolation Kernel for $\kappa$ (yielding density estimator $\mathbbm{f}_{\kappa}$ in place of mass estimator $\mathbbm{m}_{\kappa}$).

From Propositions \ref{Lemma-cohesiveness-Gaussian} \& \ref{Lemma-cohesiveness} and
Section \ref{sec-representative-sample}, we know that step 1 of DMC could not produce  $\tau$-cohesive clusters $Q^\tau$ which are representative samples for all clusters with any $\tau$, when there are clusters of varied densities. This often yields a bias towards dense clusters in the clustering outcome. We provide the conditions under which DMC fails to identify all clusters in a dataset in the next section. 

\begin{table}[h]
\vspace{-4mm}
    \centering
     \caption{MMC versus DMC. 
$Q_\alpha$ and $Q_\beta$ denote sparse and dense clusters, respectively.}
    \begin{tabular}{c|cc}
    \toprule
      & MMC  & DMC\\ \midrule
 Kernel used  & Isolation & Gaussian\\
 Estimator used  & Mass & Density\\
The type of clusters discovered  & Highest Mass & Highest Density\\ 
 Average similarity of $\tau$-cohesive clusters (Definition \ref{def_same-tau-cohesion})  & $\bar{S}(Q^\tau_\alpha) \approx  \bar{S}(Q^\tau_\beta)$ & $\bar{S}(Q^\tau_\alpha) <  \bar{S}(Q^\tau_\beta)$\\ 
 high-density bias  & No & Yes\\ 
 Difficulty in finding all clusters of varied densities & No & Yes$^\dagger$ \\
 \bottomrule
\multicolumn{3}{l}{$^\dagger$ See some example conditions in Section \ref{sec-fail-condition}.}
    \end{tabular}
   
    \label{tab:MMC-vs-DMC}
\end{table}

\begin{table}[h]
    \vspace{-10mm}
    \centering
\caption{Time complexities. $n$: dataset size; $s$: sample size in MMC/DMC; $a$ anchor size in SGL and GLSHC. $\mathcal{O}(\cdot)$ denotes the worst-case time complexity}

\begin{tabular}{c|c|c}
    \hline
MMC \& DMC & DP, DBSCAN \& LGD & SGL \& GLSHC\\ \hline
$\mathcal{O}( n+s^2)$ &  $\mathcal{O}( n^2)$ & $\mathcal{O}( na^3)$ \\
 \hline
    \end{tabular}    
    \label{tab:time-complexities}
\end{table}

\textcolor{black}{Table \ref{tab:time-complexities} shows the time complexities of MMC/DMC in comparison with those of density-based and spectral clustering algorithms. \\
The MMC procedure has the following time complexities: Building the feature map of IK (with parameters $\psi$ \& $t$) and mapping $n$ points in $\mathbb{R}^d$ to RKHS take $\mathcal{O}(ndt\psi)$. It is linear with respect to $n$ because all the other parameters are constant. $\mathcal{O}(s^2)$ is required to produce the $\tau$-cohesive initial clusters in step 1 of MMC/DMC; step 2 has $\mathcal{O}(nk\psi t)$ to assign all points  in $D$ to $k$ clusters, so as the post-processing. Thus, the total worst-case time complexity is  $\mathcal{O}(n t\psi(k+d) + s^2)$. As all other parameters, from apart $n$, are constant, MMC/DMC has $\mathcal{O}(n)$. \\
DMC requires a preprocessing of applying the Nystr\"{o}m method \cite{Nystrom_NIPS2000} to produce the approximate feature map of Gaussian kernel before performing the conversion of the data points in the input space to RKHS. This step replaces the IK building process. The rest of the procedure is the same as MMC.
}

\section{When DMC fails to discover all clusters correctly}
\label{sec-fail-condition}

\textcolor{black}{Here we show that DMC which employs the density distributions has the high-density bias but MMC which employs the mass distributions does not, while both have the exactly the same algorithm.}


Given a dataset with $k$ clusters, in order to correctly identify all clusters based on DMC/MMC, the following criteria must be satisfied: 
\begin{enumerate} 
\item  Each cluster contains only one $\tau$-cohesive cluster $Q^\tau$ in the detected largest $k$ $\tau$-cohesive clusters.
\item All $\tau$-cohesive clusters $Q^\tau$ are representative samples for all clusters.
\end{enumerate}






Two conditions of the data distribution in which the density-cluster bias has a negative impact on the clustering outcomes of DMC are given in the next two subsections.

\subsection{First condition}
\label{sec-first-condition}

Consider a dataset which consists of two dense clusters in close proximity and a distant sparse cluster, as shown in Figure \ref{fig:chain-demo}. 


 \begin{figure}[h]
    \subfloat[Gaussian Kernel as $\kappa$]{\includegraphics[width=0.45\textwidth]{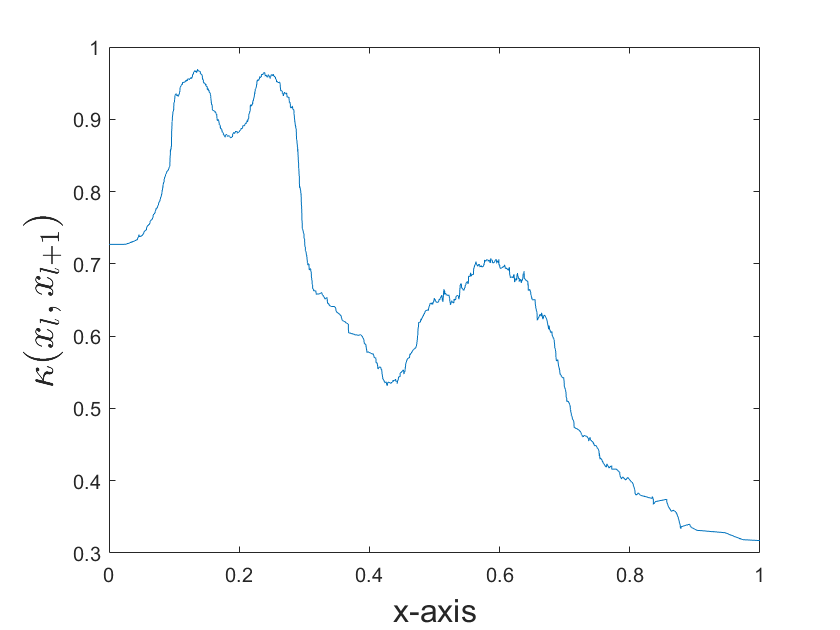}}    \subfloat[Isolation Kernel as $\kappa$]{\includegraphics[width=0.45\textwidth]{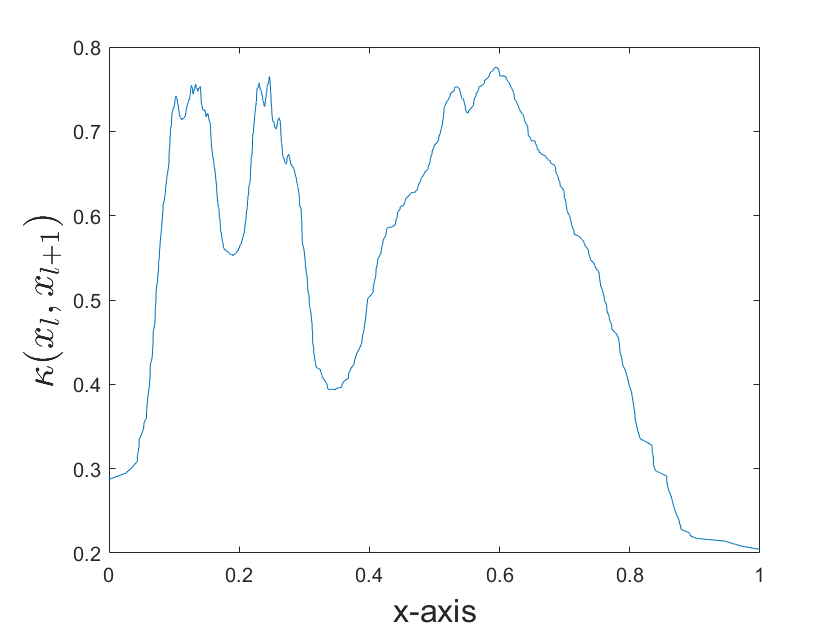}}
     \caption{An illustration of the first condition in terms of the distribution of $\kappa(\mathbf{x}_{\imath},\mathbf{x}_{\imath+1})$. Moving averages with window sizes of 150 and 25 are used to produce the distributions of GK and IK, respectively. }
     \label{fig:chain-demo}
    \end{figure}

Let ${c}_{i}$ be the highest-similarity peak of cluster $C_i$, i.e., ${c}_{i}=\argmax_{\mathbf{x}\in C_i} \sum_{\mathbf{y} \in C_i} \kappa(\mathbf{x},\mathbf{y})$; and the maximum similarity with respect to the peak of the sparse cluster $C_\alpha$ be $\hat{s}_{\alpha}= \max_{\mathbf{x}\in C_\alpha} \kappa(\mathbf{x},c_\alpha)$. Further let $E$ be the set of all $\kappa_\tau$-connected chains linking the peaks of the two dense clusters $C_{\beta_1}$ and $C_{\beta_2}$; and $\imath \lhd e$ denote the index $\imath$ in the chain ${e}= [\mathbf{x}_1,\cdots,\mathbf{x}_{\imath},\cdots,\mathbf{x}_w], \forall e \in E$.

\begin{lem}
The condition under which DMC fails to detect all three clusters of $C_\alpha$, $C_{\beta_1}$ and $C_{\beta_2}$ is:
\[
\hat{s}_{\alpha} < \max_{e \in E} \min_{\imath \lhd e} \kappa(\mathbf{x}_{\imath},\mathbf{x}_{\imath+1}). 
\]
\end{lem}

\begin{proof}
With a low setting of $\tau < \hat{s}_{\alpha}$, it is possible to obtain $Q^\tau_\alpha$ which is a good representative sample of the sparse cluster. However, this creates an oversampling of both dense clusters in such a way that only a single $Q^\tau_\beta$ emerges because the two samples, from the two clusters, have mingled into one in step 1 of DMC. With a single $Q^\tau_\beta$, the final clustering outcome is one cluster only, for the two dense clusters.

A high setting of $\tau > \hat{s}_{\alpha}$ exists such that it produces $Q^\tau_{\beta_1}$ and $Q^\tau_{\beta_2}$, i.e., the two representative samples of the two dense clusters. But non-empty $Q^\tau_\alpha$ could not be obtained.

In both cases, which are a direct outcome of the high-density bias shown in Figure~\ref{fig:chain-demo}(a), neither could identify all three clusters correctly in the dataset because no appropriate representative samples are created for all three clusters.
\end{proof}

In simple terms, the first condition is the data distribution of the three clusters such that the maximum similarity with respect to the mode of the sparse cluster is less than the maximum similarity of points at the valley between the two dense clusters.

Recall that MMC has exactly the same algorithmic procedure as DMC with the exception of using IK instead of the Gaussian Kernel. Yet, MMC has no high-density bias in step 1 and correctly identifies all the two dense clusters and one sparse cluster in the final clustering outcome. This is because MMC is able to produce a representative sample for each cluster, irrespective of their densities, as shown in Figure \ref{fig:chain-demo}(b), and it is stipulated in Section \ref{sec-representative-sample}. 

\subsection{Second condition}
\label{sec-second-condition}

\begin{lem}
    DMC always produces a clustering outcome having the dense cluster encroach into the region of the sparse cluster under the following condition:

The data distribution has a sparse cluster $C_\alpha$ overlapping with a dense cluster $C_\beta$ such that 
\[\min_{\mathbf{x} \in C_\beta} \kappa(\mathbf{x}, \eta_\mathbf{x}) \gg \min_{\mathbf{y} \in C_\alpha} \kappa(\mathbf{y}, \eta_\mathbf{y})
\]
\end{lem}
\begin{proof} We only need to show that $Q^\tau_\alpha$ and $Q^\tau_\beta$ produced in step 1 are not all representative samples of the two clusters. When $\kappa$ is the Gaussian Kernel which leads to $\bar{S}_\kappa(Q^\tau_\alpha) < \bar{S}_\kappa(Q^\tau_\beta)$,  there exists no $\tau$ which produces representative samples for both $Q^\tau_\alpha$ and $Q^\tau_\beta$, as stated in Section \ref{sec-representative-sample}. 
\end{proof}

In contrast, due to the property $\bar{S}_\kappa(Q^\tau_\alpha) \approx \bar{S}_\kappa(Q^\tau_\beta)$, MMC produces the representative samples for both clusters, irrespective of their densities (as stipulated in Section~\ref{sec-representative-sample}).

It is interesting to note that the underlying reason of DMC's failure to detect all clusters in the above two conditions (as well as other conditions which have clusters of varied densities) is that, in step 1 of DMC, \emph{no $\tau$ exists which could produce $Q^\tau$ to be a representative sample  for every cluster in a given dataset}.

\subsection{The impact of the high-density bias in DMC}
\label{sec-Density-maximization-issue}

Here we show that, under certain conditions, the high-density bias has a negative impact on  the density maximization criterion used in step 2 of DMC, re-stated as follows:

\[C'_j=\{\mathbf{x}\in D\ |\ \mathop{\argmax}\limits_{i \in [1,k]} \mathbbm{f}_{\kappa}(\mathbf{x}|Q_i )=j \},\forall_{j\in [1,k]}.
\]

For the same second condition described in Section \ref{sec-second-condition},
it is interesting to note that, even if $Q_i=C_i$ (i.e., the ground-truth cluster), the dense cluster produced $C'_\beta$ encroaches on the region of sparse cluster $C_\alpha$; resulting in $C'_\alpha$ to cover an area less than it is supposed to be. An example is shown in Figure \ref{fig:demonstration-step2}. This is the effect of the high-density bias, where a dense cluster has a stronger field of attraction that `sucks in' sparse points in the border region between the dense and sparse clusters\footnote{A study in the context of k-nearest neighbor classification revealed a similar phenomenon \cite{LMN-MLJ2019}, where the points in the sparse region, bordering the dense region, are more likely to be classified as belonging to the dense class.}.

\begin{figure}[h]
    
    \centering

    \subfloat[$Q_i$ in Step 1 is replaced with the ground-truth clusters]{\label{smp-init}\includegraphics[width=0.3\textwidth]{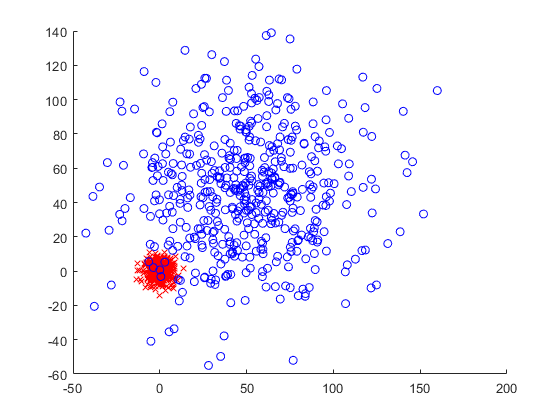}}
    \hspace{0.5cm}
    \subfloat[Clustering outcome of Step 2 based on the ground-truth clusters shown in Figure \ref{smp-init}]{\includegraphics[width=0.3\textwidth]{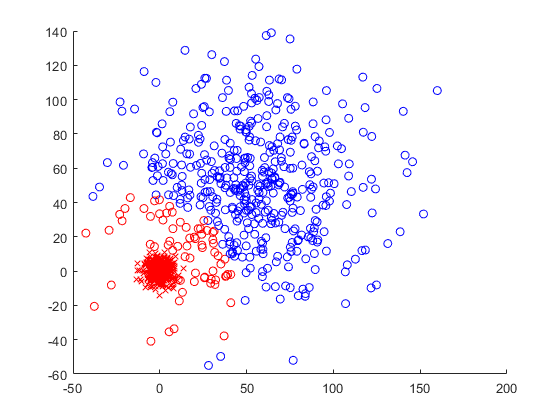}}
  \caption{An example impact of the high-density bias in step 2 of DMC (where $Q_i$ obtained in step 1 is replaced with the ground-truth clusters). }
  \label{fig:demonstration-step2}
\end{figure}

Indeed, when $Q$ is an unrepresentative sample of a cluster (as described in the last subsection), the situation becomes worse.

For the same first condition described in Section \ref{sec-first-condition}, step 2 of DMC has no problem identifying the three clusters correctly if $Q_i=C_i$, or $Q_i$ is a representative sample of $C_i$ (e.g., as discovered by using the Isolation Kernel as $\kappa$ in step 1). This is because the sparse cluster is far from the dense clusters.


\subsection*{Section summary}

DMC has the high-density bias, due to the use of density distributions. The bias has two impacts: (a) it could prevent step 1 of DMC from finding the representative samples of the to-be-discovered clusters, and (b) it could impede step 2 from assigning points to the correct clusters, when a dataset has clusters of varied densities of some condition. We have identified two such conditions in this section.

\section{Conceptual differences in clustering algorithms}

\subsection{\textcolor{black}{Conceptual differences in mass-based clustering and density-based clustering}}

It is interesting to contrast the \textbf{conceptual differences between the first two steps in terms of cluster formation} in MMC, shown in Algorithm \ref{alg:MMC}. From the algorithmic perspective, the first step employs a point-to-point linking process, typically used in existing density-based clustering algorithms (though the details differ), to form $\tau$-cohesive clusters (not final clusters). The second step, which is the main process for cluster formation, assign points by maximizing the similarity of each point with respect to a distribution (represented by a $\tau$-cohesive cluster). This process does not rely on a point-to-point linking process, but mass-maximization for all clusters (as mass distributions) via the Isolation Kernel. It is a point-to-distribution operation because $\mathbbm{m}_{\kappa}(\mathbf{x}|Q^\tau_i )$ is performed via Eq (\ref{eqn-mass-wrt-cluster}).

From the search perspective, the point-to-point linking is a \emph{point} search operation which involves a nearest neighbor search only. The mass-maximization process is a \emph{cluster-as-distribution} search operation which searches the most similar distribution out of the $k$ distributions, representing the $k$ clusters. MMC is unique among existing clustering algorithms because it uses point search in the first step and cluster-as-distribution search in the second step. Density-based clustering algorithms  require essentially point search only\footnote{
Note that density estimation is a point estimation problem that often involves a point search (e.g., the $\epsilon$-neighborhood estimator, employed in DBSCAN and DP, uses nearest neighbor search). This search is not for the purpose of clustering, but an essential computational expense.}; and k-means-based clustering algorithms (including spectral clustering) rely on cluster search only, and the cluster is not treated as a distribution, even in kernel k-means \cite{Scalable-kMeans-JMLR19,KNNKernel}.

Note that any point-to-point linking process has at least quadratic time complexity because of the need to compute pair-wise similarities or distances. MMC reduces the actual runtime because the first step only needs to be conducted on a small dataset which is sufficient to represent the data distribution of every cluster. The mass-maximization step makes use of the cluster representation to form the final cluster from the full dataset in linear time. This allows MMC to complete the entire clustering process in linear time, instead of at least quadratic time complexity of existing density-based clustering algorithms.

In addition, the definition of the density-connected clusters in DBSCAN and that of the $\kappa_\tau$-cohesive clusters in MMC bear some resemblance, and they both use a threshold to define clusters. But the former is based on density and thus requires a density estimator, and the latter relies on the Isolation Kernel as $\kappa$ only.

\subsection{\textcolor{black}{Existing key approaches to mitigate issues in density-based clustering algorithms}}

Density-based clustering defines a cluster as a contiguous region of  high-density points, where each cluster is separated by  contiguous regions of low-density points. DBSCAN is a classic density-based clustering algorithm. However, it has two key weaknesses since its introduction, i.e., difficulty in detecting clusters with varied densities and unable to scale to large datasets.

To overcome the first weakness, existing methods  either use an adaptive similarity/dissimilarity measure\footnote{Similar ideas have been used in spectral clustering (e.g., \cite{zelnik2005self}) and classification via distance metric learning \cite{yang2006distance,bellet2022metric} and Multiple Kernel Clustering~\cite{liu2020optimal}. But they all have all high computational cost.} (e.g., \cite{Jarvis-Patrick-1973,SNN-DBSCAN-SDM2003,pei2009decode,
IsolationKernel-AAAI2019}) or apply a hierarchical approach to extract different density levels of clusters (e.g., \cite{ankerst1999optics,HDBSCAN-2015,neto2017efficient,zhu2022hierarchical}). But the high time complexity remains an issue for these methods because of the use of a density estimator and the point-to-point linking process to form a cluster. 

To scale to large datasets,  incremental-based methods (e.g., \cite{ntoutsi2012density,hassani2015subspace}), distributed methods (e.g., \cite{he2014mr,heidari2019big}) and approximate methods (e.g., \cite{luchi2019sampling,chen2021block,huang2023grit}) have been developed in the last two decades. However, these methods only enable execution on a large dataset up to a certain scale, without addressing the fundamental high time complexity of density estimator and point-to-point linking process\footnote{Some algorithms \cite{huang2023grit} utilise a grid-based method for approximation and merging neighboring dense grids to link cluster members, but this process is still a variation of point-to-point linking process and forming grids in high dimensions is a computationally expense process.}. 

None of the above methods resolve both weaknesses simultaneously and satisfactorily. In contrast, the proposed MMC has the superior advantage of detecting clusters with varied densities in massive data, not constrained by high time complexity. 




\subsection{MMC/DMC versus GMM}

Note that mass (probability) maximization is a generic criterion which has been used commonly in probabilistic modeling, and in the clustering context,  Gaussian Mixture Model (GMM).

The algorithm optimizes the parameters of a GMM that best fit a given dataset, where the individual model components are assumed to take some specific parametric form, i.e., Gaussian distribution. The best fit is achieved via some parameter optimization methods such as MLE (Maximum Likelihood Estimation) or MAP (Maximum A Posteriori Estimation) \cite{Parameter-estimation-Book1980}, which are a form of (probability) mass/density maximization, albeit aims at parameter estimation of an assumed model.

MMC/DMC optimizes the clusters that best represent the dataset, without a parametric assumption, enabling each cluster to be arbitrary shape, size and density. The best representation is achieved via a kernel mass/density estimator for each cluster by maximizing the mass/density of each point with respect to the representative sample of a cluster.

The details of these differences are given below.

GMM assumes Gaussian distribution $p(\cdot|\theta)$ with parameter $\theta$ (which consists of a mean vector and a covariance matrix). The probability of a dataset $D$ generated from  a mixture of $k$ components of Gaussian distributions with parameter set $\Theta$ is expressed as:   
\[
p(D|\Theta) = \prod_{\mathbf{x} \in D} \sum_{i=1}^k w_i \times p(\mathbf{x}|\theta_i)
\]
\noindent where $\Theta = \{w_i, \theta_i, i=1,\dots,k \}$ is a collection of parameters of the mixture model which is to be optimized via MLE.

MLE maximizes the probability $p(D|\Theta)$ as follows:
\[
\Theta_{MLE} = \argmax_\Theta p(D|\Theta)
\]

In contrast, MMC/DMC uses a kernel estimator ($\mathbbm{m}_{\kappa}$ or $\mathbbm{f}_{\kappa}$) to represent an initial cluster, as shown in Eq (\ref{eqn-mass-wrt-cluster}).
It then assigns points based on mass-maximization (or density-maximization) with respect to the initial cluster $Q^\tau_i$ (Definitions \ref{def-highest-mass-cluster} \& \ref{def-objective}) as follows:
\[C'_j=\{\mathbf{x}\in D\ |\ \mathop{\argmax}\limits_{i \in [1,k]} \mathbbm{m}_{\kappa}(\mathbf{x}|Q^\tau_i )=j \},\forall_{j\in [1,k]}.
\]

The above reveals three key differences. First, GMM maximizes the probability of the entire given dataset. In DMC, the density of a point of the dataset is estimated from an initial cluster of which the point could be a member. The same applies to mass using MMC. Second, MMC/DMC does not attempt to estimate the parameters of a parametric model, but assign points with respect to initial cluster $Q^\tau_i$ by recruiting members of the highest mass/density, as estimated by the kernel estimator based on $Q^\tau_i$. This makes a huge difference in two aspects: (a) GMM optimizes a collection of parameters which is significantly larger than that in MMC/DMC. (b) The clusters discovered by MMC/DMC can be of arbitrary shapes, sizes and densities; but those found by GMM are constrained to Gaussian distribution  only.   Third,  in the final step of Algorithm \ref{alg:MMC}, MMC/DMC refines the point assignment to achieve the final objective:
\[
\displaystyle \max_\mathbb{C} \sum_{C \in \mathbb{C}} \sum_{\mathbf{x} \in C} \mathbbm{m}_{\kappa}(\mathbf{x}|C ) \mbox{ or } \displaystyle \max_\mathbb{C} \sum_{C \in \mathbb{C}} \sum_{\mathbf{x} \in C} \mathbbm{f}_{\kappa}(\mathbf{x}|C )
\]
\noindent by simply  tweaking at the edges of every cluster $C'$ obtained in the second step.

Interestingly, step 2 in MMC/DMC may be viewed as a generative model (though not in a conventional way) where a non-parametric `model' (i.e., $Q^\tau_i$ in step 2) is assumed to be the representative sample of an unknown distribution which generates the cluster in the given dataset $D$.

Table \ref{tab:comparison} provides a comparison of the characteristics of different clustering algorithms. The first group (k-means, spectral clustering and GMM) is based on an optimization algorithm. The second group is existing density-based algorithms which have no objective function. MMC/DMC is the proposed algorithms which have an objective function, and yet the cluster formation procedure (step 2) does not rely on an optimization algorithm.

\begin{table}[h]
    \centering
    \setlength{\tabcolsep}{1pt}
    \begin{tabular}{c|cccc}
         & Cluster Definition & Objective & Core Operation & Comments  \\ \hline
    k-means & Mean vector & Min mean-square-error & EM algorithm & Super-polynomial\\
    SC & Undefined & Minimum graph cut & Eigen-decomposition  & Feature transf.\\ 
    GMM & Gaussian distribution & MLE or MAP & EM algorithm & Parameter est. \\ \hline
    DBSCAN  & Density-connected clusters & Nil & Pt-to-pt linking & Quadratic\\
    DP      & $\eta$-linked clusters & Nil & Pt-to-pt linking & Quadratic\\
    LGD      & kNN-graph-defined clusters & Nil & Pt-to-pt linking & Quadratic\\
    \hline
    DMC & Highest-density clusters & Max total density & MD pt assignment & Linear\\
    MMC & Highest-mass clusters & Max total mass & MM pt assignment & Linear\\
    \hline
    \end{tabular}
    \caption{Characteristics of different clustering algorithms. MD \& MM denote maximum density and maximum mass, respectively; $\eta$ denotes higher density nearest neighbor \cite{zhu2022hierarchical}.}
    \label{tab:comparison}
\end{table}

\begin{figure}[h]
    \centering
    \includegraphics[width=0.45\linewidth]{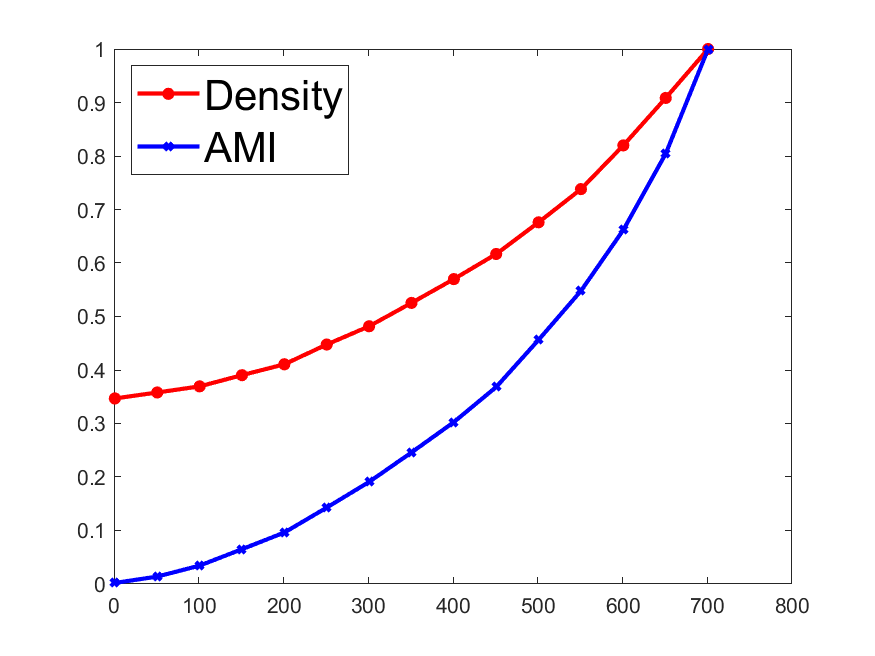}
    \includegraphics[width=0.45\linewidth]{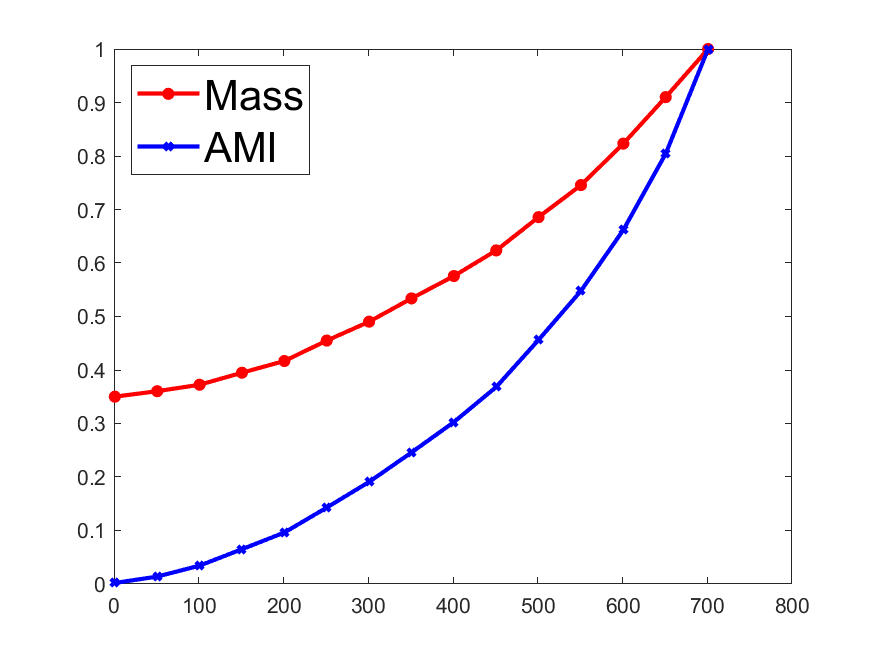}
    \caption{Total density or mass is a proxy for the goodness of a clustering outcome. The total density or mass  versus AMI on the 3G dataset as the error correction process progresses. Beginning with completely random point assignments to all clusters. Then repeatedly randomly select a subset (of 50 points) and correct its individual point assignment if the initial assignment was incorrect, i.e., the error correction rate increases as more points are selected.}
    \label{fig:enter-label}
\end{figure}

Figure \ref{fig:enter-label} shows that the total density or mass is a proxy to the goodness of a clustering outcome, i.e., the total density or mass is monotonically increasing, as the clustering errors are corrected incrementally, indicated by AMI. 

Recall the impact of the high-density bias on the density maximization criterion used in DMC, described in Section \ref{sec-Density-maximization-issue}. As this criterion is analogous to MLE, the same high-density bias applies in GMM too. We are not aware of the discussion of this impact in the GMM literature.

In summary, while MMC and DMC share a similar probability maximization criterion used in GMM at the highest level, the three algorithms differ substantially in terms of cluster definitions, clustering objectives and core operations in the algorithms.

\section{Experimental settings}

We use the Isolation Kernel which is implemented using isolating hyperspheres \cite{IDK} in MMC; and the Gaussian Kernel with the Nystr\"{o}m method \cite{Nystrom_NIPS2000}, which produces its approximate feature map, is employed in DMC.

The parameter search ranges for all algorithms under comparison are shown in Table \ref{tab:search-ranges}.

\begin{table}[!ht]
    \centering
     \caption{Parameter search range. The range $[L:Inc:H]$ denotes the range of values starting from the lowest value $L$ to the highest value $H$ with multiple increments of $Inc$. $k$ is the number of nearest neighbors for LGD and GLSHC. $s$ is the number of anchors or landmarks for SGL and GLSHC. The parameter of the Isolation Kernel used in MMC is: $t=200$. GMM needs no parameter setting.}
     \label{tab:search-ranges}
     \setlength{\tabcolsep}{1pt}
    \begin{tabular}{l|l}
    \hline
        Algorithm & parameter search range  \\ \hline
        MMC & $\psi$ $\in [2,4,6,8,16,24,32,64,128,256]$;  $\tau$ $\in [0.05:0.05:0.95]$  \\ 
        DMC & $\sigma$ $\in \{ 2^i| i \in [-5:1:5]\}$;  $\tau$ $\in [0.05:0.05:0.95]$   \\ 
        DP &  $\epsilon$ $\in [0.01:0.01:0.5]$  \\ 
        DBSCAN &  $\epsilon$ $\in [0.01:0.01:0.5]$; minPts $\in [2:1:20]$   \\ 
        MBSCAN & $\psi$ $\in [2,4,8,16,32,64,128,256,512,1024]$; \& the same parameters  as used in DBSCAN \\ 
        HDBSCAN$^*$ & $min\_cluster\_size\in [5:5:50]$; \& the same parameters as used in DBSCAN \\ 
        LGD & $k\in$[2,4,6,10,15,20]; $\tau \in$ [0.5,0.52,0.56,0.62]\\ 
        SGL & $s\in$[30,40,50]; $\alpha \in$[0.1,1,10]; $\beta \in $[0.01,1,100] \\ 
        GLSHC & $k\in$ [3,5,10]; $s\in$ [10,20,30,50,100,200,500,1000]   \\ \hline
    \end{tabular}
   
\end{table}

Out of the ten artificial datasets shown in Table \ref{tab:f1}, the first eight datasets have different characteristics of dense and sparse clusters. For example, 3G and 2Gaussian have the first and second conditions described in Sections \ref{sec-first-condition} and \ref{sec-second-condition}, respectively. 3L has the first condition and additional conditions. G-Strip and 3G-HL are two datasets which have been used to demonstrate the fundamental problems of spectral clustering \cite{SC-Limitations-2006}. The last two datasets have two Gaussian subspace clusters of the same variance:  w50Gaussian has two 50-dimensional subspace clusters which make density estimation a difficult task on this 100-dimensional dataset; and w10Gaussian is a low dimensional version which has two 10-dimensional subspace clusters.

We present the clustering outcomes of the a clustering algorithm in terms of two commonly used metrics, i.e., F1 and AMI (see the details in Appendix A). For each randomized algorithm (MMC, DMC, SGL and GLSHC), an average of 5 trials is reported for each dataset. 

The detailed descriptions of the datasets, additional experiment settings and the sources of the codes used are provided in Appendix B.

\section{Experiments}

The aims of the experiments are to:
\begin{enumerate}
    \item Examine the relative clustering outcomes of mass-based clustering and density-based clustering using the same proposed clustering algorithm.
    \item 
    Analyse the algorithmic limitations of individual clustering algorithms under the influence of the use of density in DMC,  DP \cite{rodriguez2014clustering}, DBSCAN \cite{ester1996density}, HDBSCAN$^*$ \cite{HDBSCAN-2015}\footnote{\textcolor{black}{HDBSCAN is claimed to have solved the single-threshold problem of DBSCAN in order to deal with datasets having varied densities \cite{HDBSCAN-PAKDD2013,HDBSCAN-2015}. Yet, it is a hierarchical version of DBSCAN which has the same cluster definition and core clustering procedure of DBSCAN. Therefore, they share many of the same limitations discussed in this paper. }}, LGD \cite{li2019LGD} and GMM \cite{reynolds2009gaussian}, as well as two versions of spectral clustering:  SGL \cite{kang2021structured} and GLSHC \cite{yang2021graphlshc}.
    \item Compare the scalability of the proposed clustering algorithm versus existing density-based clustering algorithms.
\end{enumerate}

\begin{table}[b]
\caption{Clustering results in terms of F1. Random denotes the outcome of random cluster assignment.  \textcolor{black}{NC} denotes that the parameter search could not be completed in five days.\\ * The average rank is computed over 14 real-world datasets, where the algorithms with \textcolor{black}{NC} are ranked last (or equal last). The average is computed on all datasets with F1 results, ignoring the ones having \textcolor{black}{NC}. `DB' and `HDB' denote DBSCAN and HDBSCAN$^\star$, respectively. MMC and MMC$^v$ are the same algorithm, except that the IK's used are implemented using Hyperspheres and Voronoi Diagrams, respectively. $\dagger$ denotes the best result that can be achieved in 5 days, without completing the entire parameter search.}
\label{tab:f1}
\setlength{\tabcolsep}{1.5pt}
\begin{tabular}{l|rrr|r|rr|rrrrrr|rr}
     \hline
        Datasets & Size & Dim & \#C & Random & MMC  & MMC$^v$ & DMC  & DP   & DB   & HDB  & LGD     & GMM   & SGL  & GLSHC \\ \hline
        Jain     & 373  & 2   & 2   & 0.49   & 1    & 1 & 1    & 0.74 & 1    & 0.98 & 1             & 0.58  & 0.91 & 1 \\
        3L       & 560  & 2   & 3   & 0.33   & 0.84 & 0.91  & 0.73 & 0.64 & 0.59 & 0.66 & 0.72      & 0.74  & 0.70 & 0.68 \\     
        3G-HL    & 900  & 2   & 3   & 0.27   & 0.97 & 0.70  & 0.97 & 0.98 & 0.87 & 0.93 & 0.98      & 0.53  & 0.55 & 0.48 \\
        3G       & 1500 & 2   & 3   & 0.36   & 0.98 & 0.99  & 0.73 & 0.98 & 0.57 & 0.67 & 0.97      & 0.51  & 0.65 & 0.98\\   
        2Gaussians	&1000	&2&	2   & 0.51	 & 0.99	& 0.98  & 0.93 & 0.97 &	0.83 & 0.89 &	0.99	& 0.99  & 0.94 &	0.99\\
        AC & 1004 & 2 & 2           & 0.50   & 1    & 1     & 1    & 1    & 1    & 1    & 0.90      & 0.85  & 0.91 & 1 \\ 
        G-Strip & 1400 & 2 & 2      & 0.50   & 0.97 & 0.94  & 0.95 & 0.93 & 0.70 & 0.95 & 0.86      & 0.98  & 0.84 & 0.85 \\
        RingG & 1536 & 2 & 4        & 0.26   & 1    & 1     & 1    & 0.96 & 0.67 & 0.96 & 1         & 0.37  & 0.64 & 1 \\ 
        w10Gaussian & 2000 & 20 & 2 & 0.50   & 1    & 1  & 0.98 & 0.60 & 0.34 & 0.44 & 1            & 1     & 1 & 1 \\ 
        w50Gaussian & 2000 & 100 & 2 & 0.50 & 1     & 1 & 0.60 & 0.35 & 0.33     & 0    & 0.4       & 1     & 1 & 1 \\ \hline
        \multicolumn{5}{r|}{Average}       & 0.98  & 0.95  &	0.89    &	0.82   & 0.69 & 0.75 &	0.88	& 0.76  & 0.81 &	0.90 \\
        \multicolumn{5}{r|}{Average rank} & 3.00  & 3.60  &	4.85    &	6.10   & 7.95 & 6.90 &	4.70    & 6.20     & 6.95  & 4.75   \\ \hline
         wine & 178 & 13 & 3 & 0.38 & 0.95          & 0.96 & 0.97& 0.92& 0.72    & 0.59 & 0.90      & 0.89  & 0.99& 0.58\\
        seeds & 210 & 7 & 3 & 0.38 &  0.92          & 0.93 & 0.93& 0.91& 0.76      & 0.60 & 0.90    & 0.66  & 0.93& 0.85\\ 
        dermatology & 358 & 34 & 6 & 0.21& 0.91     & 0.95 & 0.82& 0.82& 0.52       & 0.51 & 0.92   & 0.75  & 0.76& 0.69\\ 
        Foresttype & 523 & 27 & 4& 0.27 & 0.83      & 0.85 & 0.82& 0.55& 0.25         & 0.38 & 0.80 & 0.31  & 0.80& 0.76\\
        COIL & 1440 & 1024 & 20 & 0.09 & 0.91       & 0.88 & 0.82 & 0.68 & 0.84       & 0.89 & 0.92 & 0.42  & 0.82 & 0.84 \\
        spam & 4601 & 57 & 2      & 0.50 & 0.75     & 0.80 & 0.60 & 0.68 & 0.38      & 0.59 & 0.73  & 0.48  & 0.80 & 0.81 \\ 
        gisette & 7000 & 5000 & 2 & 0.50 & 0.91     & 0.88 & 0.53 & 0.50 & 0.01     & $^\dagger$0.50 & 0.83 & 0.34 & 0.84 & 0.92 \\ 
        Pendig & 10992 & 16 & 10 & 0.11 & 0.87      & 0.83 & 0.75 & 0.78 & 0.70      & 0.75 & 0.86  & 0.52  & 0.76 & 0.88 \\
        USPS & 11000 & 256 & 10 & 0.12 & 0.73       & 0.68 & 0.51 & 0.26 & 0.27       & $^\dagger$0.26 & 0.64 & 0.38  & 0.60 & 0.73 \\ 
        imagenet-10 & 13000 & 128 & 10 & 0.11 & 0.91 & 0.91 & 0.91 & 0.85 & 0.85    & \textcolor{black}{NC} & 0.90 & 0.60   & 0.94 & 0.95 \\ 
        stl-10 & 13000 & 128 & 10 & 0.11 & 0.75     & 0.72  & 0.66 & 0.75 & 0.53    & \textcolor{black}{NC} & 0.60 & 0.51   & 0.73 & 0.74 \\ 
        letters & 20000 & 16 & 26 & 0.05 & 0.40     & 0.37  & 0.34 & 0.31 & 0.29    & \textcolor{black}{NC} & 0.37 & 0.23   & 0.35 & 0.37 \\ 
        cifar10 & 60000 & 128 & 10 & 0.10 & 0.74    & 0.78  & 0.73 & 0.76 & $^\dagger$0.01 & \textcolor{black}{NC} & \textcolor{black}{NC} & 0.54   & 0.76 & 0.76 \\ 
        mnist & 100000 & 784 & 10 & 0.10 & 0.77     & 0.77  & 0.54 & \textcolor{black}{NC} & $^\dagger$0.01 & \textcolor{black}{NC} & \textcolor{black}{NC} & 0.47 & 0.59 & 0.80 \\ 
\hline
        \multicolumn{5}{r|}{Average}       & 0.81  & 0.81  &	0.71    &	0.68   & 0.44 & 0.56 &	0.78	& 0.51     & 0.76  &	0.76    \\
        \multicolumn{5}{r|}{Average rank*} & 2.79  & 2.75  &	5.18    &	5.96   & 8.29 & 8.06 &	4.38    & 8.36     & 4.04  & 3.71       \\ \hline
    \end{tabular}
\end{table}

\subsection{Analyses on artificial datasets and real-world datasets}

Here we analyze the comparison results shown in  Table \ref{tab:f1} in two parts on artificial datasets (the first ten rows) and real-world datasets (the following fourteen rows). 

On the artificial datasets, interesting differences between MMC and other algorithms
are summarized as follows:

\begin{itemize}
    \item Density-based algorithms DMC, DP, DBSCAN, \textcolor{black}{HDBSCAN$^*$, LGD and GMM: Not one of these algorithms can do well on all ten datasets. None of them could do well on 3L; and none except GMM on w50Gaussian. In addition, the highest ranked density-based method LGD did poorly on G-Strip; the second highest ranked DMC on 3G. DP on Jain \& w10Gaussian and GMM on Jain, 3G-HL \& RingG. The lowest ranked DBSCAN and HDBSCAN$^*$ did poorly on 3G, 2Gaussians, w10Gaussian, w50Gaussian \& RingG; though HDBSCAN$^*$ has higher F1 results than DBSCAN on all these datasets, except w50Gaussian (on which HDBSCAN$^*$ assigned all points to noise, yielding F1=0).}
    \item Spectral clustering algorithms SGL and GLSHC: they both did poorly on G-Strip, 3G-HL and 3L. SGL did poorly on additional Jain, 3G, AC and RingG datasets.
    \item MMC is the only algorithm which did well on all ten datasets. MMC is better than all four density-based algorithms on 3L and w50Gaussian; and it is better than both spectral clustering algorithms on the 3G-HL, 3L and G-Strip datasets.
\end{itemize}

Table \ref{tab-visualization} shows the visualization of the clustering outcomes of the best 5 algorithms on the eight 2-dimensional datasets. Only MMC performs well on all eight datasets.

\begin{table}[b]
  \centering
    \caption{Visualization for the eight 2-dimensional artificial datasets. The yellow frame indicates that the clustering outcome is perfect or near-perfect. Each outcome is based on a single trial only. }
    \label{tab-visualization}
    \setlength{\tabcolsep}{4pt}
  \begin{tabular}{c|cccccc}
    \hline
   & MMC & DMC & DP&  LGD & GLSHC\\
    \hline 
\rotatebox{90}{Jain}& \includegraphics[width=.16\textwidth,cframe=yellow 0.5mm]{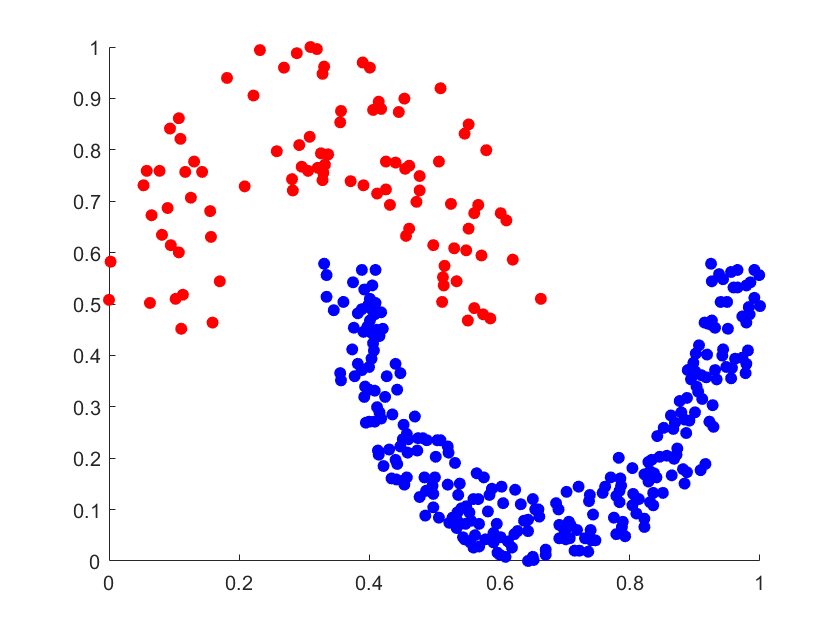} 
& \includegraphics[width=.16\textwidth,cframe=yellow 0.5mm]{figures/2dimVis/jain-dmc.png} &\includegraphics[width=.16\textwidth]{figures/2dimVis/jain-dp.png} &\includegraphics[width=.16\textwidth,cframe=yellow 0.5mm]{figures/2dimVis/jain-lgd.png} 
&\includegraphics[width=.16\textwidth,cframe=yellow 0.5mm]{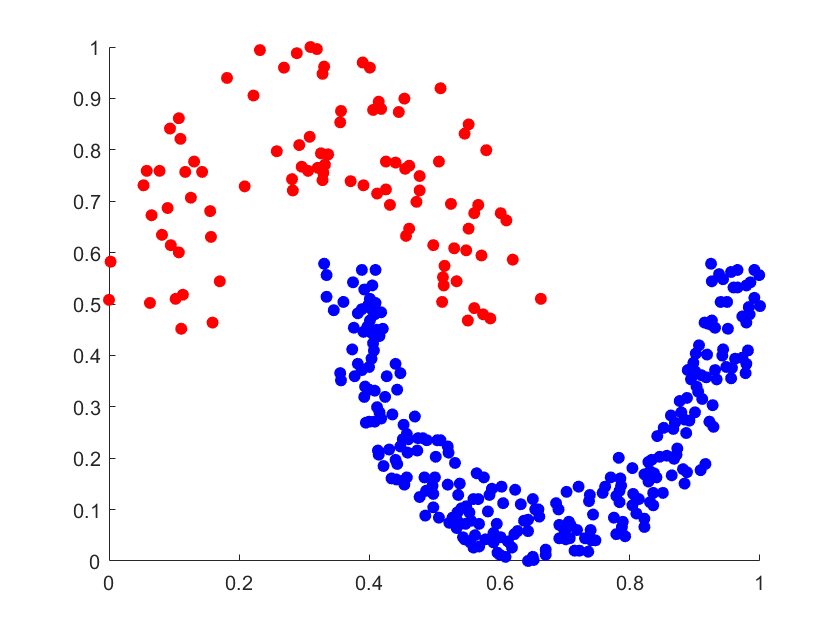} \\ \hline
\rotatebox{90}{3L}& \includegraphics[width=.16\textwidth,cframe=yellow 0.5mm]{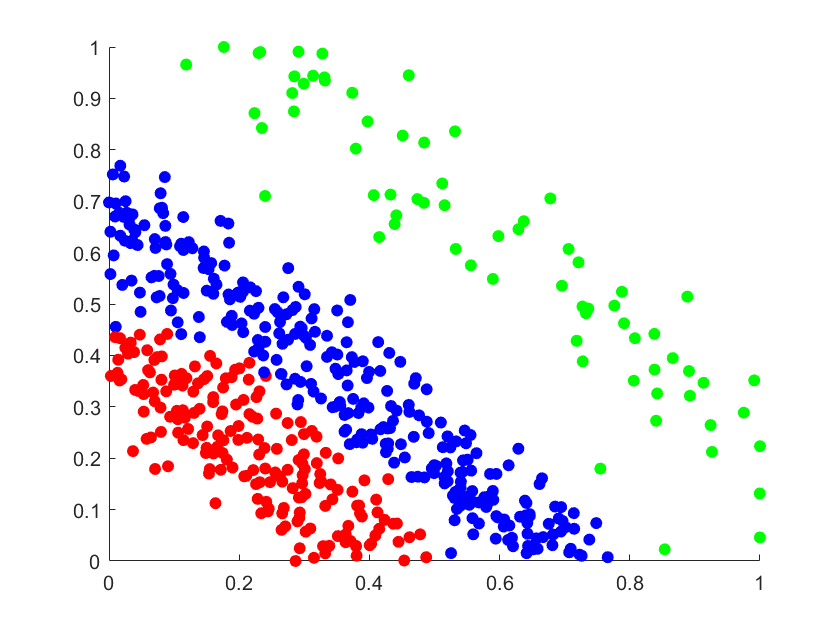} & \includegraphics[width=.16\textwidth]{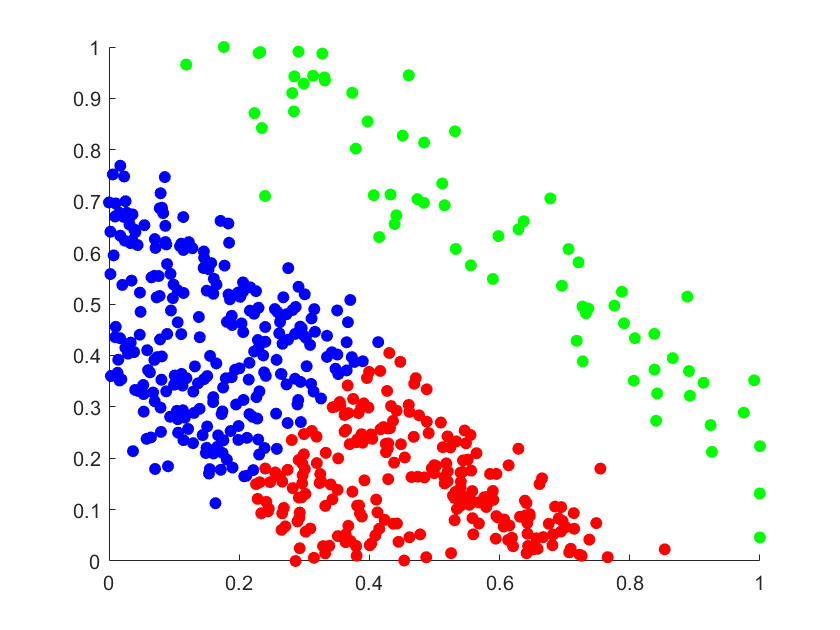} &\includegraphics[width=.16\textwidth]{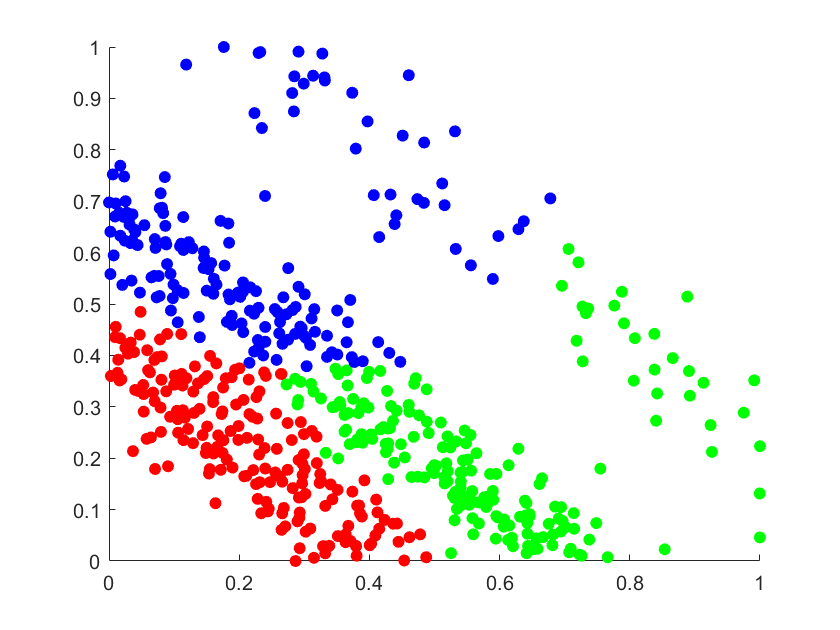} 
&\includegraphics[width=.16\textwidth]{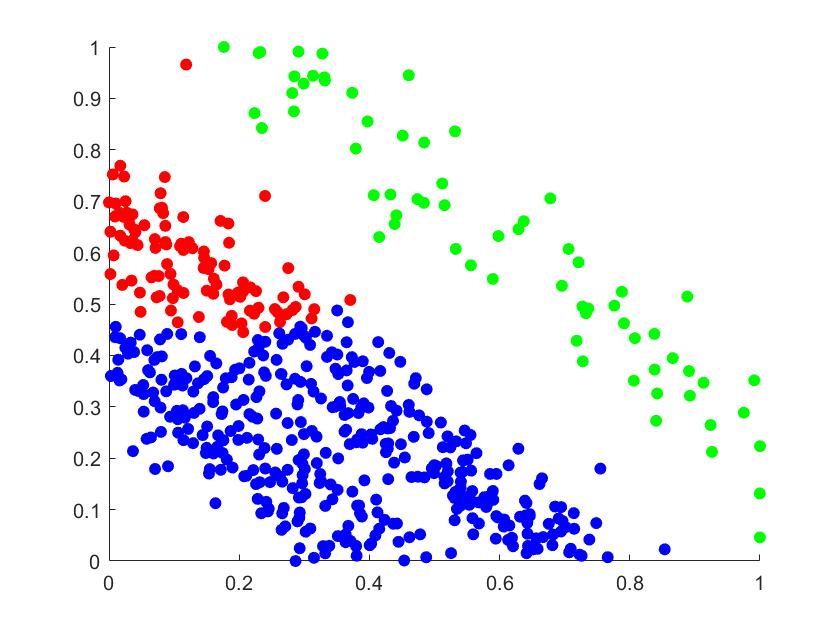} 
&\includegraphics[width=.16\textwidth]{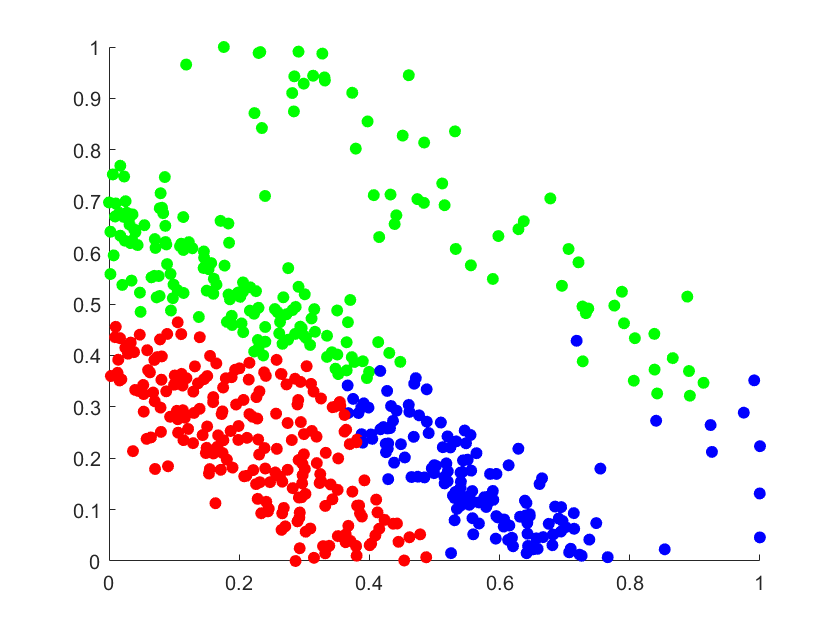} \\ \hline
\rotatebox{90}{3G-HL}& \includegraphics[width=.16\textwidth,cframe=yellow 0.5mm]{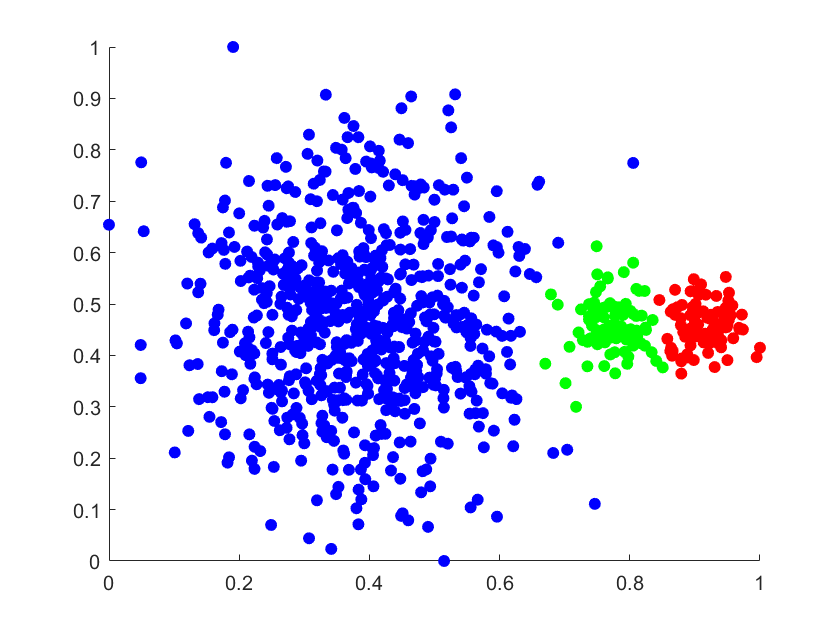} & \includegraphics[width=.16\textwidth,cframe=yellow 0.5mm]{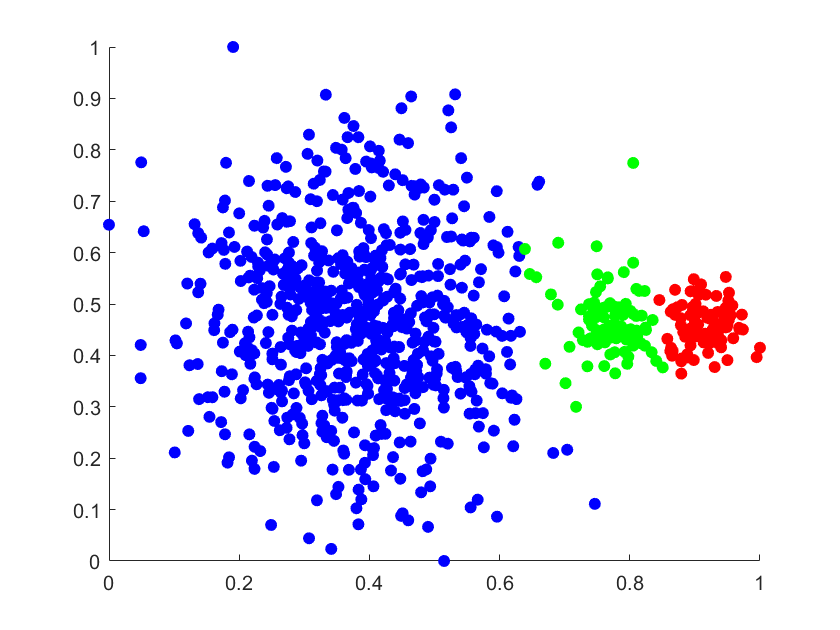} &\includegraphics[width=.16\textwidth,cframe=yellow 0.5mm]{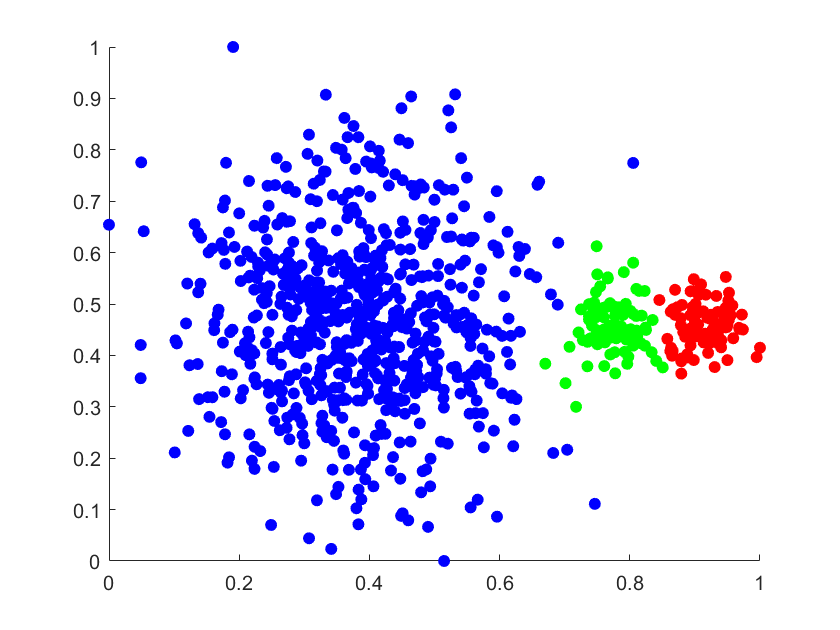} &\includegraphics[width=.16\textwidth,cframe=yellow 0.5mm]{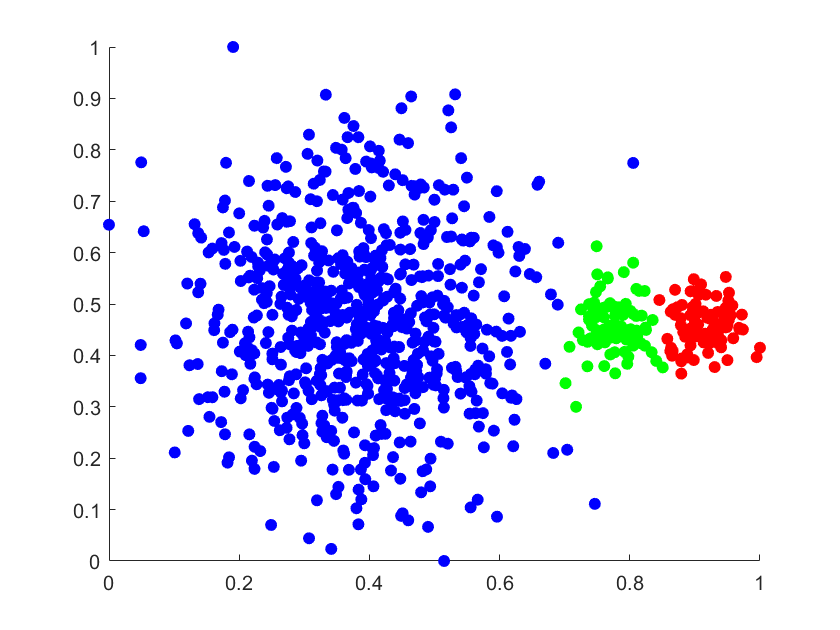} 
&\includegraphics[width=.16\textwidth]{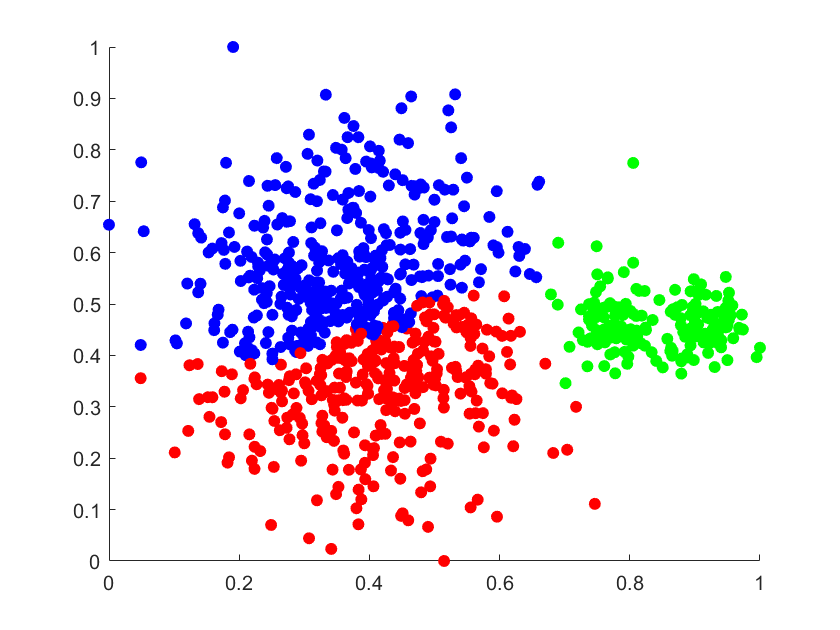} \\ \hline
\rotatebox{90}{3G}& \includegraphics[width=.16\textwidth,cframe=yellow 0.5mm]{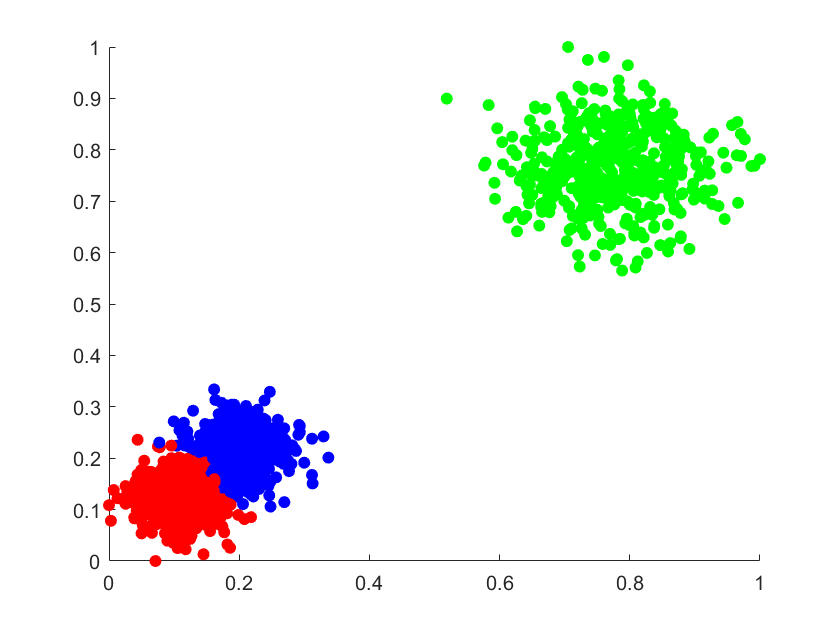} &\includegraphics[width=.16\textwidth]{figures/2dimVis/3g-dmc.png} &\includegraphics[width=.16\textwidth,cframe=yellow 0.5mm]{figures/2dimVis/3g-dp.png} &\includegraphics[width=.16\textwidth,cframe=yellow 0.5mm]{figures/2dimVis/3g-lgd.png} 
&\includegraphics[width=.16\textwidth,cframe=yellow 0.5mm]{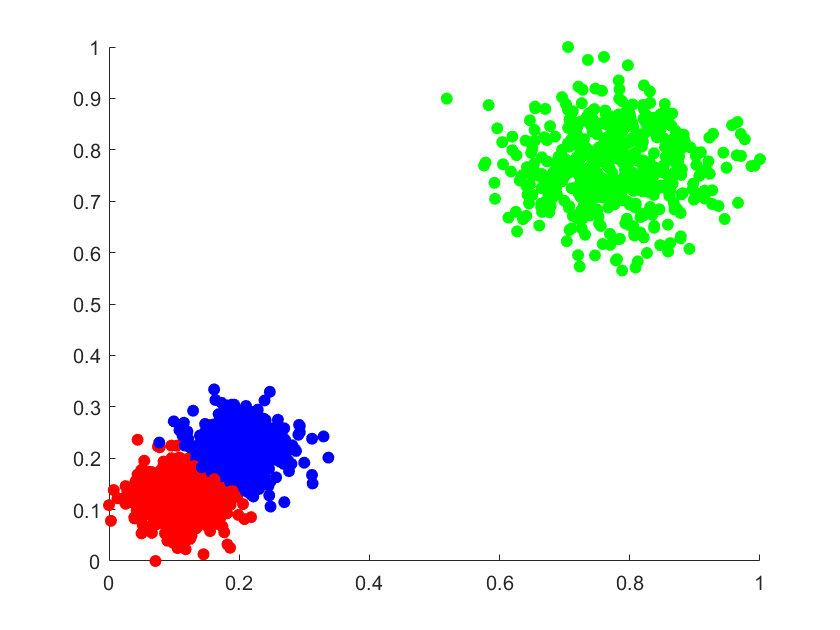} \\ \hline
\rotatebox{90}{2Gaussians}& \includegraphics[width=.16\textwidth,cframe=yellow 0.5mm]{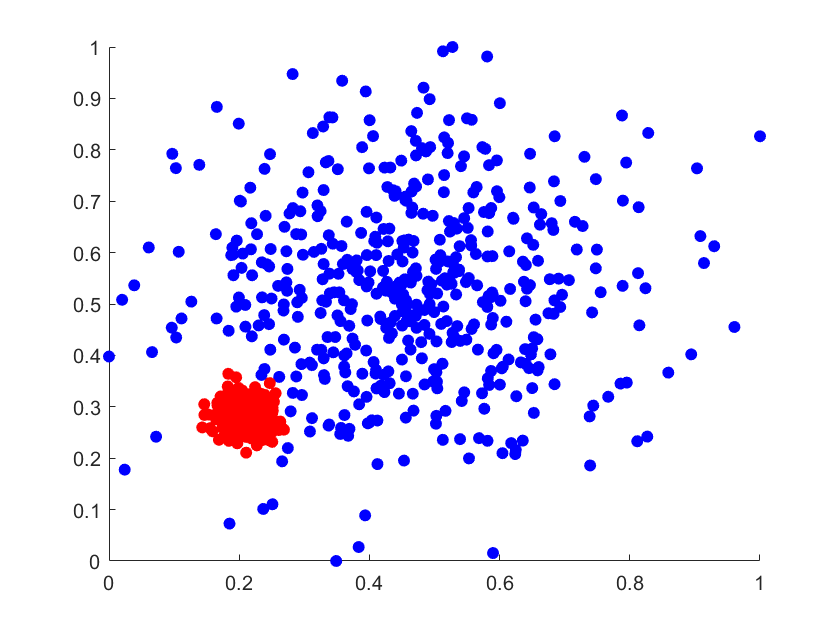} & \includegraphics[width=.16\textwidth]{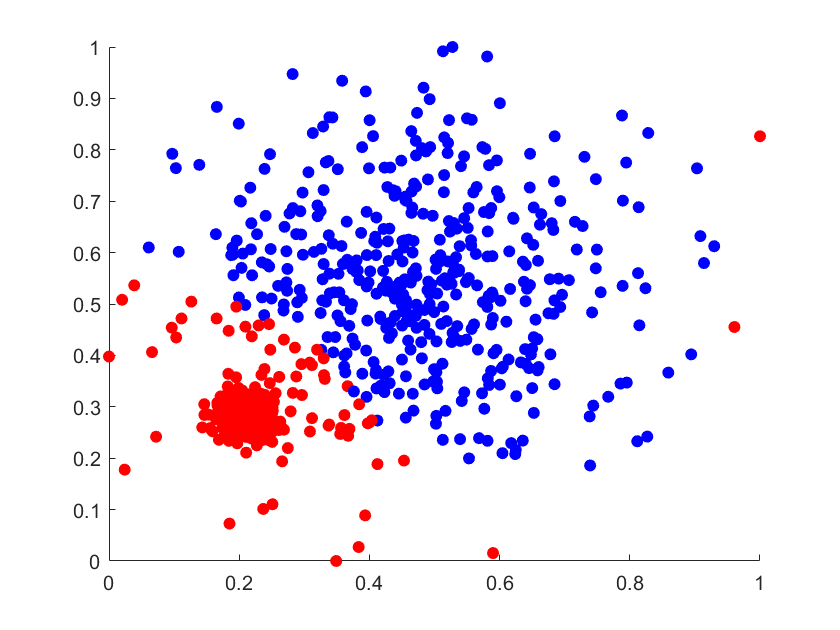} &\includegraphics[width=.16\textwidth]{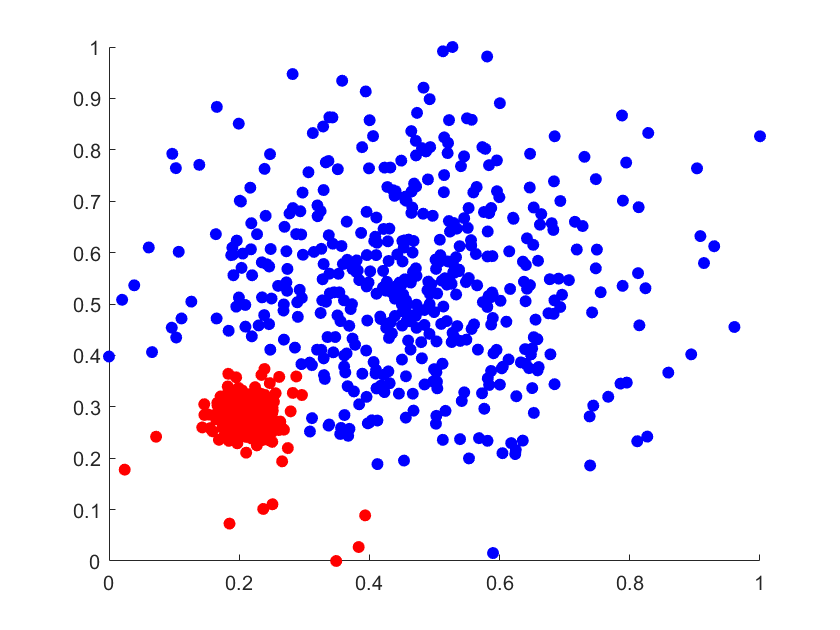} &\includegraphics[width=.16\textwidth,cframe=yellow 0.5mm]{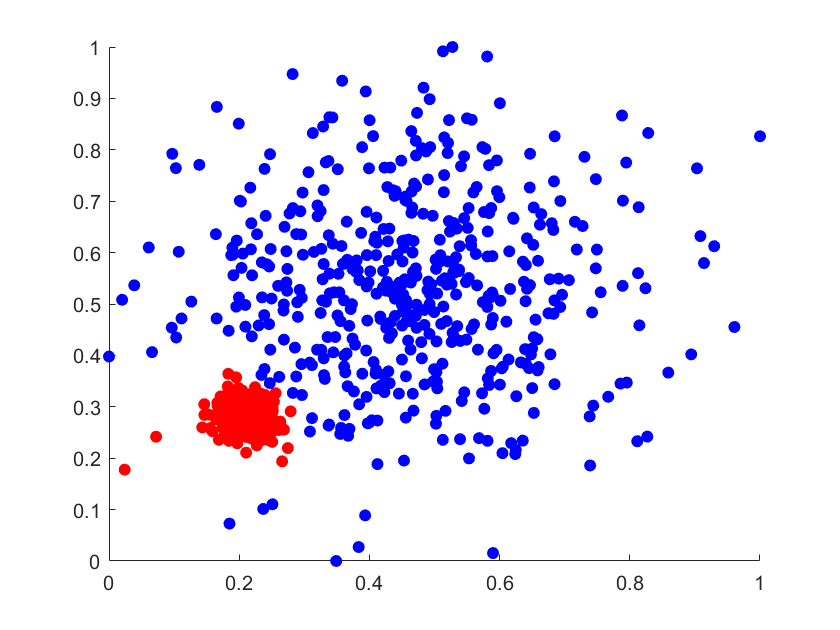} 
&\includegraphics[width=.16\textwidth,cframe=yellow 0.5mm]{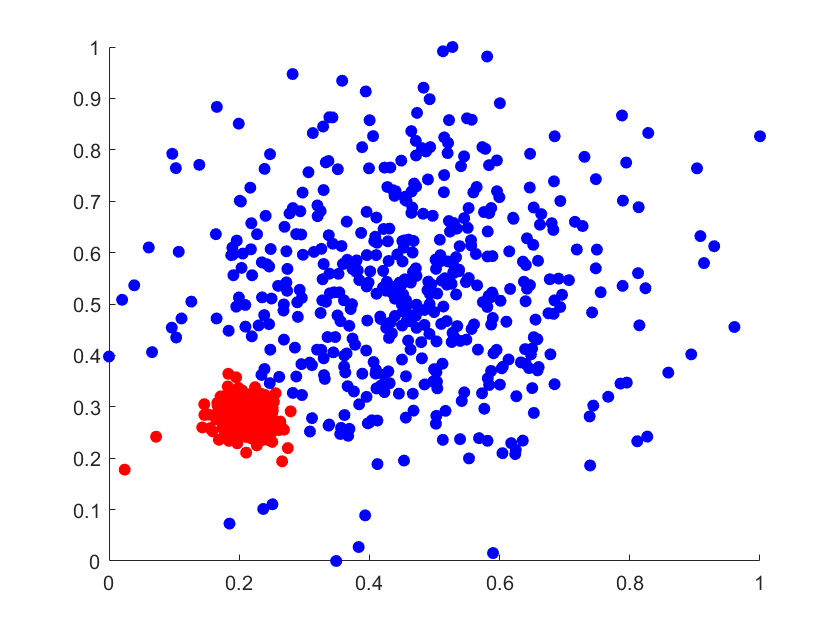} \\ \hline
\rotatebox{90}{AC}& \includegraphics[width=.16\textwidth,cframe=yellow 0.5mm]{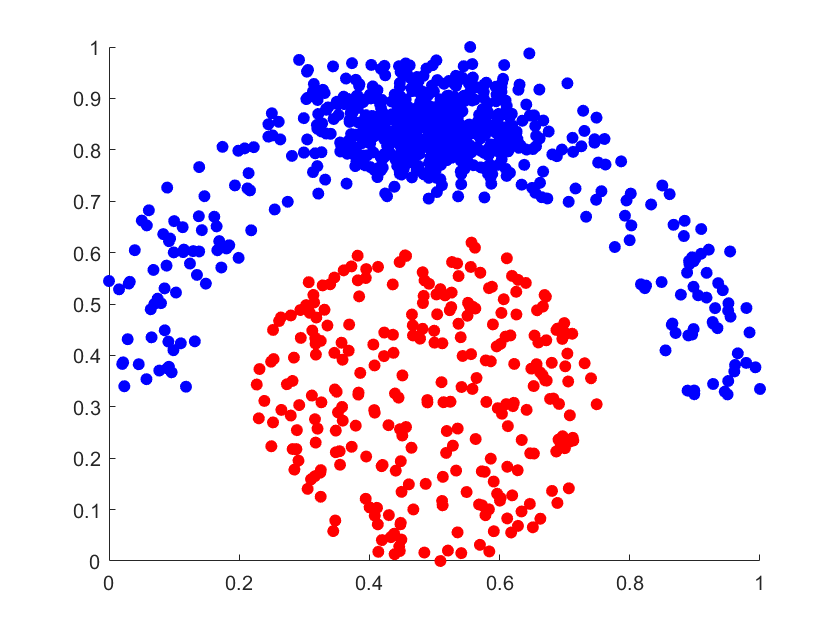} & \includegraphics[width=.16\textwidth,cframe=yellow 0.5mm]{figures/2dimVis/AC-dmc.png} &\includegraphics[width=.16\textwidth,cframe=yellow 0.5mm]{figures/2dimVis/AC-dmc.png} &\includegraphics[width=.16\textwidth]{figures/2dimVis/AC-lgd.png} 
&\includegraphics[width=.16\textwidth,cframe=yellow 0.5mm]{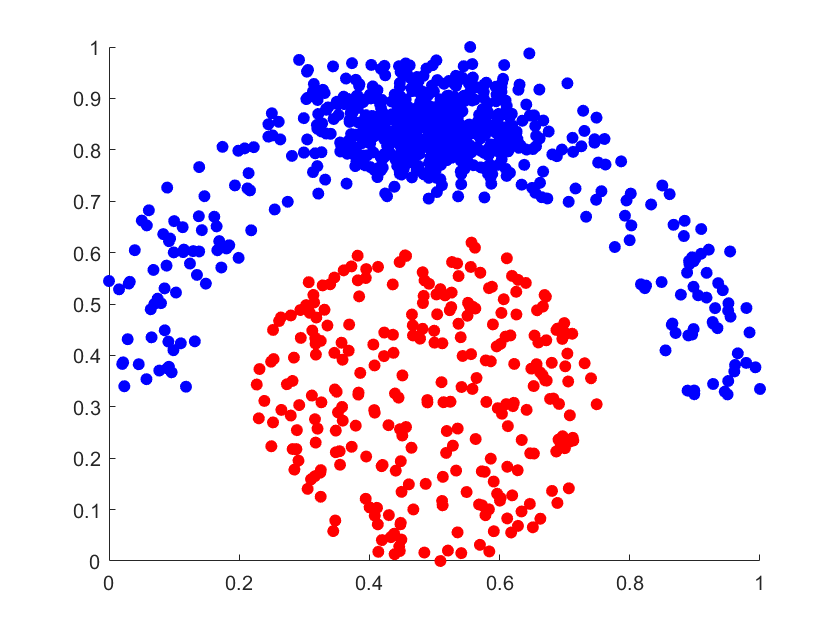} \\ \hline
\rotatebox{90}{G-Strip}& \includegraphics[width=.16\textwidth,cframe=yellow 0.5mm]{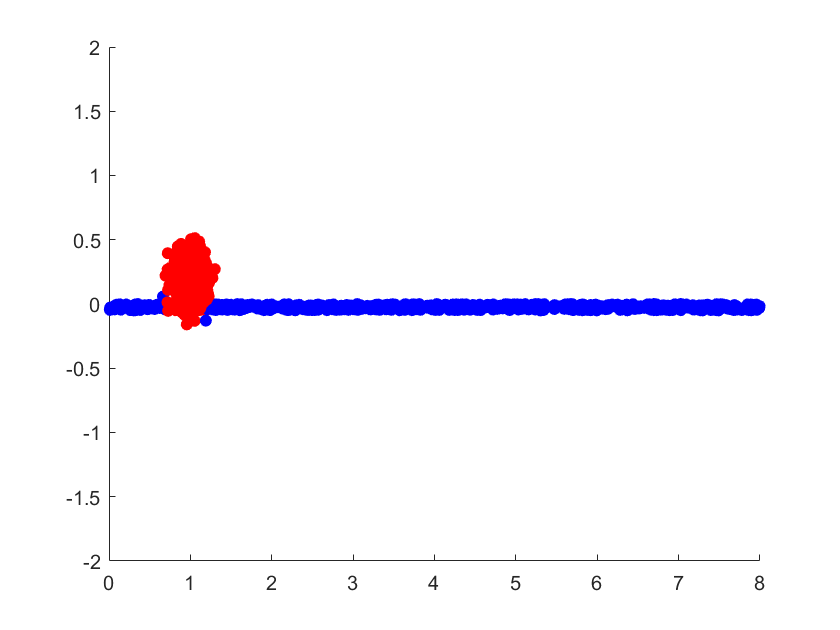} & \includegraphics[width=.16\textwidth,cframe=yellow 0.5mm]{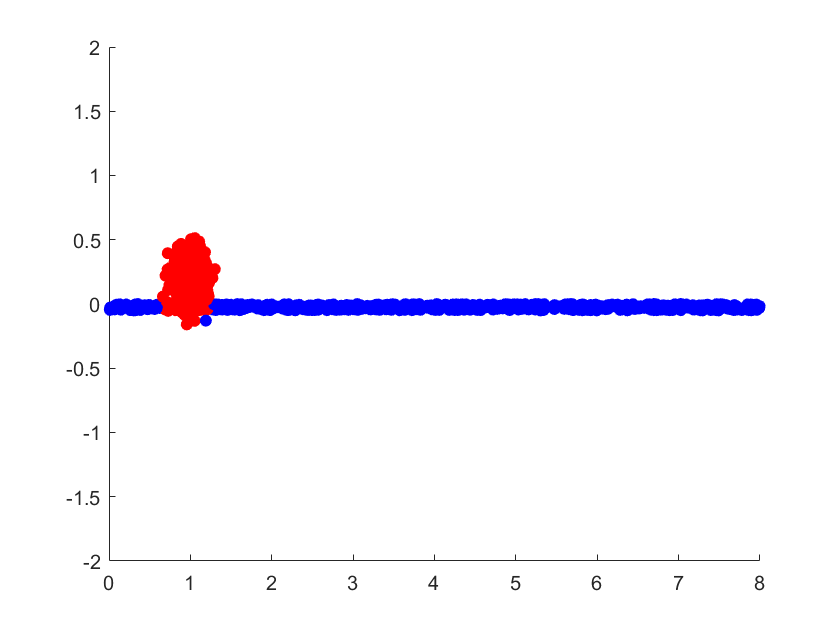} &\includegraphics[width=.16\textwidth]{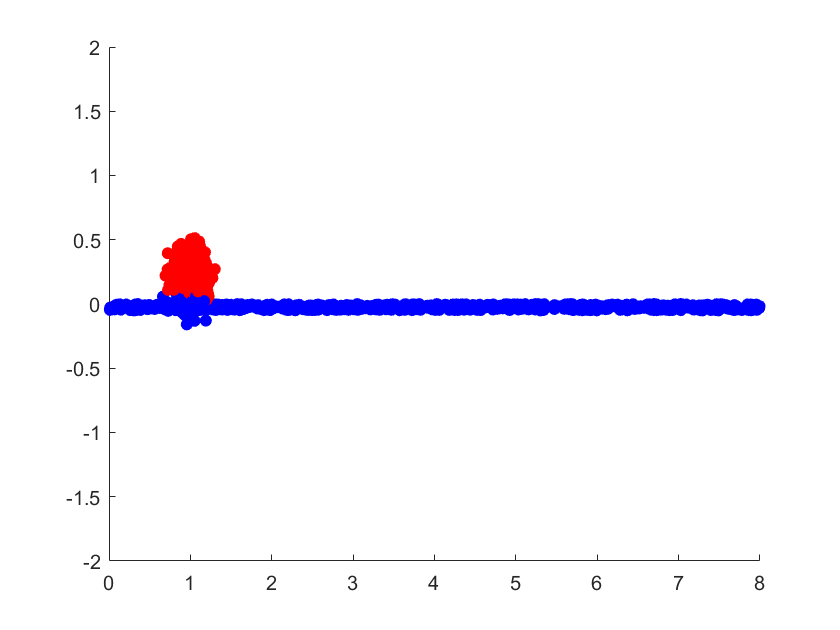} &\includegraphics[width=.16\textwidth]{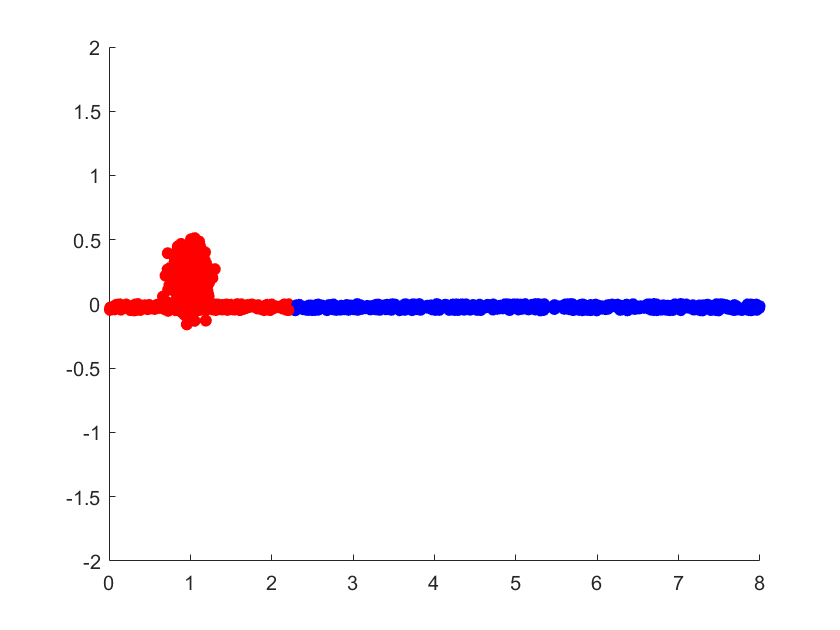} 
&\includegraphics[width=.16\textwidth]{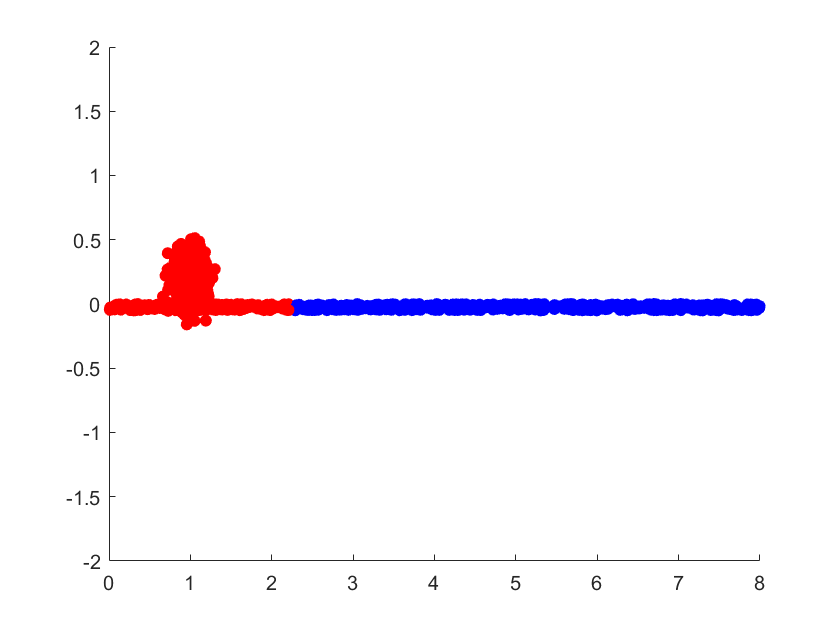} \\ \hline
\rotatebox{90}{RingG}& \includegraphics[width=.16\textwidth,cframe=yellow 0.5mm]{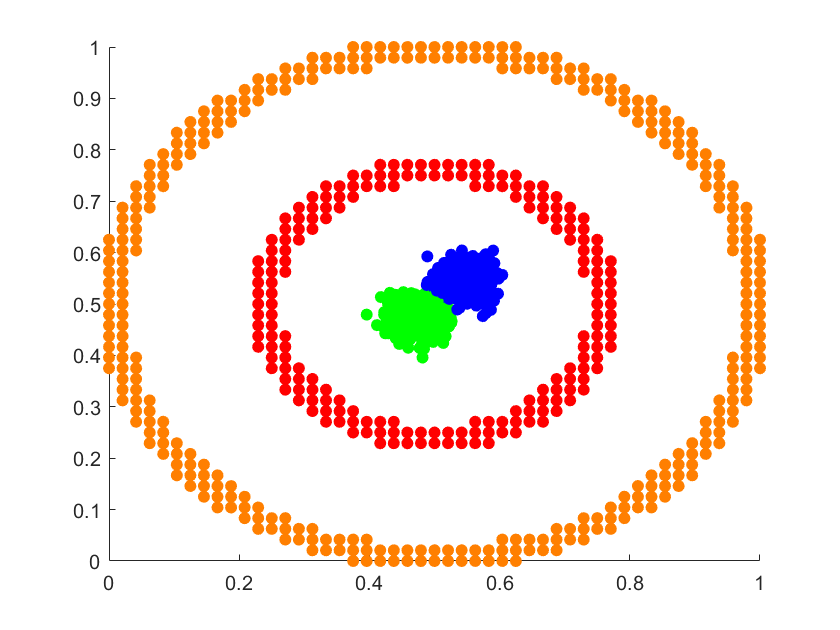} & \includegraphics[width=.16\textwidth,cframe=yellow 0.5mm]{figures/2dimVis/RingG-dmc.png} &\includegraphics[width=.16\textwidth]{figures/2dimVis/RingG-dp.png} &\includegraphics[width=.16\textwidth,cframe=yellow 0.5mm]{figures/2dimVis/RingG-dmc.png} 
&\includegraphics[width=.16\textwidth,cframe=yellow 0.5mm]{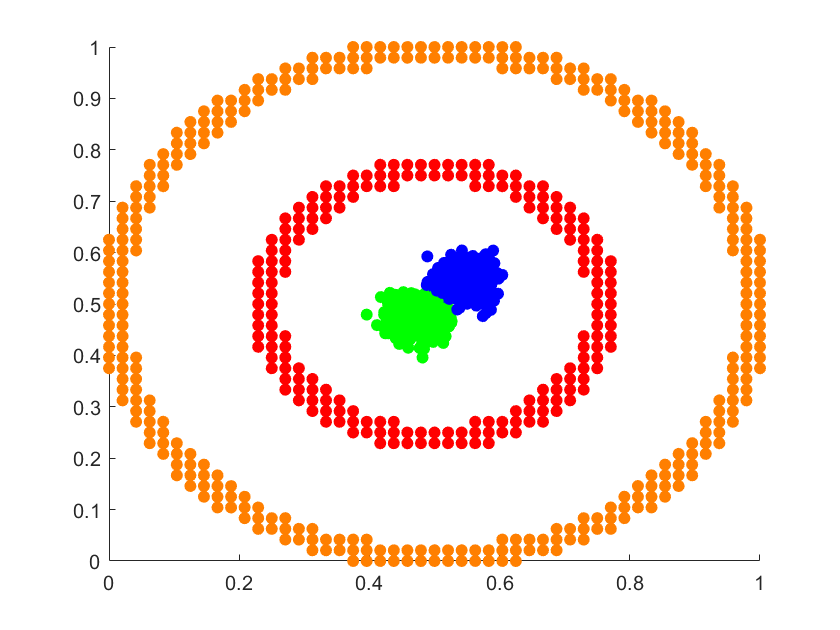} \\ \hline

  \end{tabular}

\end{table}

We have the following observations on the fourteen real-world datasets:

\begin{itemize}
    \item MMC is the best or close to the best performer on each of the 14 datasets, i.e., MMC does not perform poorly on any of these datasets.
    \textcolor{black}{MMC and MMC$^v$ are the same algorithm, except that the IK's used are implemented using Hypersheres and Voronoi Diagrams, respectively (as stated in Section \ref{sec_prelim-IK}). They have comparable results.}
    \item All the other contenders perform poorly on at least 3 datasets. For example, \textcolor{black}{DBSCAN and GMM have a significant performance gap in comparison with the best performer on every dataset. HDBSCAN$^*$ performs poorly in 8 out of the 9 datasets in which it can complete the run. Note that there is no guarantee that HDBSCAN$^*$ can perform better than DBSCAN. In fact, it is worse in 4 out of 9 datasets on which it could complete the run. This result is consistent with the previous results which focused on hierarchical clustering algorithms \cite{zhu2022hierarchical,malzer2020hybrid}.} DP performs close to the best performer on 2 datasets only (wine and seeds). The best density-based method LGD has poor performance on 3 datasets (gisette, USPS, stl-10). Spectral clustering SGL and GLSHC have poor performance on 6 and 5 datasets, respectively.  
    \item MMC, DMC, SGL and GLSHC are scalable to the two largest datasets. DP, DBSCAN, HDBSCAN$^*$ and LGD failed to complete the entire parameter search in five days (on at least one dataset) due to their high computational complexities.
\end{itemize}    

The overall result is summarized in the last two rows in Table~\ref{tab:f1}, where MMC and MMC$^v$ are the highest ranked performers  and have the highest average clustering results, compared with other contenders. 


MMC versus DMC is a head-to-head comparison where the difference is due to mass distribution versus density distribution only, as the algorithm is exactly the same. It is interesting to note that MMC is either better than or equal to DMC on all the datasets, except two. Some differences are huge, e.g., w50Gaussian, gisette, USPS and mnist. The only two exceptions are wine and seeds, where their differences are very small.

GLSHC ranks second, and it performed significantly poorer than MMC on G-Strip 3G-HL, 3L, wine and dermatology datasets. Ranked third are SGL, LGD, DMC and DP, where three of them are density-based algorithms. GMM, DBSCAN and HDBSCAN$^*$ are the weakest performers. The clustering performance differences among the five density-based algorithms, i.e.,  DMC, LGD, DP, DBSCAN and HDBSCAN$^*$, are mainly due to their algorithmic differences.

The conclusion is similar in terms of AMI, shown in Table \ref{tab:ami}.

\begin{table}[h]
\caption{Clustering results in terms of AMI. Note that AMI could not be computed for DBSCAN because it identifies noise points which could not be accounted for any of the ground-truth clusters. The average is computed without \textcolor{black}{NC}. * The average rank is computed by assigning \textcolor{black}{NC} the last rank. }
\setlength{\tabcolsep}{2pt}
\begin{tabular}{l|rrr|r|rr|rrrr|rr}
\hline
        Datasets     & Size & Dim & \#C & Random & MMC & MMC$^v$ & DMC & DP & LGD &GMM & SGL & GLSHC \\ \hline
        Jain         & 373 & 2 & 2 & 0 & 1 & 1 & 1 & 0.32 & 1   & 0.19 & 0.61 & 1 \\
        3L           & 560 & 2 & 3 & 0 & 0.59 & 0.69 & 0.38 & 0.45 & 0.47 & 0.49 & 0.45 & 0.47 \\  
        3G-HL        & 900 & 2 & 3 & 0 & 0.90  & 0.51 & 0.88 & 0.92 & 0.93 & 0.55 & 0.73 & 0.45 \\
        3G           & 1500 & 2 & 3 & 0 & 0.91 & 0.95 & 0.82 & 0.91 & 0.90	& 0.67 &0.73	&0.91\\ 
        2Gaussians	&1000	&2	&2&	0		& 0.91 & 0.88	&0.7	&0.85&	0.91	& 0.94 &0.74&	0.91\\
        AC          & 1004 & 2 & 2 & 0 & 1 & 1 & 1 & 1 & 0.56 & 0.48 & 0.58 & 1 \\ 
        G-Strip     & 1400 & 2 & 2 & 0 & 0.81 & 0.74 & 0.75 & 0.69 & 0.51 & 0.84 & 0.48 & 0.59 \\ 
        RingG       & 1536 & 2 & 4 & 0 & 0.99 & 0.99 & 0.98 & 0.92 & 0.98 & 0.55 & 0.72 & 0.98 \\ 
        w10Gaussian & 2000 & 20 & 2 & 0 & 1 & 1 & 0.88 & 0.17 & 0.99 & 0.99 & 1 & 1 \\ 
        w50Gaussian & 2000 & 100 & 2 & 0 & 1 & 1& 0.17 & 0.03 & 0.08 & 1 & 1 & 1 \\ \hline
        \multicolumn{5}{r}{Average}      & 0.91 & 0.88 & 0.76	& 0.63	& 0.73	& 0.67  & 0.70   & 0.83 \\
        \multicolumn{5}{r}{Average rank} & 2.60 & 3.10 & 5.20 	& 5.45	& 4.70  & 5.15	& 5.80 & 4.00 \\ \hline
        wine & 178 & 13 & 3 & 0& 0.83  & 0.86      & 0.88& 0.76& 0.72 & 0.69 & 0.97& 0.47      \\ 
        seeds & 210 & 7 & 3 & 0& 0.73  & 0.75      & 0.76& 0.71 & 0.70 & 0.51 & 0.77& 0.61   \\ 
        dermatology & 358 & 34 & 6 & 0 & 0.88 & 0.92  & 0.84& 0.80& 0.86    & 0.84 & 0.87& 0.81\\
        Foresttype & 523 & 27 & 4 &0&  0.59  & 0.62 & 0.56& 0.37& 0.64      & 0.24  & 0.66& 0.46\\ 
        COIL        & 1440 & 1024 & 20 & 0 & 0.95 & 0.91 & 0.89 & 0.86 & 0.98 & 0.61    & 0.81 & 0.92 \\        
        spam & 4601 & 57 & 2 & 0 & 0.21 & 0.28 & 0.14 & 0.10 & 0.23         & 0.01  & 0.29 & 0.30 \\ 
        gisette     & 7000 & 5000 & 2 & 0 & 0.59 & 0.51 & 0.08 & 0.04 & 0.39 & 0.01 & 0.36 & 0.60 \\ 
        Pendig      & 10992 & 16 & 10 & 0 & 0.84 & 0.83  & 0.74 & 0.75 & 0.80 & 0.58    & 0.73 & 0.84 \\        
        USPS        & 11000 & 256 & 10 & 0 & 0.77 & 0.62 & 0.49 & 0.38 & 0.71 & 0.33    & 0.55 & 0.72 \\        
        imagenet-10     & 13000 & 128 & 10 & 0 & 0.86 & 0.86 & 0.85 & 0.86 & 0.83 & 0.64     & 0.88 & 0.89 \\ 
        stl-10      & 13000 & 128 & 10 & 0 & 0.66 & 0.65 & 0.64 & 0.66 & 0.60 & 0.54    & 0.66 & 0.68 \\ 
         letters    & 20000 & 16 & 26 & 0 & 0.51 & 0.46 & 0.44 & 0.41 & 0.46    & 0.30  & 0.41 & 0.45 \\ 
        cifar10         & 60000 & 128 & 10 & 0 & 0.71 & 0.73 & 0.68 & 0.69 & \textcolor{black}{NC} & 0.61    & 0.74 & 0.74 \\
        mnist       & 100000 & 784 & 10 & 0 & 0.74 & 0.70 & 0.49 & \textcolor{black}{NC} & \textcolor{black}{NC} & 0.48 & 0.56 & 0.75 \\ 
\hline
        \multicolumn{5}{r}{Average} & 0.71 & 0.69 &	0.61	&0.57	&0.66	& 0.46  & 0.66 &	0.66  \\
        \multicolumn{5}{r}{Average rank*} & 2.82 & 3.18 & 5.11	& 5.81 	& 4.21	& 7.54 & 3.50	&3.21 \\ \hline
    \end{tabular}
\label{tab:ami}
\end{table}

\subsection{Why do LGD, SGL and GLSHC fail on some datasets?}
\label{sec-why-fail}
The recent density-based clustering method LGD \cite{li2019LGD} aims to address the problem of identifying sparse clusters of density-based clustering. LGD builds a kNN-graph from a given dataset, and then removes edges in the kNN-graph based on a criterion called local gap density. The next step chooses the $k$ largest connected subgraphs as the initial clusters. The last step
is to assign all unassigned points to one of these initial clusters based on 1-nearest neighbour.

LGD\footnote{Note that all the other three density-based clustering algorithms: DMC, DP and DBSCAN also performed poorly on w50Gaussian.} performed poorly on w50Gaussian which has two subspace clusters of the same density (i.e., no issue of sparse and dense clusters on this dataset). Its proposed criterion failed in separating the two clusters in high dimensions because pair-wise distances become similar for all points. This creates difficulties in removing the edges that could effectively separate the two clusters.  


Both spectral clustering algorithms, SGL \cite{kang2021structured} and GLSHC \cite{yang2021graphlshc},  use an approximation method to reduce the time complexity, i.e., it builds a sparse graph by randomly selecting a set of objects as anchors and calculating the affinities between anchors and all data points. As random selection ignores the structural information, it is one of the framework bottlenecks \cite{yang2023reskm}. Moreover, we found that both algorithms are sensitive to the number of anchors. It is interesting to mention that when all points are used as anchors, GLSHC produces much poorer clustering results on the 3G and Jain datasets.

          
SGL performed poorly on 3G and RingG. But it did well when each of these datasets is transformed by the Isolation Kernel's feature map in a preprocessing step. This can be attributed to the issue of sparse and dense clusters which exist on these datasets. The IK-mapped datasets effectively convert the clusters with varied densities into clusters of approximately the same density (recall that IK produces clusters of approximately the same cohesiveness $\bar{S}_\kappa(C^\tau_\beta) \approx \bar{S}_\kappa(C^\tau_\alpha)$, stated in Section \ref{sec-cohesiveness}). This highlights the assumption in SGL, i.e., clusters are assumed to have approximately the same density, in order for SGL to work well.



\subsection{Ablation studies}

We conduct two ablation studies here. The first one examines the effect of sparse and dense clusters on clustering algorithms using a real-world dataset. The second investigates the effect of post-processing (step 3) of MMC.

\subsubsection{Study 1: The effect of sparse and dense clusters using the mnist dataset}

We simulate a dataset having sparse and dense clusters from the mnist dataset. It consists of three digits 4, 7 \& 9, where digit 7 is down-sampled to 10\% of the original data size, and the other digits have the same sizes as in the original dataset. The distributions of the three clusters, visualized via MDS \cite{torgerson1952multidimensional}, are shown in the first row in Table \ref{tab-mds479}.

\begin{table}[p]
  \centering
    \caption{Visualization by using MDS  on the mnist479 dataset. (Digit 7 is down-sampled to 10\% of the original data size). Because the three clusters have significant overlaps, each cluster is drawn separately to improve readability. Each of MMC, DMC and GLSHC is run 5 trials, and the result with the highest F1 is shown. The clustering outcome of each algorithm is indicated by one color for each cluster with respect to the ground-truth cluster shown in each column.}
    \label{tab-mds479}
  \begin{tabularx}{\textwidth}{c|ccc}
    \hline
   & Digit 4 & Digit 7& Digit 9 \\
    \hline 
\rotatebox[origin=c]{90} {Ground Truth }&  \includegraphics[width=.32\textwidth,valign=c]{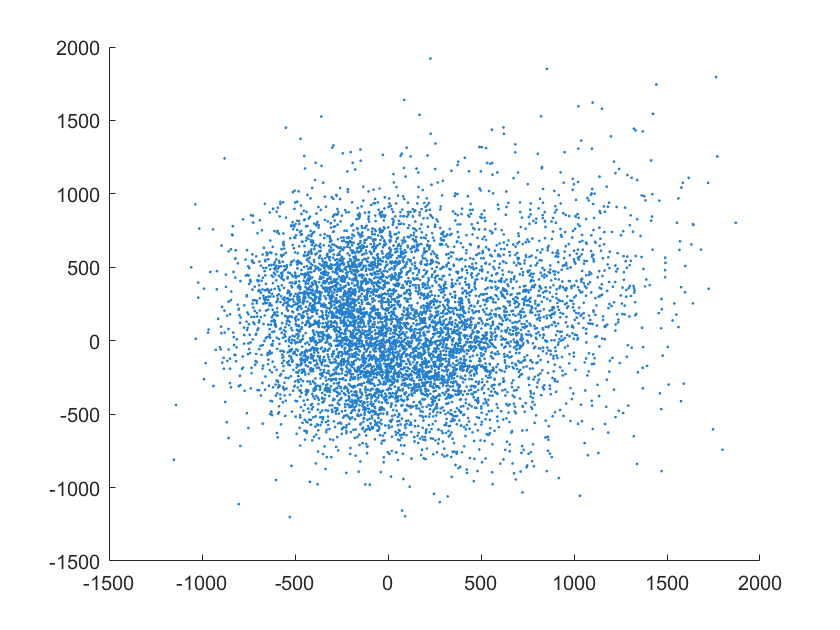} &
\includegraphics[width=.32\textwidth,valign=c]{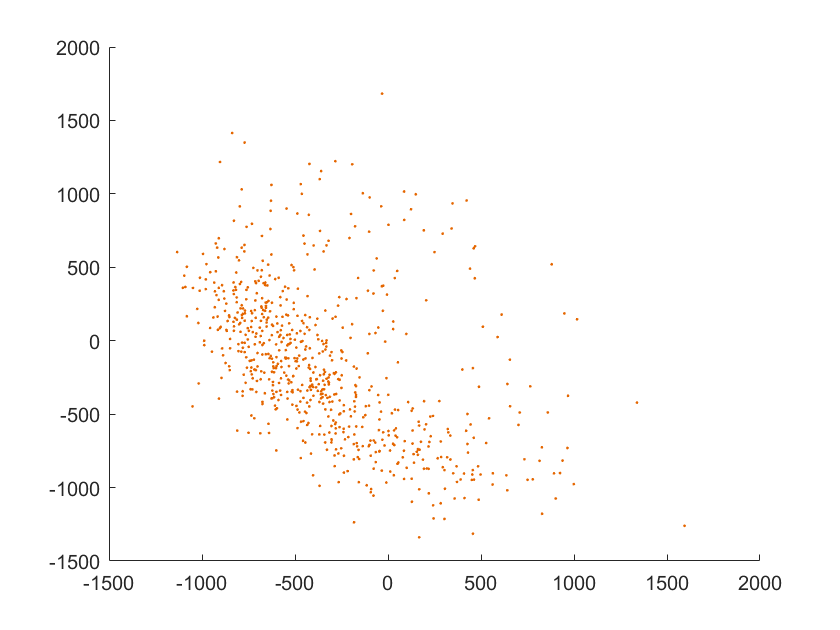} 
&\includegraphics[width=.32\textwidth,valign=c]{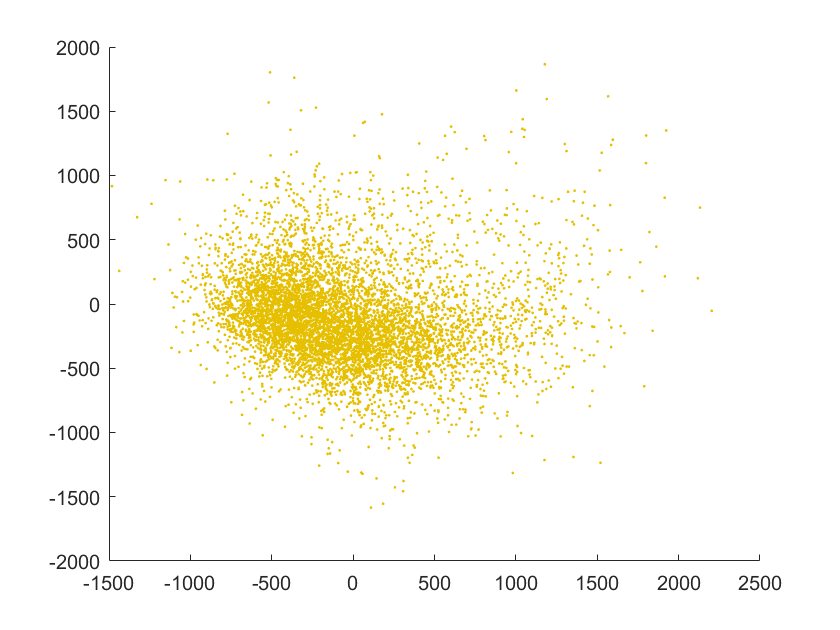} \\ \hline
\rotatebox[origin=c]{90}{MMC F1=0.68}&  \includegraphics[width=.32\textwidth,valign=c]{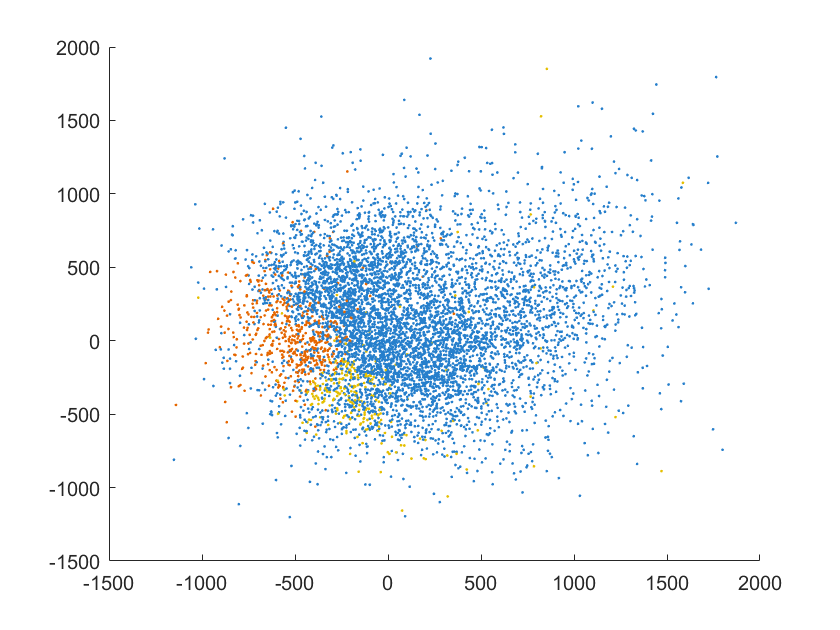} &
\includegraphics[width=.32\textwidth,valign=c]{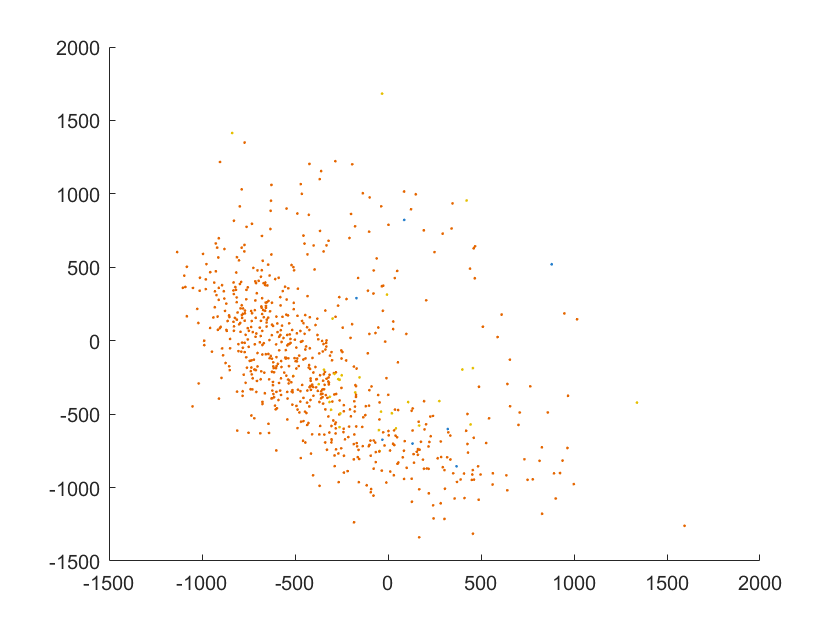} 
&\includegraphics[width=.32\textwidth,valign=c]{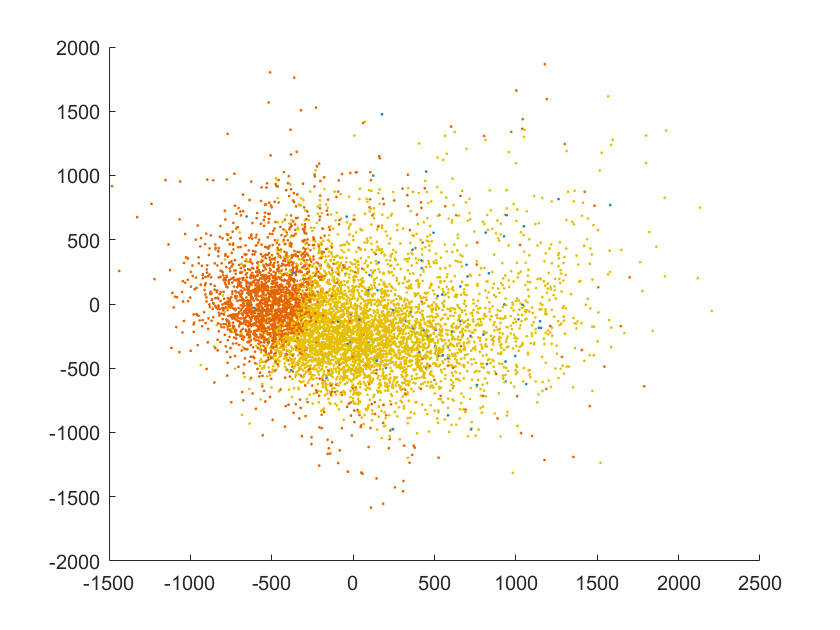} \\ \hline
\rotatebox[origin=c]{90}{DMC F1=0.41}&  \includegraphics[width=.32\textwidth,valign=c]{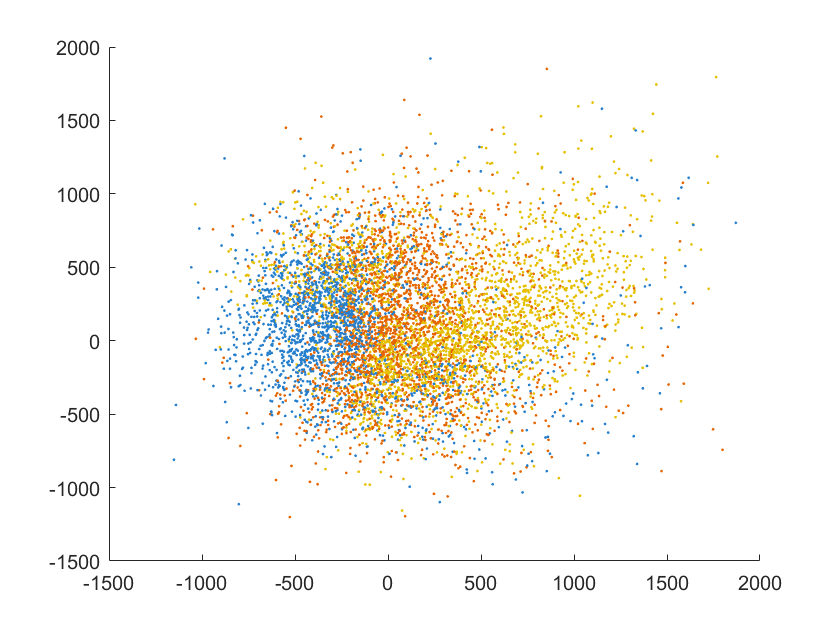} &
\includegraphics[width=.32\textwidth,valign=c]{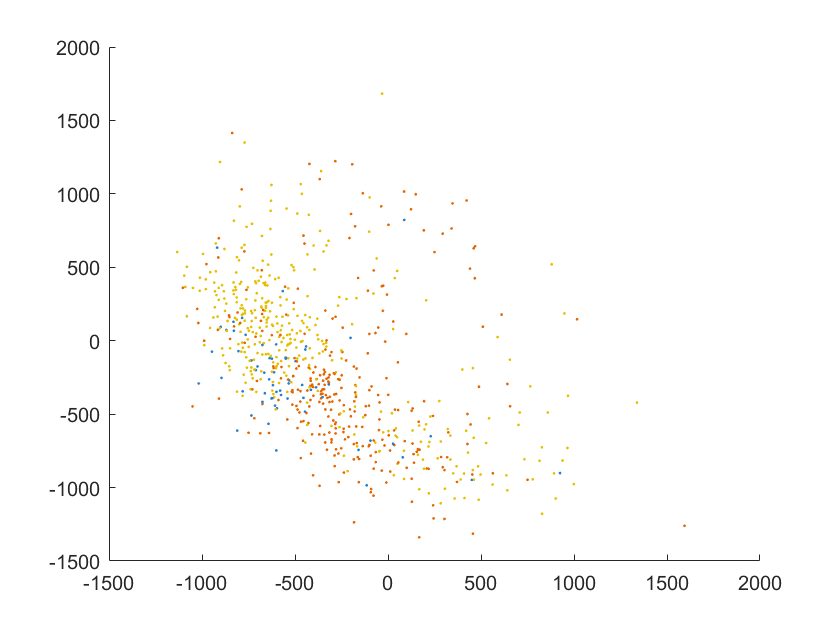} 
&\includegraphics[width=.32\textwidth,valign=c]{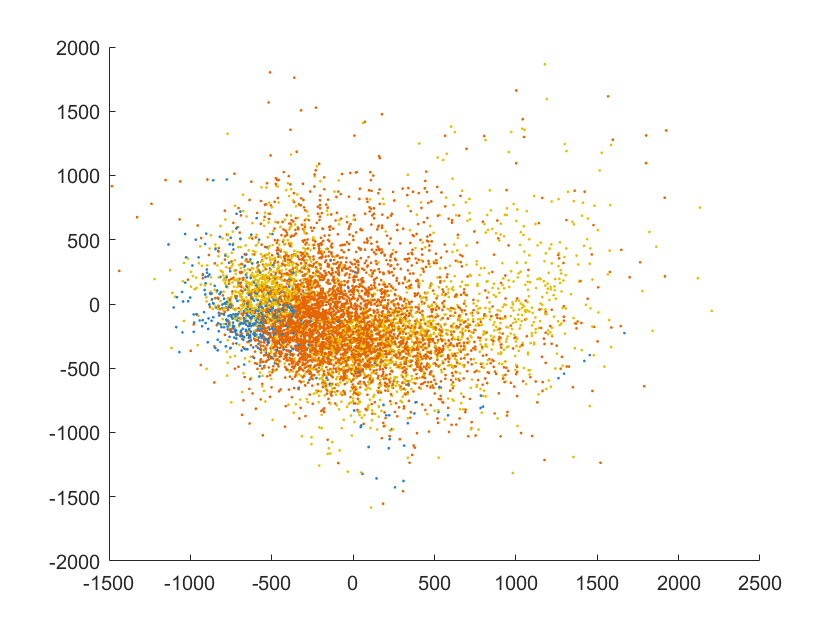} \\ \hline
\rotatebox[origin=c]{90}{LGD F1=0.53}&  \includegraphics[width=.32\textwidth,valign=c]{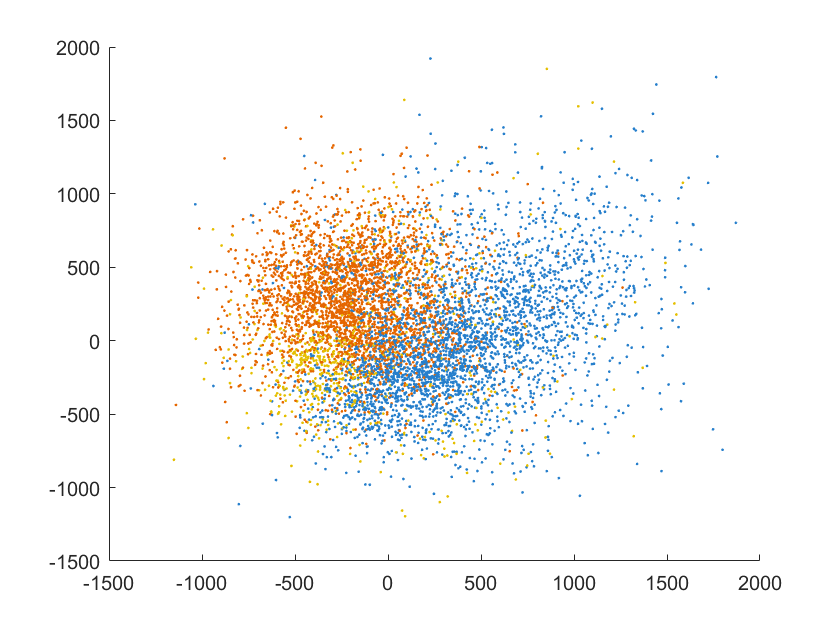} &
\includegraphics[width=.32\textwidth,valign=c]{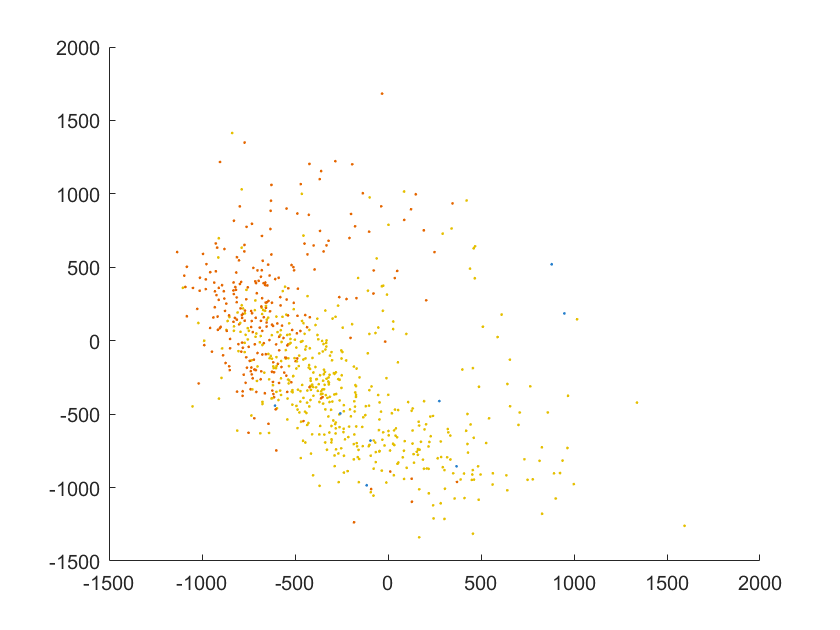} 
&\includegraphics[width=.32\textwidth,valign=c]{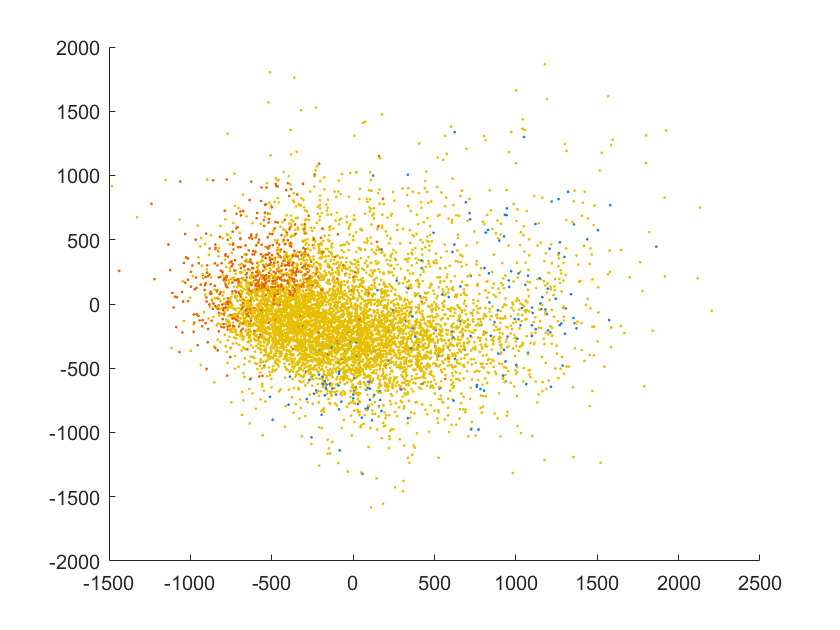} \\ \hline
\rotatebox[origin=c]{90}{GLSHC F1=0.49}&  \includegraphics[width=.32\textwidth,valign=c]{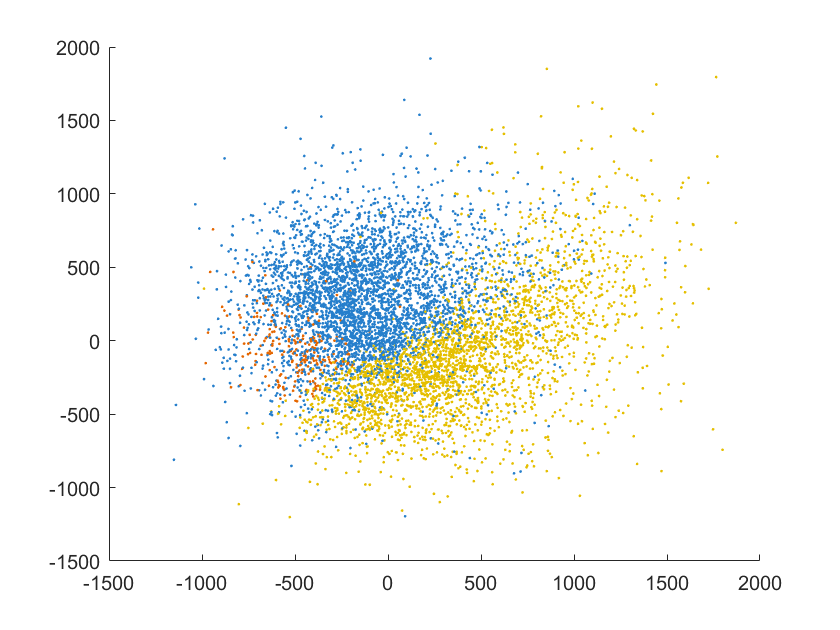} &
\includegraphics[width=.32\textwidth,valign=c]{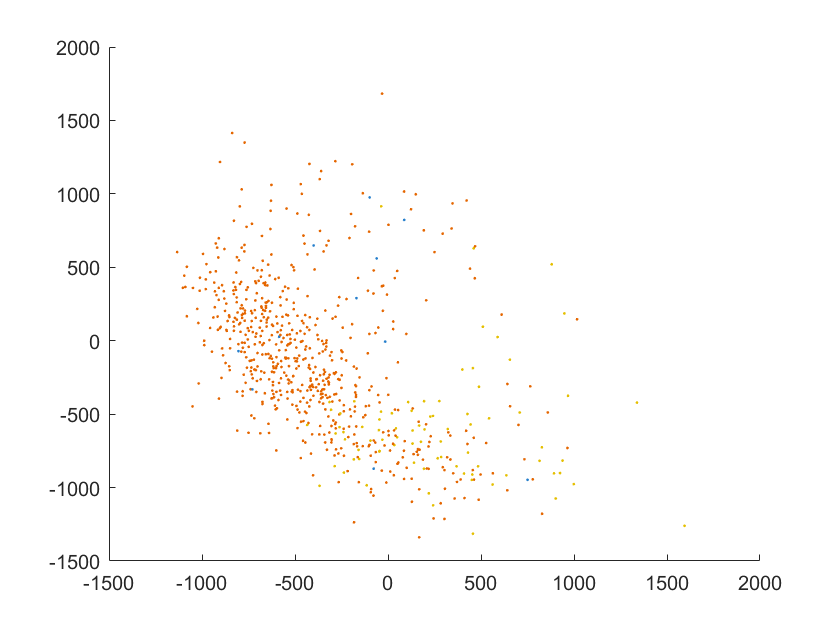} 
&\includegraphics[width=.32\textwidth,valign=c]{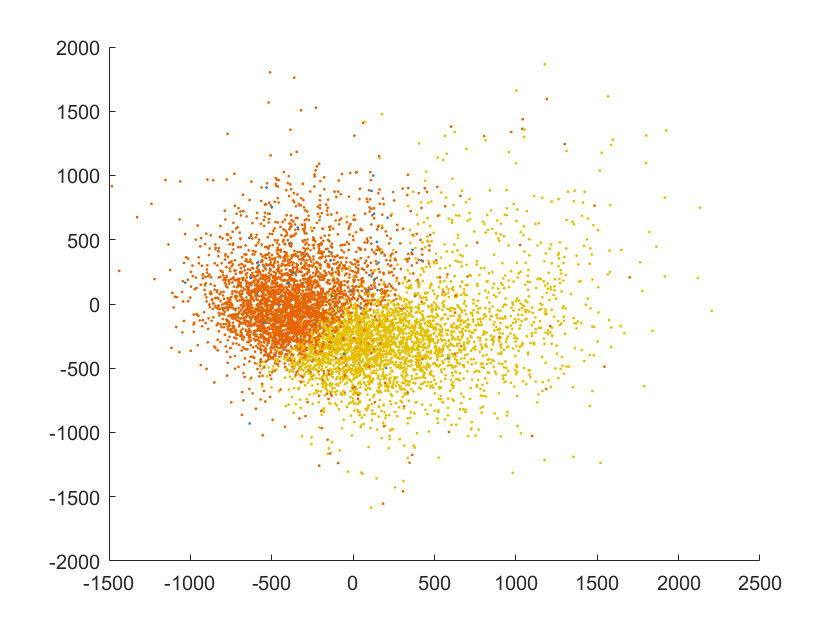} \\ \hline
  \end{tabularx}

\end{table}

The clustering outcomes of MMC, DMC, LGD and GLSHC are shown in the last four rows in Table \ref{tab-mds479}. As there are significant overlaps among the three clusters (of the hand-written digits 4, 7 \& 9), no clustering algorithms can perform very well. However, MMC could deal with the dense and sparse clusters a lot better than the density-based clustering LGD \& MMC and spectral clustering GLSHC. MMC has the least errors in two out of the three clusters (digits 4 \& 7), producing F1= 0.68. The closest contender is LGD. Though it has less errors than MMC on the digit 9 cluster, it has much more errors on the other two clusters, producing F1= 0.53.  The other two algorithms have much more clustering errors than MMC in all three clusters.

This ablation study provides another example of the impact of varied density on density-based clustering algorithms, and the superior ability of using mass distribution in clustering on a real-world dataset.

\subsubsection{Study 2: The effect of post-processing (step 3) of MMC}

The effect of post-processing (step 3) of MMC is shown in Figure \ref{fig:ami_imp}. Fourteen out of the twenty datasets have less than 0.1\% or no improvement of AMI due to post-processing. These are the datasets in which representative samples can be obtained in step 1. Post-processing has exerted an significant improvement on the first four datasets because the samples obtained in step 1 are not representative enough. These are due to their peculiar distributions and possibly high dimensions (on the gisette, w50Gaussian and w10Gaussian datasets).   

Example improvements on three datasets in terms of AMI and total mass (Eq (\ref{eqn-objective-function})), before and after post-processing, are shown in Table \ref{tab:post}.


    
\begin{figure}[b]
    \centering
    \includegraphics[width=0.8\textwidth]{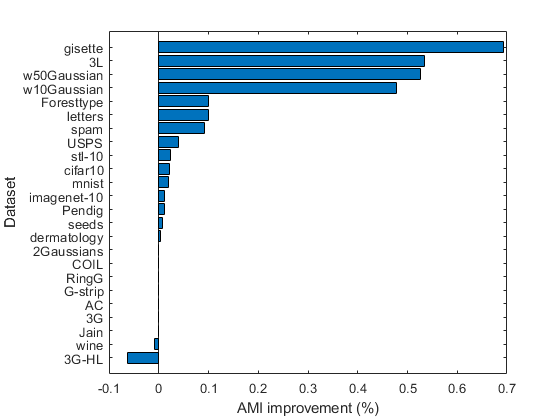}
    \caption{Improvement due to  post-processing in terms of percentage of AMI before post-processing (average results over 5 trials). The influence of post-processing is small or none on many datasets except for the first 4 datasets. }
    \label{fig:ami_imp}
\end{figure}
\begin{table}[b]
    \centering
        \caption{Example improvements due to post-processing (pp) on MMC (on a single trial).}
    \label{tab:post}
    \begin{tabular}{l|c|c|c|c} \hline 
        \multirow{2}{*}{Dataset} &\multicolumn{2}{c|}{AMI} &\multicolumn{2}{c}{Total mass (normalized)} \\ \cmidrule{2-5}
        & without pp& with pp &without pp& with pp\\ \hline
        COIL & 0.96& 0.96& 0.3713& 0.3713\\ 
        imagenet10 & 0.87 & 0.88 & 0.1887&0.1894\\
        gisette & 0.16& 0.58&  0.1740& 0.1988\\ \hline
    \end{tabular}
\end{table}

\subsection{Assessment using a Spatially Transcriptomics dataset}

Spatial Transcriptomics (ST) are a key tool in profiling gene expression and spatial information in tissue samples \cite{marx2021method}. Given a ST dataset, a major task is to find the clusters so that further analyses on the dataset can be conducted. 


A ST dataset,  containing tissue samples from a healthy human brain derived from  the dorsolateral prefrontal
cortex (DLPFC) domain \cite{pardo2022spatiallibd,maynard2021transcriptome} (\url{http://spatial.libd.org/spatialLIBD/}), is used to compare the clustering capabilities of different algorithms.

\begin{figure}[h]
    
    \centering
    \subfloat[Ground truth]{\includegraphics[width=.3\textwidth]{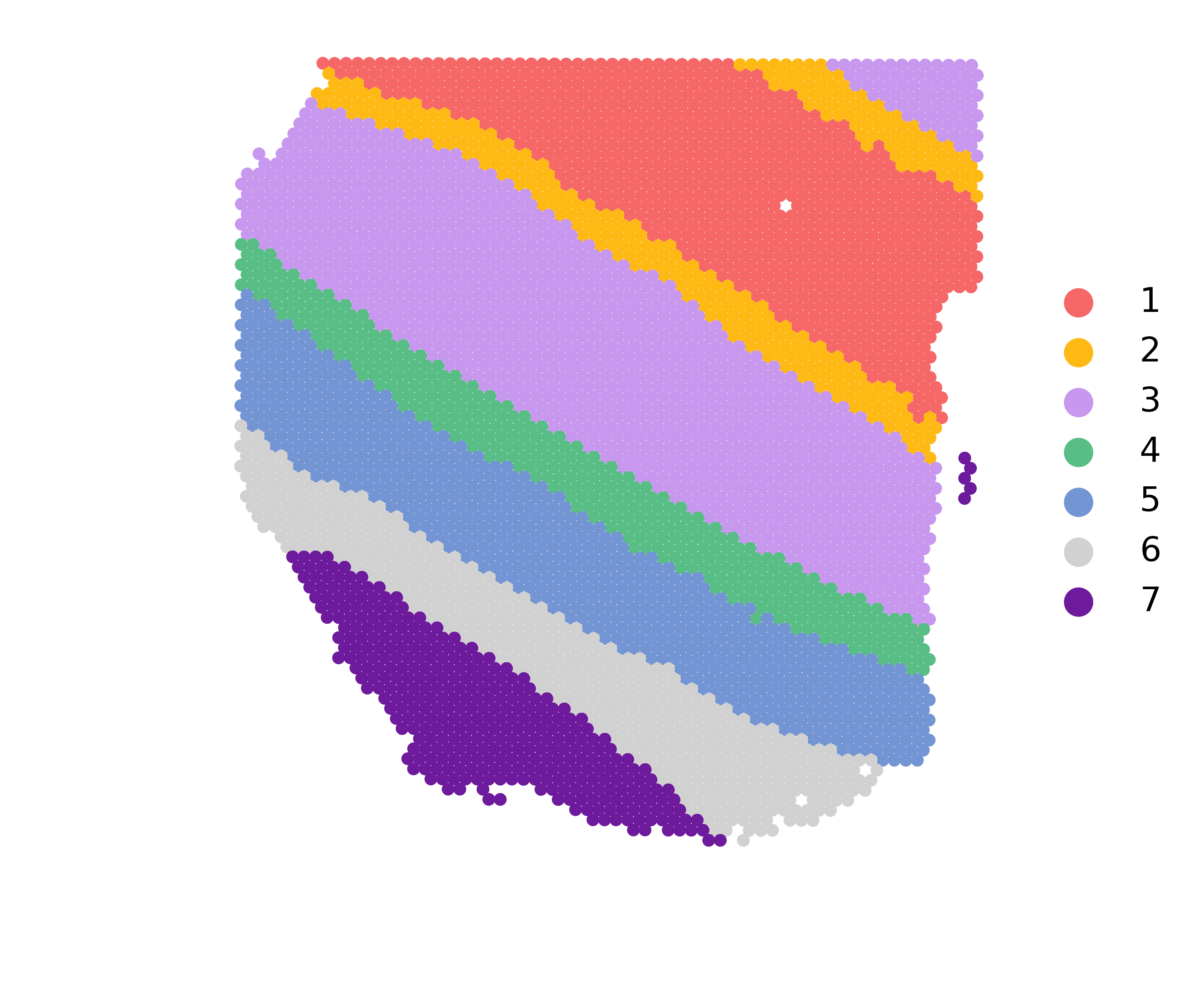} }
    \hspace{0.5cm}
     \subfloat[MMC F1=0.69]{\includegraphics[width=0.3\textwidth]{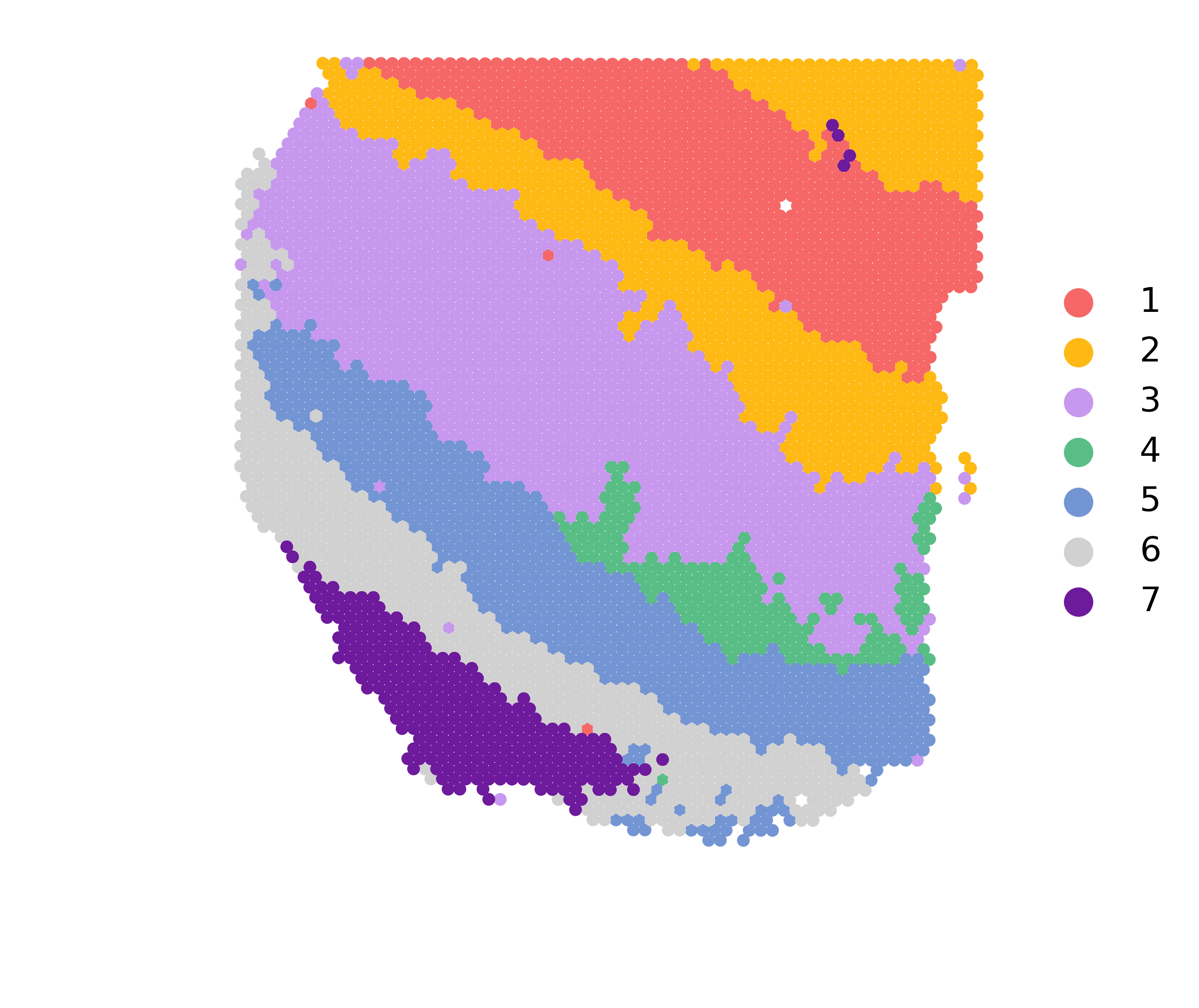}}
    \hspace{0.5cm}
    \subfloat[DMC F1=0.63]{\includegraphics[width=0.3\textwidth]{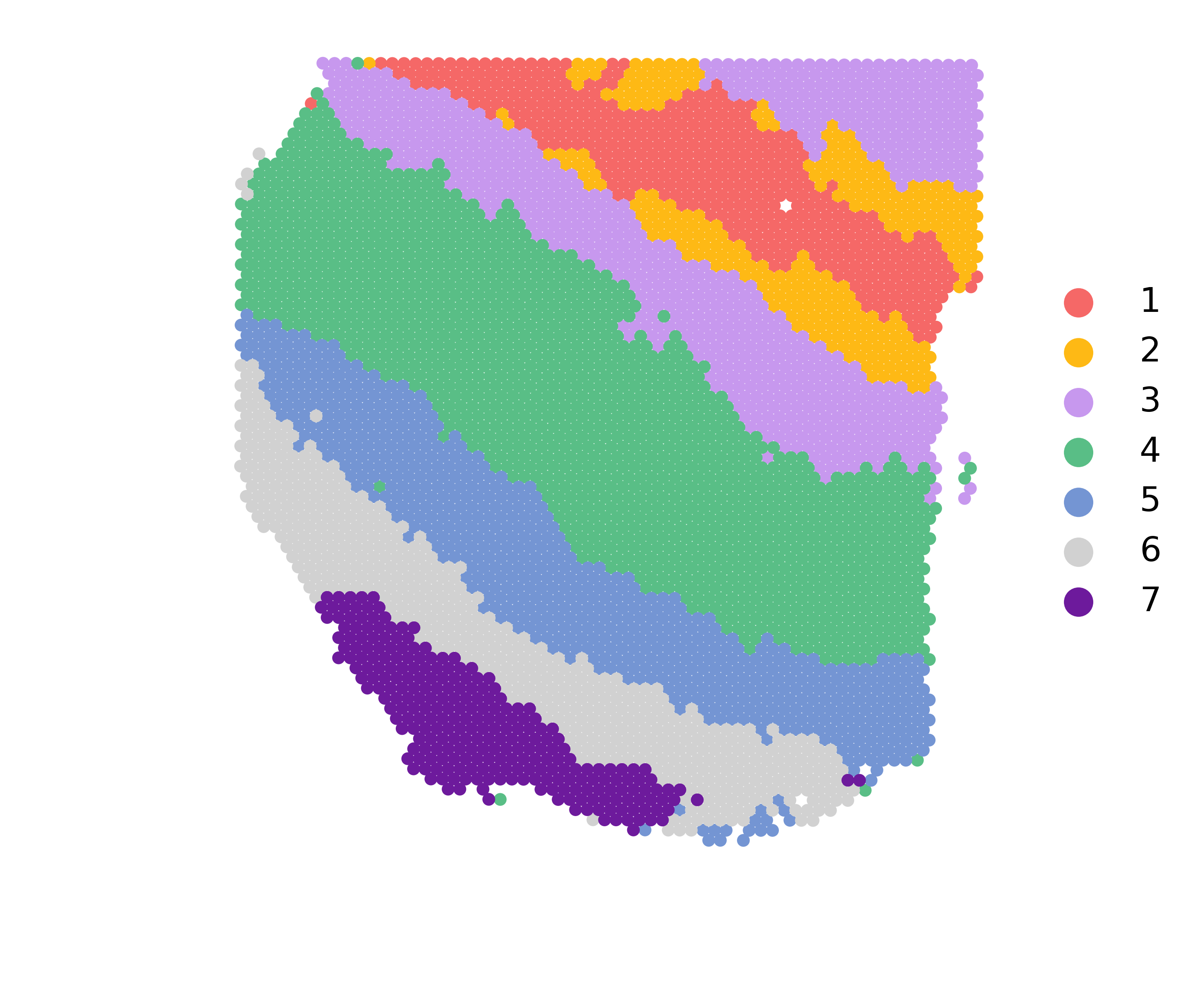}}
   
    \subfloat[DP F1=0.35]{\includegraphics[width=0.3\textwidth]{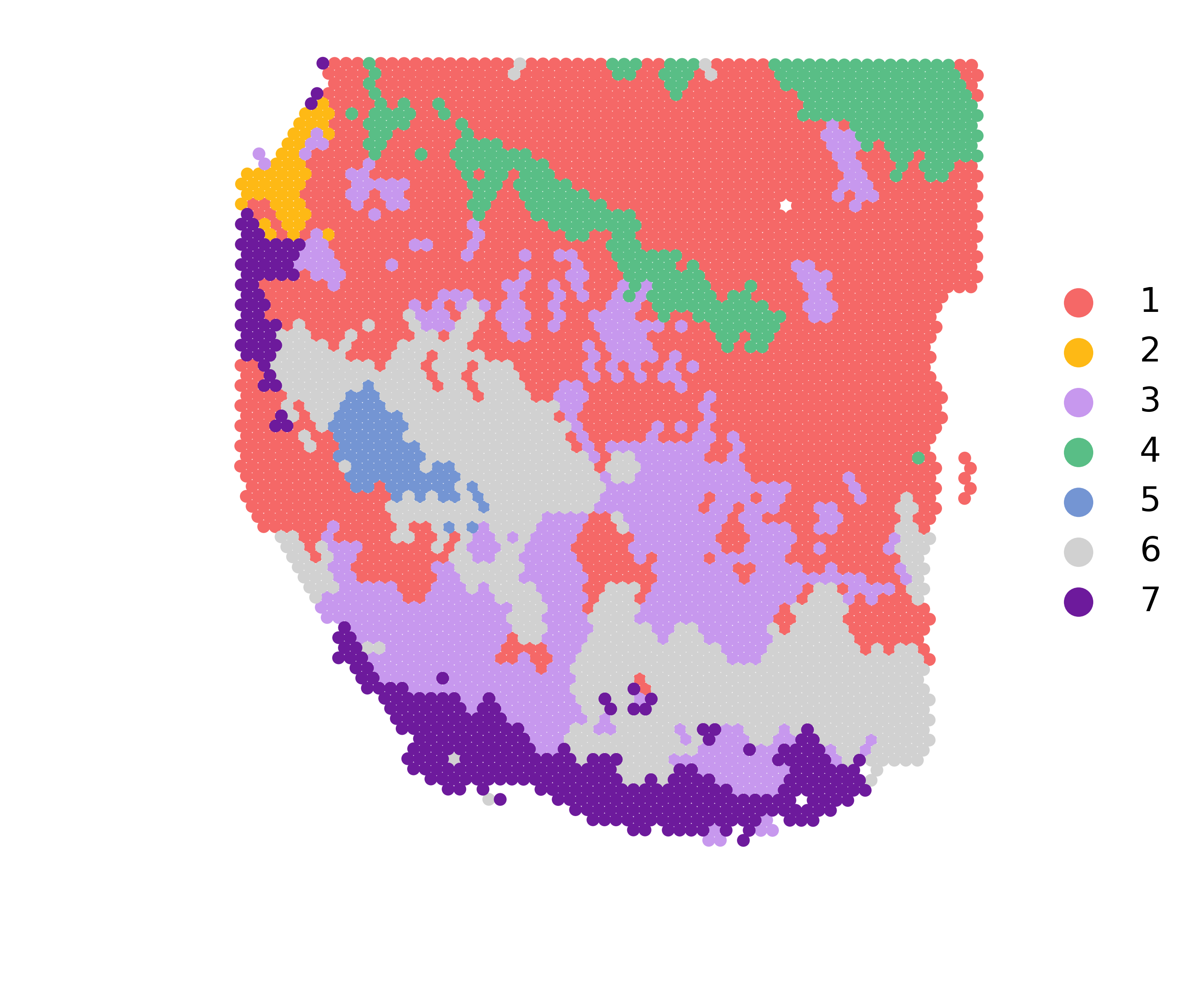}}
    \hspace{0.5cm}
    \subfloat[LGD F1=0.42]{\includegraphics[width=0.3\textwidth]{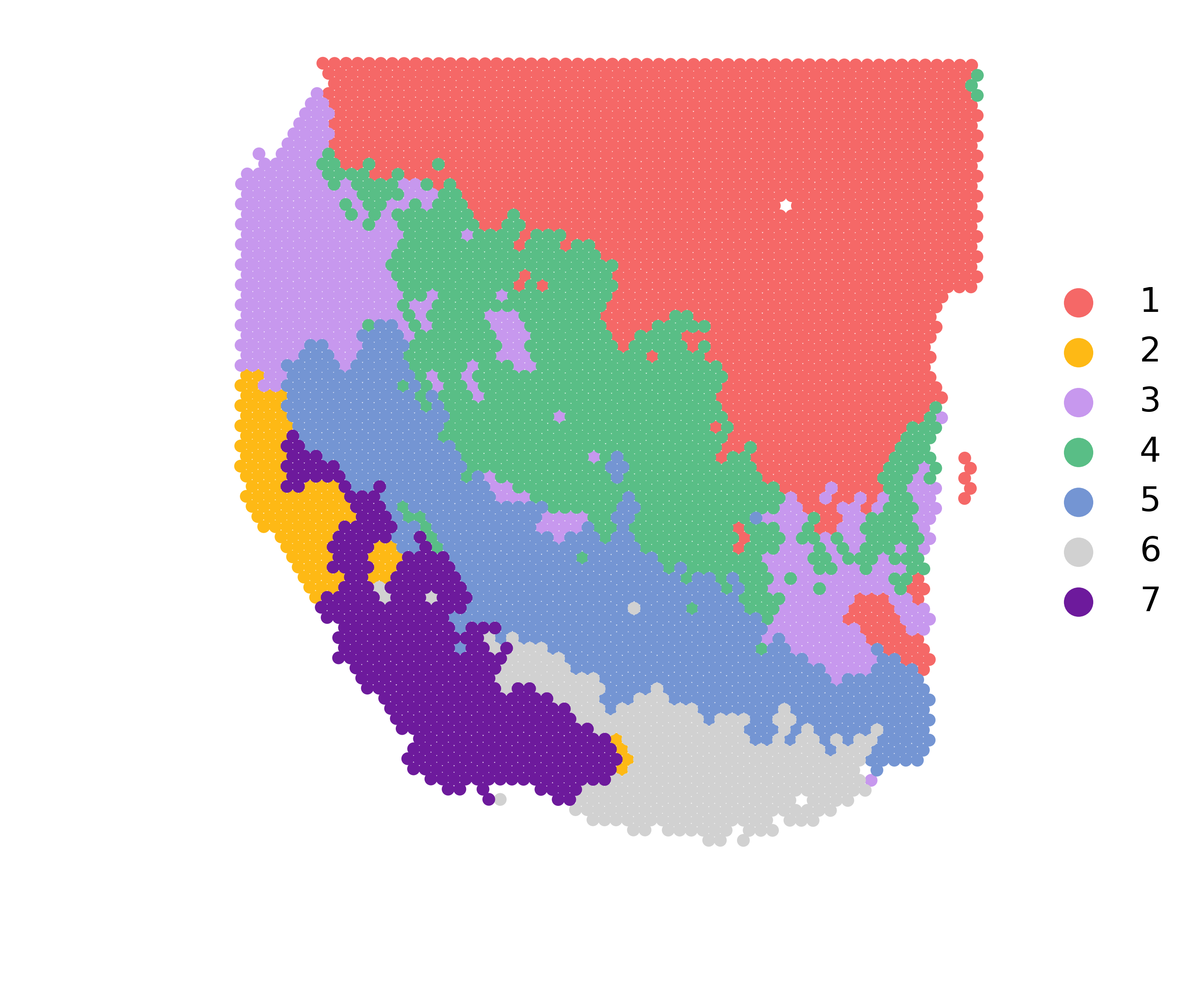}}
    \hspace{0.5cm}
    \subfloat[GLSHC F1=0.67]{\includegraphics[width=0.3\textwidth]{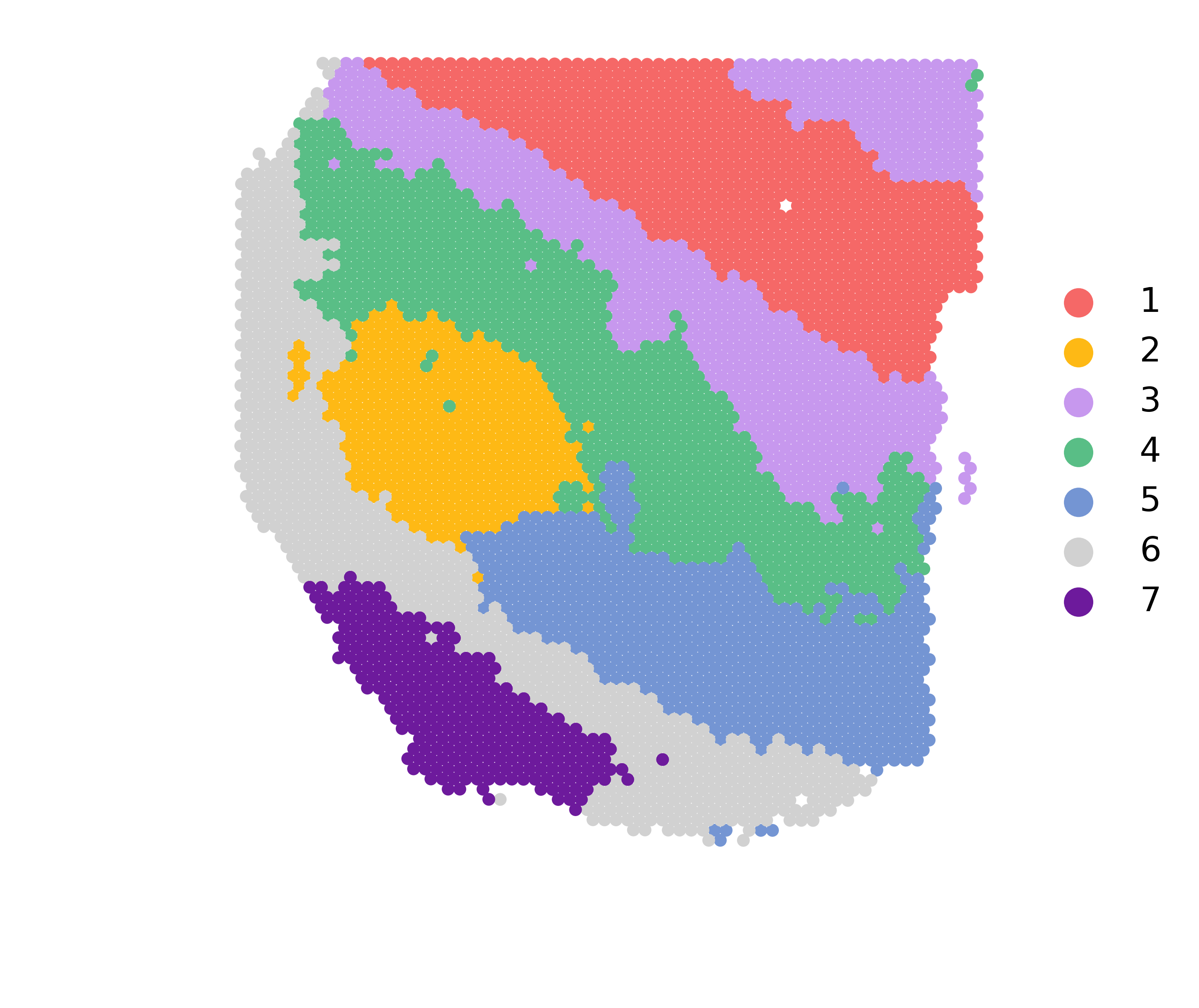}}    
  \caption{An application to a Spatial Transcriptomics (ST)  sample\_151507 on the DLPFC dataset. We ran each algorithm 5 times and show the results with the highest F1 score. SpatialPCA \cite{shang2022spatially} is used to perform the embedding of the original ST sample; and the same embedded dataset is used as the input to all clustering algorithms.}
  \label{fig:DLPFC}
\end{figure}

Figure \ref{fig:DLPFC} shows the clustering outcomes of five clustering algorithms. MMC has the best clustering outcome, shown in Figure \ref{fig:DLPFC}(b), which identifies most layers of clusters in the dataset that closely match those in the ground truth shown in Figure \ref{fig:DLPFC}(a). 
Visually, DMC is a close second, even though it has a slightly lower F1 score than GLSHC because the former finds better layers of clusters than the latter. Both LGD and DP have the poorest outcomes with many mixed clusters without clear layers.

\subsection{Parameter sensitivity study}

The parameter sensitivity of MMC on two datasets is shown in Figure \ref{fig:sensitivity}. This shows that MMC is not too sensitive to the parameter settings of $\tau$ and $\psi$.

\begin{figure}[h!]
\vspace{-8mm}
    \subfloat[COIL]{\includegraphics[width=0.45\textwidth]{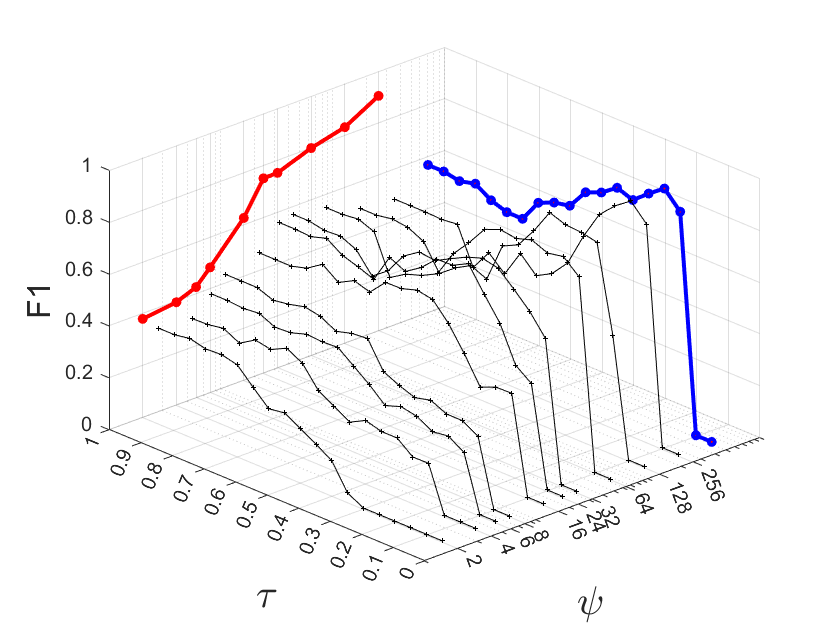}}
    \hspace{1mm}
     \subfloat[Pendig]{\includegraphics[width=0.45\textwidth]{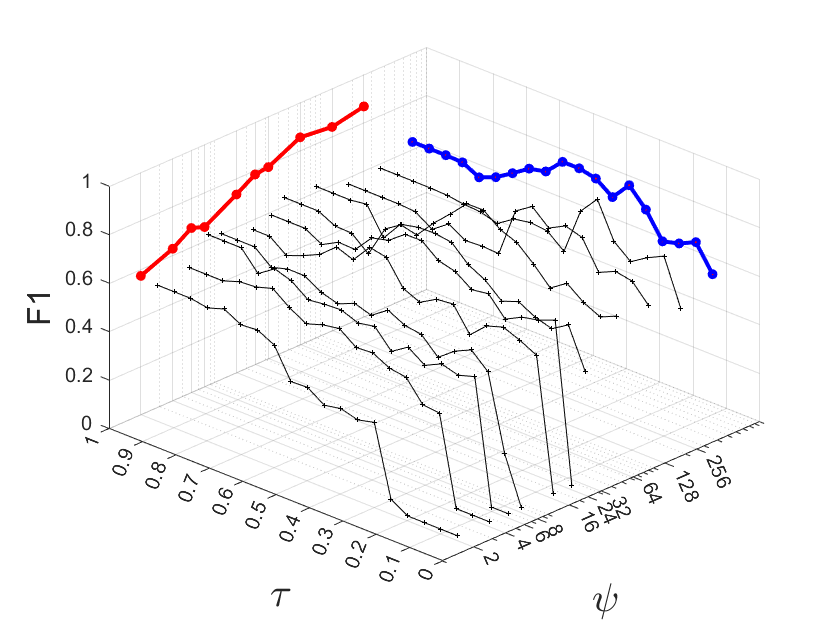}}

 \caption{Parameter sensitivity of MMC. The best clustering results for a $\psi$ setting (over the range of $\tau$) and for  a $\tau$ setting (over the range of $\psi$) are shown 
 in red and blue lines, respectively. }
  \label{fig:sensitivity}
\vspace{-6mm}
\end{figure}

\subsection{The effect of sampling size on MMC}

\textcolor{black}{Note that the use of a small representative sample set in MMC is different from that typically used in the literature, where sampling is often used to improve efficiency, knowing that it reduces its task-specific accuracy. Examples are: subsampling approximation is used to (i) reduce the learning cost in constructing minimal enclosing sphere in Support Vector Clustering \cite{SVC-TKDE2015}; (ii) enable a high computational algorithm such as DP to run on a huge dataset that would otherwise be impossible (see Section 7.2 in \cite{psKC-2023}); and (iii) significantly reduce the eigendecomposition cost in spectral clustering via a small set of anchor points rather than the entire dataset \cite{kang2021structured,yang2021graphlshc} (as discussed in Section \ref{sec-why-fail}).\\
In contrast, mass-based methods, including its first method called Isolation Forest \cite{liu2008isolation}, rely on small samples to do well; in fact, the model trained from a large set performs poorer, defying the conventional wisdom that more data the better. A theoretical analysis \cite{LearningCurve} on a nearest neighbor anomaly detector reveals that the sample size has the following impacts. First, increasing the sample size increases the chances of including anomalies in a training set, leading to a lower detection accuracy of the trained model. Second, the optimal sample size is the one which best represents a data distribution or the geometry of normal instances and anomalies, producing the optimal separation between normal instances and anomalies. Increasing the sample size beyond the optimal size reduces the separation between normal instances and anomalies, leading to decreased detection accuracy. See \cite{LearningCurve} for more details.\\
The above discussion also point to the fact that the $s$ setting in MMC shall not be proportional to the dataset size, unlike the setting of sample set size when sampling is used as a means to trade-off accuracy for efficiency mentioned above.
As a result, MMC is able to deal with large datasets such that one can set $s \ll n$, where $n$ is the dataset size. None of the existing density-based clustering algorithms have the same ability.
}

\textcolor{black}{Here we show the $s$ settings required on two example datasets: The first is the simple 2Gausssians dataset (having 1,000 data points) and the second is the real-world mnist dataset (having 100,000 data points). The clustering results of MMC with different $s$ settings are shown in Figure \ref{fig:s-setting}. On both datasets, with the optimal settings, MMC produces the best F1 results. Note that as 2Gaussians is a simple dataset, only $s=50$ is sufficient; whereas $s=2,000$ is required on the more complex mnist dataset. In both cases, $s \ll n$. 
This demonstrates that the representative sample set size required depends on the data distribution (but not proportional to a given dataset size) in order to produce a good clustering outcome.\\
It is interesting to note that using a sample set size larger than the optimal size has no benefit, and it could be counter-productive, leading a worse clustering outcome than that with the optimal $s$ setting. This is shown on the mnist dataset.}

\begin{figure}[h!]
     
    \includegraphics[width=0.48\textwidth]{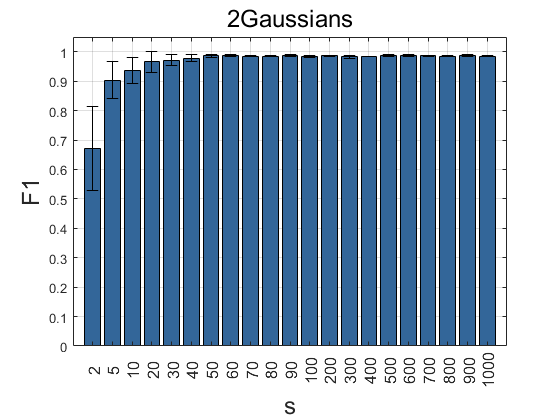}
    \includegraphics[width=0.48\textwidth]{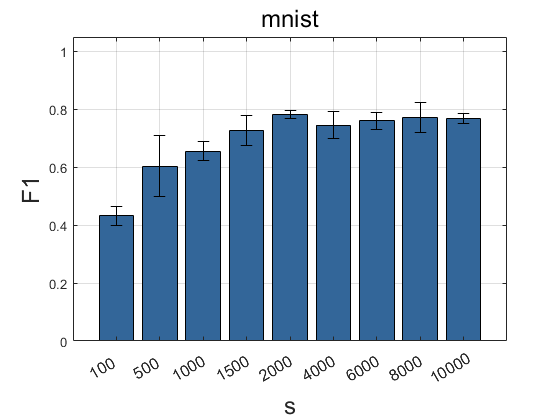}
 \caption{The effect of sampling size on MMC. Each error bar indicates the standard deviation of the F1 results over 5 trials in each $s$ setting. Note that the standard deviations are significantly smaller on the optimal $s$ settings than those on the sub-optimal settings. }
  \label{fig:s-setting}
\vspace{-6mm}
\end{figure}

\subsection{Scaleup test}


Figure \ref{fig:scalup} shows the scaleup test of mass and density based algorithms as well as spectral clustering.
\textcolor{black}{GMM has the lowest runtime, followed by GLSHC. It is interesting to note that while HDBSCAN$^*$ exhibits linear time complexity when the dataset size increases 100 times from $1500$ to $1500 \times 10^2$, it has quadratic time complexity when the dataset size further increases 100 times to $1500 \times 10^4$. This shows that its worst-time complexity is quadratic.}
MMC and DMC have linear time complexity and approximately the same runtime, so as SGL and GLSHC if the number of anchors $a$ employed is significantly smaller than the dataset size $n$. 
These results are consistent with the time complexities of all algorithms shown in Table \ref{tab:time-complexities}.

\begin{figure}[h]
    \centering
        \includegraphics[width=0.49\linewidth]{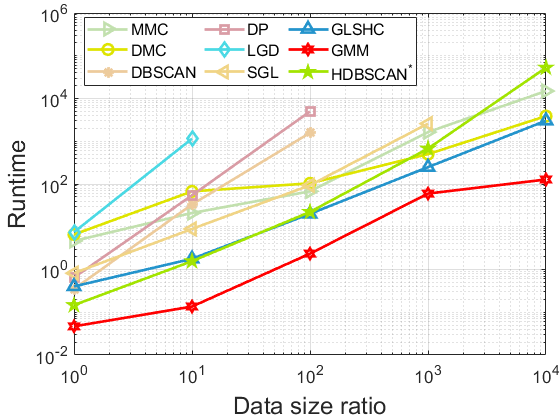}
        \includegraphics[width=0.49\linewidth]{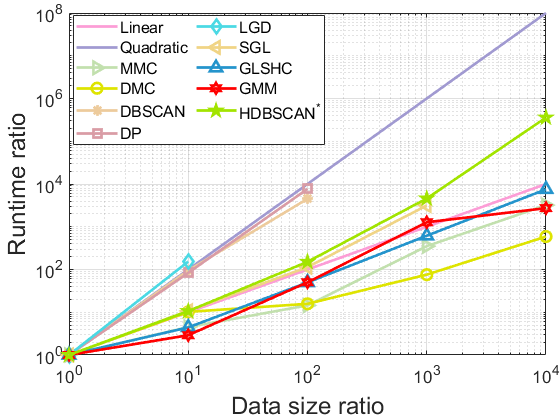}
    \caption{Scaleup test result of actual run time (left) and runtime ratios (right). Data size ratio = 1 denotes that a dataset of 1500 points is employed. The  $10^4$ ratio is conducted using a dataset of 15 million points.  DP and DBSCAN could afford to run up to a dataset of 0.15 million points only. LGD did not complete the run on a dataset of 0.15 million within 2 days. }
    \label{fig:scalup}
\end{figure}

\section{Relation to other kernel-based clustering algorithms}
\label{sec-KBC}

IDKC \cite{IDKC-IS2023} and psKC \cite{psKC-2023}  are two recent clustering algorithms based on kernel, closely related to MMC. They are motivated to improve kernel k-means \cite{dhillon2004kernel,KNNKernel,Scalable-kMeans-JMLR19} and Laplacian k-modes \cite{Scalable-kModes-NEURIPS2018}, respectively. In contrast, MMC is motivated to use mass distribution to describe data distribution, as opposed to the typical density distribution.

Though all MMC, IDKC and psKC share the same treatment by using a kernel to represent each cluster as a distribution, this paper provides two significant breakthroughs over IDKC and psKC:
\begin{itemize}
    \item The first breakthrough of MMC is the use of mass distribution to interpret the distribution represented via the Isolation Kernel. Without this conceptual breakthrough, the reason why these clustering algorithms could work well in clusters of varied densities cannot be explained satisfactorily. IDKC and psKC  have attributed this ability to the data dependency of the Isolation Kernel without further explanation. Using mass distribution, we can now fully explain this ability via  highest-mass clusters and the mass-maximization objection function which produce clusters with approximately the same average cohesiveness $\bar{S}(C^\tau)$ (see the next point).   
    \item IDKC and psKC  have the first step initialization based on one peak for each cluster. As a result, multiple iterations are required in the second step to grow the clusters.

    The second breakthrough of MMC is the use of $\tau$-cohesive clusters $C^\tau$ as the initial clusters. This enables the cluster assignment process to be completed in one iteration because  $\tau$-cohesive clusters are much better representatives, than single-point peaks, of clusters of arbitrary shapes, sizes and densities.  Because all clusters have the same cohesiveness (as stated in Proposition \ref{Lemma-cohesiveness}), regardless of the densities of the clusters, we are able to explain that both the initial and final clusters have no bias towards dense clusters, unlike the clusters discovered by density-based clustering algorithms. 
\end{itemize}
Both IDKC and psKC  do not use the concept of mass distribution to explain their operational principles. Yet, given our revelation, they can be interpreted as mass-based clustering methods. 
This is because both these algorithms and MMC  use the same Isolation Kernel \cite{IDK} to discover clusters by treating each cluster as a mass distribution, albeit the algorithmic details differ. In fact, all three of them share the same objective function.

A comparison between MMC and IDKC is provided in Table \ref{tab-idkc}.
IDKC has comparable clustering performance with MMC, except on Jain, 3G-HL, RingG and COIL. The inferior clustering
performance of IDKC on these four datasets is largely due to the use of peaks (instead of representative samples) of clusters.

\begin{table}[b]
    \centering
        \caption{MMC vs IDKC and DBSCAN vs MBSCAN.}
    \label{tab-idkc}
    		\renewcommand{\arraystretch}{1}
		\setlength{\tabcolsep}{4pt}
\begin{tabular}{l|rr|rr|rr}
    \hline 
     & \multicolumn{4}{c|}{F1} & \multicolumn{2}{c}{AMI} \\ 
  Dataset & MMC & IDKC      & DBSCAN    & MBSCAN  & MMC & IDKC  \\ \hline 
  Jain     &1	  &    0.93	  & 1         & 1  &1	  &  0.83\\
3L         &0.84	&0.82	& 0.59      & 0.95     &0.59	&0.59  \\
3G-HL      &0.97	&0.83	& 0.87      & 0.98    &0.90	&    0.72 \\
3G         &0.98	&0.98	& 0.57      & 0.99 &0.91	&0.91 \\
2Gaussians &0.99	&0.99	& 0.83      & 0.98 &0.91	&0.92 \\
AC         &1	&1	& 1         & 1    &1	&0.96 \\
G-Strip    &0.97	&0.92	& 0.70      & 0.98 &0.81	&0.67 \\
RingG      &1	  &  0.85	& 0.67      & 1 &0.99	&0.87 \\
w10Gaussian&1	  &  1	    & 0.34      & 0.41 &1	    &0.98 \\
w50Gaussian&1	  &  1	    & 0.33      & 0.46 &1	    &0.99 \\ \hline
Average    & 0.98 & 0.93    & 0.69      & 0.88 & 0.91   & 0.84 \\ \hline
wine &   0.95& 0.95        & 0.72      & 0.96 & 0.83&0.84\\ 
seeds  & 0.92   & 0.92      & 0.76      & 0.93 & 0.73&0.74 \\ 
dermatology & 0.91 &0.94    & 0.52      & 0.91 &0.88&0.92\\ 
Foresttype & 0.83&0.85      & 0.25      & 0.84 &0.59&0.63\\ 
COIL&0.91	&0.73	        & 0.84      & 0.90 &0.95	&0.84 \\
spam       &0.75	&0.78	& 0.38      & 0.64 &0.21	&0.26 \\
gisette    &0.91	&0.91	& 0.01      & 0.59 &0.59	&0.56 \\
Pendig      &0.87	&0.83	& 0.70      & 0.61 &0.84	&0.80 \\
 USPS   &0.73	&   0.69   & 0.27      & 0.31 &0.77&	   0.72 \\
imagenet-10 &0.91	&0.94	& 0.85      & 0.87 &0.86	&0.88 \\
stl-10    &0.75	&0.74	    & 0.53      & 0.60 &0.66	&0.66 \\
letters    &0.40&  0.36	    & 0.29      & 0.43 &0.51	&0.48 \\ 
cifar10    &0.74	&0.70	& $^\dagger$0.01  & \textcolor{black}{NC} &    0.71	&0.70 \\
mnist      & 0.77   &\textcolor{black}{NC} & $^\dagger$0.01 & \textcolor{black}{NC}& 0.74 & \textcolor{black}{NC}\\
\hline
Average    & 0.81   & 0.80   & 0.44     & 0.72 & 0.71   & 0.69 \\ \hline
    \end{tabular}
\end{table}

\textcolor{black}{An early kernel-based clustering is Maximum Margin Clustering  \cite{SVC-JMLR2002,xu-nips-2004-MMC,zhang-MMC-aaai-2018-ODMC} which borrows the idea of Support Vector Machine (SVM) in classification to perform clustering. It formulates the clustering problem as an optimization that maximizes the margin between clusters. This approach differs from the mass-based methods described above in two key aspects. First, Maximum Margin Clustering does not consider clusters as distributions and does not use a distributional kernel to represent distributions. Second, it is a typical kernel based method which must rely on an optimization procedure to optimize an objection function. None of the mass-based clustering methods thus far need an optimization procedure to achieve the objective stated in Definition~\ref{def-objective}.\\
On another note, it is possible to use the IK instead of the Gaussian Kernel in Maximum Margin Clustering, as shown in SVM when IK was first introduced \cite{ting2018IsolationKernel}. MMC is a much simpler and efficient clustering than Maximum Margin Clustering, and MMC  has an arguably better objective function which can be achieved without optimization.  
}

\section{Discussion}

\subsection{Related recent works on kernel mass estimation and kernel density estimation} 
\label{sec-MBSCAN}
MBSCAN \cite{IsolationKernel-AAAI2019} converts the $\epsilon$-neighborhood density estimator (used in DBSCAN) into a mass estimator by replacing the Euclidean distance with the distance version of the Isolation Kernel. This work demonstrates the impact of mass estimation in DBSCAN using exactly the same algorithm, uplifting its clustering performance significantly \cite{IsolationKernel-AAAI2019}. \textcolor{black}{This is consistent with the comparison results between DBSCAN and MBSCAN we have presented in Table \ref{tab-idkc}, where MBSCAN uplifts DBSCAN's clustering F1 results in almost all datasets, some with very large margins, e.g., 3L, 3G, RingG, dermatology, Foresttype and gisette.}

Our work differs in three key aspects. First, the earlier work does not explain mass estimator from its fundamental. Section \ref{sec_Mass} provides this fundamental without referring to an existing density estimator. Second, MMC is a brand new algorithm, and it uses the mass-maximization criterion to form the final cluster, unlike MBSCAN/DBSCAN in which the same algorithmic shortcoming (of using point-to-point linking) remains. \textcolor{black}{This is the reason why MBSCAN still performs significantly worse than MMC on quite a number of datasets shown in Table \ref{tab-idkc}, e.g., w10Gaussian, w50Gaussian, spam, gisette, Pendig, USPS and stl-10. In general, DBSCAN is weaker than the more recent density-based algorithms such as DP and DMC, as shown in Tables \ref{tab:f1} and \ref{tab:ami}. It is clear from the results in Table \ref{tab-idkc} that MBSCAN is weaker than MMC, echoing the relative performance of their density counterparts.} Third, MMC has linear time complexity, whereas MBSCAN and DBSCAN have quadratic time complexity.

There are a number of improvements on kernel density estimators (see e.g., \cite{VariableKDE1992,SOMKE-2012,AdaptiveKDE-2018,DiffusionKDE2010}).
Recent advances in kernel density estimation have significantly improved the time complexity from quadratic to linear (see e.g., \cite{RACE-WWW2020,IKDE-ICDM2021}). MMC has made use of this advancement, though not directly from the perspective of kernel density estimator, to achieve the linear time complexity.

\subsection{Relation to mass estimation and mass-based similarity}
\label{sec-mass}
Mass estimation \cite{mass-estimation-KDD2010,Mass-estimation-MLJ2013} was proposed to be an alternative to density estimation to better model data distribution for data mining and machine learning. Our work here is the first to use the IK based mass estimation to explain its superior data distribution modelling via the notion of cluster cohesiveness in the context of clustering. 

Historically, the idea of mass estimation \cite{Mass-estimation-MLJ2013} was conceived before the introduction of mass-based similarity \cite{LMN-MLJ2019}. But the Isolation Kernel \cite{ting2018IsolationKernel} (a counterpart of mass-based similarity) came before the IK based mass estimation we proposed here. 
As mass-based similarity is a direct ancestor of the Isolation Kernel (IK), it is no surprise that  mass estimation derived from IK (stated in Section \ref{sec_Mass}) has a strong connection to the previous versions of mass estimation \cite{mass-estimation-KDD2010,Mass-estimation-MLJ2013}. 


Recall that the mass estimator, shown in Equation \ref{eqn_mass} or \ref{eqn_mass2}, does not need to be defined in terms of the Isolation Kernel. Indeed, any of the previous mass estimators \cite{Mass-estimation-MLJ2013,LMN-MLJ2019} could be used here. The use of IK produces clusters of approximately the same cohesiveness, irrespective of densities, shapes and sizes of clusters; and it yields $\tau$-cohesive clusters that are representative samples of all clusters in a dataset (as stated in Sections \ref{sec-cohesiveness} \&
\ref{sec-representative-sample}). Both are essential in the first step of the MMC algorithm. The proposed IK based mass estimator ensures that the objective of mass maximization is achieved efficiently in the last two steps of the MMC algorithm. Because any of the previous versions of mass estimation has high time complexity, their use would have a serious repercussion on the MMC's runtime. 

\textcolor{black}{An earlier work using a primitive version of mass-based similarity has been applied to DBSCAN, by simply replacing the Euclidean distance with it, to create MBSCAN \cite{Mass-based-similarity-KDD2016,LMN-MLJ2019}; and then IK is used instead to create MBSCAN \cite{IsolationKernel-AAAI2019}. In all these cases, MBSCAN has been shown to outperform DBSCAN in datasets of varied densities solely due to the use of this mass-based similarity or IK.} 
\textcolor{black}{A data dependent measure called Shared Nearest Neighbors (SNN) \cite{Jarvis-Patrick-1973} (which relies on $k$-nearest neighbors to determine SNN) has been used to replace the Euclidean distance in DBSCAN \cite{SNN-DBSCAN-SDM2003}. As revealed previously, when using SNN in the $\epsilon$-neighborhood density estimator, DBSCAN becomes a mass-based clustering algorithm \cite{Mass-based-similarity-KDD2016}. Compared to 
the iForest-based mass estimator \cite{Mass-based-similarity-KDD2016}, SNN has two shortcomings, i.e., SNN is very sensitive to the $k$ setting and has $\mathcal{O}(k^2n^2)$ time complexity. (Further discussion can be found in \cite{Mass-based-similarity-KDD2016}).
}


\subsection{Cohesive clustering and cohesion} 
The terms `cohesive clustering' and `cohesion' have been used in different contexts in the literature. More often than not, they are used without a definition. For example, `Cohesive clustering algorithm' has been used to refer to hierarchical clustering \cite{CohesiveClustering2022}, without defining what a cohesive cluster is; and the algorithms are based on a data independent distance measure. As a result, these algorithms have an issue with a dataset having varied densities, as discovered in a recent work \cite{IK-on-AHC-PRJ2023}. 

In a different usage, the term `cohesion' has been associated with Silhouettes \cite{Silhouettes1987} which is an assessment metric to measure how similar a point is to its own cluster (called cohesion) compared to other clusters (called separation). However, this metric works well for convex and compact clusters only, and it is not a good measure of clustering outcomes for clusters of arbitrary shapes \cite{kazempour2020towards}.  








Our definitions of cohesive clusters (Definitions \ref{def-tau-cohesion} \& \ref{def_same-tau-cohesion}) highlight the problem of using a data independent measure (e.g., Gaussian kernel and Euclidean distance) and its associated high-density bias. This has enabled us to pin down the root cause of many shortcomings of existing clustering algorithms, especially the density-based clustering algorithms.

\subsection{Limitations of MMC}

We envisage that there are two conditions under which MMC may fail to discover all clusters in a dataset:
\begin{itemize}
    \item Clusters that cannot be represented as distributions. This condition violates the assumption of MMC that each cluster can be represented as an unknown distribution. \textcolor{black}{An example of clusters that cannot be represented as distributions is topological clusters, where clusters are defined in terms of topological features rather than input features. In this case, a clustering method must have the ability to extract topological features in the given dataset in order to identify the clusters (e.g., \cite{TPCC-ICML2023}).}
    \item The representative samples of clusters cannot be obtained in step 1 of MMC.
\end{itemize}

\textcolor{black}{It is interesting to note that, in complex data objects such as graphs and trajectories, as long as the objects can be embedded into multi-dimensional vectors, MMC and its closely connected distribution-based clustering algorithms (stated in Section~\ref{sec-KBC}) can be expected to produce better clustering outcomes than density-based clustering and spectral clustering, especially when the embedded vectors have clusters of varied densities, as stipulated in this paper. Recent examples are clustering for Spatial
Transcriptomics data in the form of a graph \cite{GenomeResearch2025,kbc2025aij}, and clustering for a dataset of trajectories \cite{tidkc-ICDM2023}. The former uses a variant of MMC and the latter uses IDKC \cite{IDKC-IS2023}. They have been shown to produce better clustering outcomes than SOTA clustering methods, including spectral clustering and deep learning methods \cite{GenomeResearch2025,tidkc-ICDM2023}.} 

\textcolor{black}{Two other `limitations' of MCC are (a) the number of clusters to be discovered must be specified by a user; and (b) noise points are not detected. These can be resolved by minor tweaks in the algorithm, and they are not something fundamental to the use of mass distribution. For example, a previous version of MMC called psKC has already been designed to discover all clusters automatically, as in DBSCAN; and both MMC and psKC have exactly the same objective function.\\
In addition, all versions of mass-based clustering algorithms we know thus far can easily identify noise points, which have the lowest mass values already computed in the clustering process, by simply determining all points below a certain mass threshold. In fact, psKC has noise points defined exactly as we have described (see Definition 3 in the psKC paper \cite{psKC-2023}). This can be done as a post-processing and it does not affect the core members of each cluster found by the algorithm. This way of identifying noise (having the lowest mass/density) points is similar to that used in DP, though the original algorithm is designed to cluster all points in a dataset.\\
Nested clusters are in the domain of hierarchical clustering, and they are outside of the scope of this paper. We are sure that a hierarchical clustering which makes use of the distributional kernel, as used in MMC, would be able to identify nested clusters easily. This is equivalent to upgrading from DBSCAN to HDBSCAN to deal with hierarchical clustering.}

\section{Conclusions}

We establish for the first time that the Isolation Kernel is an effective means to define cluster cohesiveness and estimate mass distribution. Our definition of cluster cohesiveness enables us to affirm (a) the omnipresence of the high-density bias whenever density distribution is employed; and (b) the property that all clusters have approximately the same cohesiveness when mass distribution via the Isolation Kernel is employed; thus the absence of the high-density bias when mass distribution is used.  

We argue that, using density-based algorithms, the goal of finding clusters of arbitrary shapes, sizes and densities  is difficult to achieve. They have two fundamental issues, i.e., the high-density bias (because of the use of density distribution) and the algorithmic means of the point-to-point linking process to form the final clusters. We show that the first issue is more fundamental than the second because using a completely different algorithmic means to form the final clusters also suffers from the high-density bias (e.g., via DMC, described in Section \ref{sec-Density-maximization-issue}).

Here we show that the use of mass distribution via the Isolation Kernel is an effective means to address these fundamental issues. The proposed algorithm called Mass-Maximization Clustering (MMC) employs mass distribution (instead of density distribution) and the mass-maximization criterion to form the final clusters (instead of the point-to-point linking process).
As a result, MMC has no high-density bias, even though its density counterpart DMC (which is exactly the same algorithm, except the Gaussian Kernel is used instead of the Isolation Kernel)  has the high-density bias. 

Our empirical evaluation reveals that (a) MMC has superior clustering outcomes to DMC as well as existing density-based algorithms and spectral clustering; and (b) both MMC and DMC have linear time complexity  \textcolor{black}{with respect to the dataset size}, whereas existing density-based clustering algorithms have at least quadratic time complexity.


\newpage

\appendix
\section*{Appendix}
\section{Formulations for F1 Score and Adjusted Mutual Information}

We have used two commonly used metrics for evaluating clustering outcomes of different clustering algorithms called F1 score and Adjusted Mutual Information (AMI).
Their formulations are provided as follows:


\begin{equation}
\text{F1} = 2 \times \frac{\text{Precision} \times \text{Recall}}{\text{Precision} + \text{Recall}}
\nonumber
\end{equation}
\noindent
where Precision is the ratio of true positives (TP) to the sum of true positives and false positives (FP), and Recall is the ratio of true positives to the sum of true positives and false negatives (FN).


\begin{equation}
\text{AMI} = \frac{\text{MI}(U, V) - \mathbb{E}[\text{MI}(U, V)]}{\max(H(U), H(V)) - \mathbb{E}[\text{MI}(U, V)]}
\nonumber
\end{equation}
\noindent
where  \(\text{MI}(U, V)\) denotes the Mutual Information between clusterings \(U\) and \(V\). \(H(U)\) and \(H(V)\) are the entropies of the clusterings, and \(\mathbb{E}[\text{MI}(U, V)]\) represents the expected Mutual Information under random labeling.

\section{Sources of the datasets and codes used}

\begin{itemize}
    \item The data characteristics of artificial datasets: Jain, 3L, 3G-HL, 3G, 2Gaussian, AC, G-strips, and RingG are shown in Table \ref{tab-visualization}. 
    \item The  w10Gasussian and w50Gasussian datasets have been used  in a previous paper \cite{hdik}. An example of w1Gaussian, which has one-dimension Gaussian distribution is each of the two-dimensional dataset, is shown in Figure \ref{fig:exmw1}.
    \item The mnist dataset is obtained from \url{https://leon.bottou.org/projects/infimnist.}
    \item The 2-dimensional dataset used in the scale-up test consists of two Gaussian distributions and one arc. An example with 1500 points is shown in Figure \ref{fig:scale-data}.
    \item The other benchmark datasets are from \url{https://www.csie.ntu.edu.tw/~cjlin/libsvmtools/datasets/} and \url{https://archive.ics.uci.edu/}.
\end{itemize}

\begin{figure}[h]
    \centering
    \subfloat[An example of w1Gaussian]{\includegraphics[width=0.35\linewidth]{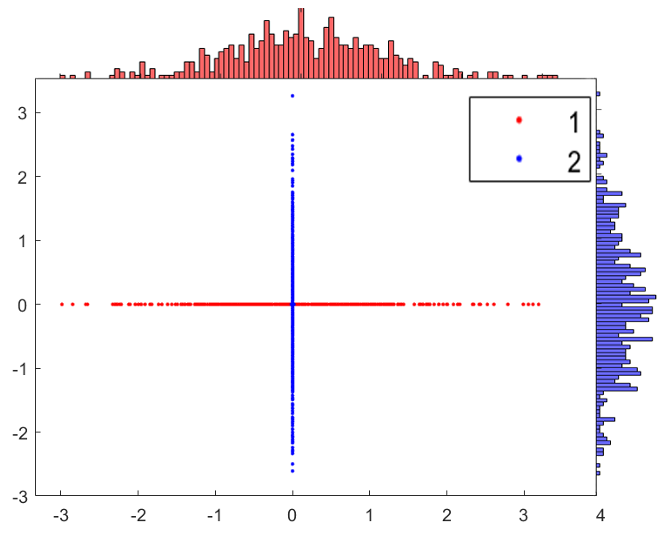}
    \label{fig:exmw1}}
    \hspace{10mm}
    \subfloat[Dataset used in the scaleup test]{\includegraphics[width=0.35\linewidth]{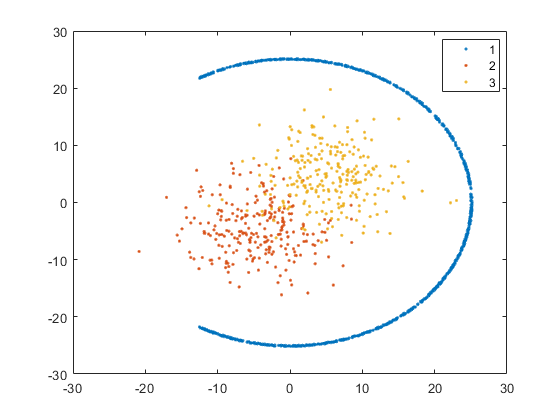}
    \label{fig:scale-data}}
    \caption{Examples of datasets used in the experiments}
\end{figure}

Each dataset is normalized to [0,1] using the min-max normalization in the preprocessing before it is applied  to all clustering algorithms. No other preprocessing is conducted.

The algorithms used in the experiments are obtained from the following sources:
\begin{itemize}
    \item HDBSCAN$^*$:  \url{https://hdbscan.readthedocs.io/en/latest/index.html}.
    \item GMM: \url{https://www.mathworks.com/help/stats/gaussian-mixture-models.html}.
    \item DBSCAN and DP: \url{https://sourceforge.net/projects/hierarchical-dp/}.
    \item MMC \& DMC: \url{https://anonymous.4open.science/r/MMC-ACD9/}.
    \item SGL: \url{https://github.com/sckangz/SGL}
    \item GLSHC: \url{https://github.com/SubaiDeng/LSSHC_matlab} 
    \item LGD: \url{https://github.com/grcai/LGD}
\end{itemize}

For each randomized algorithm, five randomized trials are conducted by using different random seeds in order to produce different initial sample sets $D_s$ from $D$. The average result from these five trials is reported. 



\section*{Acknowledgments}
Kai Ming Ting is supported by the National Natural Science Foundation of China (Grant No. 92470116 \& 62076120). This project is supported by the State Key Laboratory for Novel Software Technology at Nanjing University (Grant No.  KFKT2024A01). 

\clearpage
\bibliographystyle{ACM-Reference-Format}
\bibliography{references}

@article{SVC-JMLR2002,
author = {Ben-Hur, Asa and Horn, David and Siegelmann, Hava T. and Vapnik, Vladimir},
title = {Support vector clustering},
year = {2002},
volume = {2},
journal = {Journal of Machine Learning Research},
pages = {125–137}
}

@article{chen2021block,
  title={BLOCK-{DBSCAN}: Fast clustering for large scale data},
  author={Chen, Yewang and Zhou, Lida and Bouguila, Nizar and Wang, Cheng and Chen, Yi and Du, Jixiang},
  journal={Pattern Recognition},
  volume={109},
  pages={107624},
  year={2021},
  publisher={Elsevier}
}

@article{huang2023grit,
  title={GriT-{DBSCAN}: A spatial clustering algorithm for very large databases},
  author={Huang, Xiaogang and Ma, Tiefeng and Liu, Conan and Liu, Shuangzhe},
  journal={Pattern Recognition},
  volume={142},
  pages={109658},
  year={2023},
  publisher={Elsevier}
}

@article{luchi2019sampling,
  title={Sampling approaches for applying {DBSCAN} to large datasets},
  author={Luchi, Diego and Rodrigues, Alexandre Loureiros and Varej{\~a}o, Fl{\'a}vio Miguel},
  journal={Pattern Recognition Letters},
  volume={117},
  pages={90--96},
  year={2019},
  publisher={Elsevier}
}

@article{he2014mr,
  title={{MR-DBSCAN}: a scalable MapReduce-based {DBSCAN} algorithm for heavily skewed data},
  author={He, Yaobin and Tan, Haoyu and Luo, Wuman and Feng, Shengzhong and Fan, Jianping},
  journal={Frontiers of Computer Science},
  volume={8},
  pages={83--99},
  year={2014},
  publisher={Springer}
}

@article{yang2023reskm,
  title={Reskm: a general framework to accelerate large-scale spectral clustering},
  author={Yang, Geping and Deng, Sucheng and Chen, Xiang and Chen, Can and Yang, Yiyang and Gong, Zhiguo and Hao, Zhifeng},
  journal={Pattern Recognition},
  volume={137},
  pages={109275},
  year={2023},
  publisher={Elsevier}
}

@inproceedings{kazempour2020towards,
  title={Towards an internal evaluation measure for arbitrarily oriented subspace clustering},
  author={Kazempour, Daniyal and Kr{\"o}ger, Peer and Seidl, Thomas},
  booktitle={Proceedings of the International Conference on Data Mining Workshops},
  pages={300--307},
  year={2020},
  organization={IEEE}
}

@inproceedings{ester1996density,
    author = {Martin Ester and Kriegel, Hans-Peter and Jörg Sander and Xiaowei Xu},
    title = {A density-based algorithm for discovering clusters in large spatial databases with noise},
    booktitle = {Proceedings of the Second International Conference on Knowledge Discovery and Data Mining},
    year = {1996},
    pages = {226--231},
    publisher = {AAAI Press}
}

@article{rodriguez2014clustering,
  title={Clustering by fast search and find of density peaks},
  author={Rodriguez, Alex and Laio, Alessandro},
  journal={Science},
  volume={344},
  number={6191},
  pages={1492--1496},
  year={2014},
  publisher={American Association for the Advancement of Science}
}

@inproceedings{SC-Limitations-2006,
author = {Nadler, Boaz and Galun, Meirav},
title = {Fundamental limitations of spectral clustering},
booktitle = {Proceedings of the 19th International Conference on Neural Information Processing Systems},
pages = {1017--1024},
year = {2006}
}

@inproceedings{liu2008isolation,
	title={Isolation forest},
	author={Liu, Fei Tony and Ting, Kai Ming and Zhou, Zhi-Hua},
	booktitle={Proceedings of the IEEE International Conference on Data Mining},
	pages={413--422},
	year={2008}
	}

@article{LearningCurve,
  author    = {Kai Ming Ting and
               Takashi Washio and
               Jonathan R. Wells and
               Sunil Aryal},
  title     = {Defying the gravity of learning curve: a characteristic of nearest
               neighbour anomaly detectors},
  journal   = {Machine Learning},
  volume    = {106},
  number    = {1},
  pages     = {55--91},
  year      = {2017}
}

@inproceedings{IKDE-ICDM2021,
  author       = {Kai Ming Ting and
                  Takashi Washio and
                  Jonathan R. Wells and
                  Hang Zhang},
  title        = {Isolation kernel density estimation},
  booktitle    = {Proceedings of the {IEEE} International Conference on Data Mining},
  pages        = {619--628},
  publisher    = {{IEEE}},
  year         = {2021}
}

@inproceedings{Scalable-kModes-NEURIPS2018,
 author = {Ziko, Imtiaz and Granger, Eric and Ben Ayed, Ismail},
 booktitle = {Proceedings of Advances in Neural Information Processing Systems},
 title = {Scalable {L}aplacian k-modes},
 volume = {31},
 year = {2018}
}

@article{psKC-2023,
  title={Point-set kernel clustering},
  author={Ting, Kai Ming and Wells, Jonathan R and Zhu, Ye},
  journal={IEEE Transactions on Knowledge and Data Engineering},
year = {2023},
volume = {35},
number = {05},
pages = {5147-5158}
}

@article{IDKC-IS2023,
title = {Kernel-based clustering via Isolation Distributional Kernel},
journal = {Information Systems},
volume = {117},
pages = {102212},
year = {2023},
author = {Ye Zhu and Kai Ming Ting}
}

@article{zhu2022hierarchical,
  title={Hierarchical clustering that takes advantage of both density-peak and density-connectivity},
  author={Zhu, Ye and Ting, Kai Ming and Jin, Yuan and Angelova, Maia},
  journal={Information Systems},
  volume={103},
  pages={101871},
  year={2022},
  publisher={Elsevier}
}

@article{pei2009decode,
  title={DECODE: A new method for discovering clusters of different densities in spatial data},
  author={Pei, Tao and Jasra, Ajay and Hand, David J and Zhu, A-Xing and Zhou, Chenghu},
  journal={Data Mining and Knowledge Discovery},
  volume={18},
  pages={337--369},
  year={2009},
  publisher={Springer}
}

@article{heidari2019big,
  title={Big data clustering with varied density based on MapReduce},
  author={Heidari, Safanaz and Alborzi, Mahmood and Radfar, Reza and Afsharkazemi, Mohammad Ali and Rajabzadeh Ghatari, Ali},
  journal={Journal of Big Data},
  volume={6},
  pages={1--16},
  year={2019},
  publisher={Springer}
}

@article{hassani2015subspace,
  title={Subspace clustering of data streams: new algorithms and effective evaluation measures},
  author={Hassani, Marwan and Kim, Yunsu and Choi, Seungjin and Seidl, Thomas},
  journal={Journal of Intelligent Information Systems},
  volume={45},
  pages={319--335},
  year={2015},
  publisher={Springer}
}

@inproceedings{ntoutsi2012density,
  title={Density-based projected clustering over high dimensional data streams},
  author={Ntoutsi, Irene and Zimek, Arthur and Palpanas, Themis and Kr{\"o}ger, Peer and Kriegel, Hans-Peter},
  booktitle={Proceedings of the SIAM International Conference on Data Mining},
  pages={987--998},
  year={2012},
  organization={SIAM}
}

@inproceedings{neto2017efficient,
  title={Efficient computation of multiple density-based clustering hierarchies},
  author={Neto, Antonio Cavalcante Araujo and Sander, Joerg and Campello, Ricardo J.G.B. and Nascimento, Mario A.},
  booktitle={Proceedings of the IEEE International Conference on Data Mining},
  pages={991--996},
  year={2017},
  organization={IEEE}
}

@article{ankerst1999optics,
  title={{OPTICS}: Ordering points to identify the clustering structure},
  author={Ankerst, Mihael and Breunig, Markus M and Kriegel, Hans-Peter and Sander, J{\"o}rg},
  journal={ACM SIGMOD Record},
  volume={28},
  number={2},
  pages={49--60},
  year={1999},
  publisher={ACM New York, NY, USA}
}

@article{HDBSCAN-2015,
author = {Campello, Ricardo J. G. B. and Moulavi, Davoud and Zimek, Arthur and Sander, J\"{o}rg},
title = {Hierarchical density estimates for data clustering, visualization, and outlier detection},
year = {2015},
issue_date = {July 2015},
publisher = {Association for Computing Machinery},
address = {New York, NY, USA},
volume = {10},
number = {1},
journal = {ACM Transactions on Knowledge Discovery from Data},
articleno = {5}
}

@InProceedings{HDBSCAN-PAKDD2013,
author={Campello, Ricardo J. G. B.
and Moulavi, Davoud
and Sander, J\"{o}rg},
title={Density-Based Clustering Based on Hierarchical Density Estimates},
booktitle={Proceedings of the Pacific-Asia Conference on Knowledge Discovery and Data Mining},
year={2013},
pages={160--172}
}

@article{yang2023hcdc,
  title={{HCDC}: A novel hierarchical clustering algorithm based on density-distance cores for data sets with varying density},
  author={Yang, Qi-Fen and Gao, Wan-Yi and Han, Gang and Li, Zi-Yang and Tian, Meng and Zhu, Shu-Hua and Deng, Yu-hui},
  journal={Information Systems},
  volume={114},
  pages={102159},
  year={2023},
  publisher={Elsevier}
}

@article{ros2022detection,
  title={Detection of natural clusters via {S-DBSCAN} a self-tuning version of {DBSCAN}},
  author={Ros, Fr{\'e}d{\'e}ric and Guillaume, Serge and Riad, Rabia and El Hajji, Mohamed},
  journal={Knowledge-Based Systems},
  volume={241},
  pages={108288},
  year={2022},
  publisher={Elsevier}
}

@article{IK-on-AHC-PRJ2023,
  author       = {Xin Han and
                  Ye Zhu and
                  Kai Ming Ting and
                  Gang Li},
  title        = {The impact of isolation kernel on agglomerative hierarchical clustering algorithms},
  journal      = {Pattern Recognition},
  volume       = {139},
  pages        = {109517},
  year         = {2023}
}

@article{chen2018local,
  title={Local contrast as an effective means to robust clustering against varying densities},
  author={Chen, Bo and Ting, Kai Ming and Washio, Takashi and Zhu, Ye},
  journal={Machine Learning},
  volume={107},
  number={8-10},
  pages={1621--1645},
  year={2018},
  publisher={Springer}
}

@incollection{Nystrom_NIPS2000,
title = {Using the {N}ystr\"{o}m method to speed up kernel machines},
author = {Christopher K. I. Williams and Matthias Seeger},
booktitle = {Proceedings of Advances in {N}eural {I}nformation {P}rocessing {S}ystems},
pages = {682--688},
volume = {13},
year = {2001}
}

@book{Parameter-estimation-Book1980,
title = {Parameter Estimation: Principles and Problems},
author = {Sorenson, Harold W.},
publisher = {Marcel Dekker},
year = {1980}
}

@ARTICLE{Jarvis-Patrick-1973,
  author={Jarvis, R.A. and Patrick, E.A.},
  journal={IEEE Transactions on Computers}, 
  title={Clustering using a similarity measure based on shared near neighbors}, 
  year={1973},
  volume={C-22},
  number={11},
  pages={1025-1034}
}

@inproceedings{SNN-DBSCAN-SDM2003,
   author = {Levent Ertöz and Michael Steinbach and Vipin Kumar},
title = {Finding Clusters of Different Sizes, Shapes, and Densities in Noisy, High Dimensional Data},
booktitle = {Proceedings of the SIAM International Conference on Data Mining},
pages = {47-58},
year = {2003}
}

@inproceedings{DBSCAN-unify-KDD2023,
author = {Beer, Anna and Draganov, Andrew and Hohma, Ellen and Jahn, Philipp and Frey, Christian M.M. and Assent, Ira},
title = {Connecting the Dots---Density-Connectivity Distance unifies {DBSCAN}, k-Center and Spectral Clustering},
year = {2023},
booktitle = {Proceedings of the 29th ACM SIGKDD Conference on Knowledge Discovery and Data Mining},
pages = {80–92}
}

@article{SVC-TKDE2015,
author = {Kyoungok Kim and Youngdoo Son and Jaewook Lee},
title = {Voronoi Cell-Based Clustering Using a Kernel Support},
year = {2015},
volume = {27},
number = {4},
journal = {IEEE Transactions on Knowledge and Data Engineering},
pages = {1146–1156}
}

@article{liu2020optimal,
  title={Optimal neighborhood multiple kernel clustering with adaptive local kernels},
  author={Liu, Jiyuan and Liu, Xinwang and Xiong, Jian and Liao, Qing and Zhou, Sihang and Wang, Siwei and Yang, Yuexiang},
  journal={IEEE Transactions on Knowledge and Data Engineering},
  volume={34},
  number={6},
  pages={2872--2885},
  year={2020},
  publisher={IEEE}
}

@inproceedings{IsolationKernel-AAAI2019,
 author = {Xiaoyu Qin and Kai Ming Ting and Ye Zhu and Vincent Cheng Siong Lee},
 title = {Nearest-neighbour-induced isolation similarity and its impact on density-based clustering},
 booktitle = {Proceedings of The Thirty-Third AAAI Conference on Artificial Intelligence},
  volume={33},
  number={01},
  pages={4755--4762},
  year={2019}
}

@inproceedings{Mass-based-similarity-KDD2016,
author = {Ting, Kai Ming and Zhu, Ye and Carman, Mark and Zhu, Yue and Zhou, Zhi-Hua},
title = {Overcoming Key Weaknesses of Distance-based Neighbourhood Methods using a Data Dependent Dissimilarity Measure},
year = {2016},
booktitle = {Proceedings of the 22nd ACM SIGKDD International Conference on Knowledge Discovery and Data Mining},
pages = {1205–1214}
}

@inproceedings{TPCC-ICML2023,
author = {Grande, Vincent P. and Schaub, Michael T.},
title = {Topological point cloud clustering},
year = {2023},
booktitle = {Proceedings of the 40th International Conference on Machine Learning},
articleno = {469}
}

@inproceedings{xu-nips-2004-MMC,
  title={Maximum margin clustering},
  author={Xu, Linli and Neufeld, James and Larson, Bryce and Schuurmans, Dale},
  booktitle={Proceedings of Advances in Neural Information Processing Systems},
  volume={17},
  year={2004}
}

@inproceedings{zhang-MMC-aaai-2018-ODMC,
  title={Optimal margin distribution clustering},
  author={Zhang, Teng and Zhou, Zhi-Hua},
  booktitle={Proceedings of the AAAI Conference on Artificial Intelligence},
  volume={32},
  number={1},
  year={2018}
}

@article{GenomeResearch2025,
  title={Kernel-Bounded Clustering for spatial transcriptomics enables scalable discovery of complex spatial domains},
  author={Hang Zhang and Yi Zhang and Kai Ming Ting and Jie Zhang and Qiuran Zhao},
  journal={Genome Research},
  volume = {35},
  pages = {355-367},
  year= {2025}
}

@inproceedings{tidkc-ICDM2023,
  title={Distribution-Based Trajectory Clustering},
  author={Wang, Zi Jing and Zhu, Ye and Kai Ming Ting},
  booktitle={Proceedings of the IEEE International Conference on Data Mining},
  pages={1379--1384},
  year={2023},
  organization={IEEE}
}

@inproceedings{ting2018IsolationKernel,
  title={Isolation kernel and its effect on {SVM}},
  author={Ting, Kai Ming and Zhu, Yue and Zhou, Zhi-Hua},
  booktitle={Proceedings of the 24th ACM SIGKDD International Conference on Knowledge Discovery and Data Mining},
  year={2018},
  pages={2329--2337}
}

@ARTICLE{IDK,
  title={Isolation distributional kernel: A new tool for point and group anomaly detections},
  author={Ting, Kai Ming and Xu, Bi-Cun and Washio, Takashi and  Zhou, Zhi-Hua},
journal = {IEEE Transactions on Knowledge and Data Engineering}, 
year = {2023},
volume = {35},
number = {03},
pages = {2697-2710}
}

@inproceedings{RACE-WWW2020,
author = {Coleman, Benjamin and Shrivastava, Anshumali},
title = {Sub-linear {RACE} sketches for approximate kernel density estimation on streaming data},
year = {2020},
booktitle = {Proceedings of The World Wide Web Conference},
pages = {1739--1749}
}

@ARTICLE{KNNKernel,
author={Dmitrii Marin and Meng Tang and Ismail Ben Ayed and Yuri Boykov},
journal={IEEE Transactions on Pattern Analysis and Machine Intelligence},
title={Kernel clustering: Density biases and solutions},
year={2019}, 
volume={41}, 
number={1}, 
pages={136-147}}

@inproceedings{dhillon2004kernel,
  title={Kernel k-means: spectral clustering and normalized cuts},
  author={Dhillon, Inderjit Singh and Guan, Yuqiang and Kulis, Brian},
  booktitle={Proceedings of the tenth ACM SIGKDD International Conference on Knowledge Discovery and Data Mining},
  pages={551--556},
  year={2004}
}

@article{Scalable-kMeans-JMLR19,
  author  = {Shusen Wang and Alex Gittens and Michael W. Mahoney},
  title   = {Scalable {K}ernel {K-Means} Clustering with {N}ystr\"{o}m Approximation: Relative-Error Bounds},
  journal = {Journal of Machine Learning Research},
  year    = {2019},
  volume  = {20},
  number  = {12},
  pages   = {1-49},
}

@inproceedings{zelnik2005self,
  title={Self-tuning spectral clustering},
  author={Zelnik-Manor, Lihi and Perona, Pietro},
  booktitle={Advances in Neural Information Processing Systems},
  pages={1601--1608},
  year={2005}
}

@article{shang2022spatially,
  title={Spatially aware dimension reduction for spatial transcriptomics},
  author={Shang, Lulu and Zhou, Xiang},
  journal={Nature Communications},
  volume={13},
  number={1},
  pages={7203},
  year={2022},
  publisher={Nature Publishing Group UK London}
}

@article{marx2021method,
  title={Method of the year: Spatially resolved transcriptomics},
  author={Marx, Vivien},
  journal={Nature Methods},
  volume={18},
  number={1},
  pages={9--14},
  year={2021},
  publisher={Nature Publishing Group US New York}
}

@article{pardo2022spatiallibd,
  title={spatial{LIBD}: an {R/B}ioconductor package to visualize spatially-resolved transcriptomics data},
  author={Pardo, Brenda and Spangler, Abby and Weber, Lukas M and Page, Stephanie C and Hicks, Stephanie C and Jaffe, Andrew E and Martinowich, Keri and Maynard, Kristen R and Collado-Torres, Leonardo},
  journal={BMC Genomics},
  volume={23},
  number={1},
  pages={1--5},
  year={2022},
  publisher={Springer}
}

@article{maynard2021transcriptome,
  title={Transcriptome-scale spatial gene expression in the human dorsolateral prefrontal cortex},
  author={Maynard, Kristen R and Collado-Torres, Leonardo and Weber, Lukas M and Uytingco, Cedric and Barry, Brianna K and Williams, Stephen R and Catallini, Joseph L and Tran, Matthew N and Besich, Zachary and Tippani, Madhavi and others},
  journal={Nature Neuroscience},
  volume={24},
  number={3},
  pages={425--436},
  year={2021},
  publisher={Nature Publishing Group US New York}
}

@article{CohesiveClustering2022,
title = {Cohesive clustering algorithm based on high-dimensional generalized Fermat points},
journal = {Information Sciences},
volume = {613},
pages = {904-931},
year = {2022},
author = {Tong Li and Xiujuan Wang and Hao Zhong}
}

@article{Silhouettes1987,
title = {Silhouettes: A graphical aid to the interpretation and validation of cluster analysis},
journal = {Journal of Computational and Applied Mathematics},
volume = {20},
pages = {53-65},
year = {1987},
author = {Peter J. Rousseeuw}
}

@article{LMN-MLJ2019,
  author    = {Kai Ming Ting and
               Ye Zhu and
               Mark J. Carman and
               Yue Zhu and
               Takashi Washio and
               Zhi{-}Hua Zhou},
  title     = {Lowest probability mass neighbour algorithms: relaxing the metric
               constraint in distance-based neighbourhood algorithms},
  journal   = {Machine Learning},
  volume    = {108},
  number    = {2},
  pages     = {331--376},
  year      = {2019}
}

@article{VariableKDE1992,
author = "Terrell, George R. and Scott, David W.",
journal = "The Annals of Statistics",
number = "3",
pages = "1236--1265",
publisher = "The Institute of Mathematical Statistics",
title = "Variable kernel density estimation",
volume = "20",
year = "1992"
}

@article{DiffusionKDE2010,
author = "Botev, Z. I. and Grotowski, J. F. and Kroese, D. P.",
journal = "The Annals of Statistics",
number = "5",
pages = "2916--2957",
publisher = "The Institute of Mathematical Statistics",
title = "Kernel density estimation via diffusion",
volume = "38",
year = "2010"
}

@inproceedings{DENCLUE,
 author = {Hinneburg, Alexander and Keim, Daniel A.},
 title = {An efficient approach to clustering in large multimedia databases with noise},
 booktitle = {Proceedings of the Fourth International Conference on Knowledge Discovery and Data Mining},
 year = {1998},
 location = {New York, NY},
 pages = {58--65}
}

@article{SOMKE-2012,
author = {Cao, Yuan and He, Haibo and Man, Hong},
year = {2012},
month = {8},
pages = {1254-1268},
title = {{SOMKE}: Kernel density estimation over data streams by sequences of self-organizing maps},
volume = {23},
journal = {IEEE Transactions on Neural Networks and Learning Systems},
}

@article{AdaptiveKDE-2018,
title = "Adaptive kernel density-based anomaly detection for nonlinear systems",
journal = "Knowledge-Based Systems",
volume = "139",
pages = "50 - 63",
year = "2018",
author = "Liangwei Zhang and Jing Lin and Ramin Karim"
}

@inproceedings{mass-estimation-KDD2010,
author = {Ting, Kai Ming and Zhou, Guang-Tong and Liu, Fei Tony and Tan, James Swee Chuan},
title = {Mass estimation and its applications},
year = {2010},
publisher = {Association for Computing Machinery},
address = {New York, NY, USA},
booktitle = {Proceedings of the 16th ACM SIGKDD International Conference on Knowledge Discovery and Data Mining},
pages = {989–998}
}

@Article{Mass-estimation-MLJ2013,
author="Ting, Kai Ming
and Zhou, Guang-Tong
and Liu, Fei Tony
and Tan, Swee Chuan",
title="Mass estimation",
journal="Machine Learning",
year="2013",
volume="90",
number="1",
pages="127--160"
}

@article{DP_jain,
  title={A novel density peaks clustering with sensitivity of local density and density-adaptive metric},
  author={Du, Mingjing and Ding, Shifei and Xue, Yu and Shi, Zhongzhi},
  journal={Knowledge and Information Systems},
  volume={59},
  pages={285--309},
  year={2019},
  publisher={Springer}
}

@article{kang2021structured,
  title={Structured graph learning for scalable subspace clustering: From single view to multiview},
  author={Kang, Zhao and Lin, Zhiping and Zhu, Xiaofeng and Xu, Wenbo},
  journal={IEEE Transactions on Cybernetics},
  volume={52},
  number={9},
  pages={8976--8986},
  year={2021},
  publisher={IEEE}
}

@article{li2019LGD,
  title={Local gap density for clustering high-dimensional data with varying densities},
  author={Li, Ruijia and Yang, Xiaofei and Qin, Xiaolong and Zhu, William},
  journal={Knowledge-Based Systems},
  volume={184},
  pages={104905},
  year={2019},
  publisher={Elsevier}
}

@article{yang2021graphlshc,
  title={Graph{LSHC}: towards large scale spectral hypergraph clustering},
  author={Yang, Yiyang and Deng, Sucheng and Lu, Juan and Li, Yuhong and Gong, Zhiguo and Hao, Zhifeng and others},
  journal={Information Sciences},
  volume={544},
  pages={117--134},
  year={2021},
  publisher={Elsevier}
}

@article{torgerson1952multidimensional,
  title={Multidimensional scaling: I. Theory and method},
  author={Torgerson, Warren S},
  journal={Psychometrika},
  volume={17},
  number={4},
  pages={401--419},
  year={1952},
  publisher={Springer}
}

@article{yang2006distance,
  title={Distance metric learning: A comprehensive survey},
  author={Yang, Liu and Jin, Rong},
  journal={Michigan State Universiy},
  volume={2},
  number={2},
  pages={4},
  year={2006}
}

@book{bellet2022metric,
  title={Metric learning},
  author={Bellet, Aur{\'e}lien and Habrard, Amaury and Sebban, Marc},
  year={2022},
  publisher={Springer Nature}
}

@article{devroye2017measure,
  title={On the measure of Voronoi cells},
  author={Devroye, Luc and Gy{\"o}rfi, L{\'a}szl{\'o} and Lugosi, G{\'a}bor and Walk, Harro},
  journal={Journal of Applied Probability},
  volume={54},
  number={2},
  pages={394--408},
  year={2017},
  publisher={Cambridge University Press}
}

@article{hdik,
author = {Ting, Kai Ming and Washio, Takashi and Zhu, Ye and Xu, Yang and Zhang, Kaifeng},
title = {Is it possible to find the single nearest neighbor of a query in high dimensions?},
year = {2024},
issue_date = {Nov 2024},
volume = {336},
number = {C},
journal = {Artificial Intelligence}
}

@article{reynolds2009gaussian,
  title={Gaussian mixture models.},
  author={Reynolds, Douglas A and others},
  journal={Encyclopedia of biometrics},
  volume={741},
  number={659-663},
  pages={3},
  year={2009},
  publisher={Springer City}
}

@inproceedings{malzer2020hybrid,
  title={A hybrid approach to hierarchical density-based cluster selection},
  author={Malzer, Claudia and Baum, Marcus},
  booktitle={Proceddgins of the IEEE International Conference on Multisensor Fusion and Integration for Intelligent Systems},
  pages={223--228},
  year={2020},
  organization={IEEE}
}

@article{kbc2025aij,
  title={Kernel-bounded clustering: Achieving the objective of spectral clustering without eigendecomposition},
  author={Zhang, Hang and Ting, Kai Ming and Zhu, Ye},
  journal={Artificial Intelligence},
  pages={104440},
  year={2025},
  publisher={Elsevier}
}

\end{document}